\documentclass[twoside,11pt]{article}

\usepackage{jmlr2e}

\usepackage[11pt]{moresize}
\usepackage[utf8]{inputenc} 
\usepackage[T1]{fontenc}    
\usepackage{hyperref}       
\usepackage{url}            
\usepackage{booktabs}       
\usepackage{amsfonts}       
\usepackage{nicefrac}       
\usepackage{microtype}      
\usepackage{xr}
\usepackage{subcaption}
\usepackage{amsmath, amsfonts}
\usepackage{mathtools}
\usepackage{ifthen}
\usepackage{bm}
\usepackage{placeins}
\usepackage{indentfirst}
\usepackage{chngpage}
\usepackage{enumitem}
\usepackage{longtable}

\usepackage{multirow}
\usepackage[table]{xcolor}
\usepackage{hhline}

\usepackage{tikz}
\usetikzlibrary{bayesnet}
\usetikzlibrary{fit,positioning,arrows}

\usepackage[toc,page,header]{appendix}
\usepackage{minitoc}

\DeclarePairedDelimiterX{\divbar}[2]{}{}{%
  \left[#1\,\delimsize\|\,#2\right]%
}

\newcommand{\diag}{\mathop{\mathrm{diag}}}
\newcommand{\tr}{\mathop{\mathrm{tr}}}

\newcommand{\E}[1]{ \ifthenelse{\equal{#1}{}} {\mathbb{E}} { \mathbb{E}_{#1} }}
\newcommand{\offdiag}{\mathrm{offdiag}}

\usepackage{thm-restate}
\usepackage{lastpage}

\newcommand*\samethanks[1][\value{footnote}]{\footnotemark[#1]}

\jmlrheading{23}{2022}{1-\pageref{LastPage}}{2/21; Revised 9/22}{9/22}{21-0152}{Yaniv Yacoby, Weiwei Pan, Finale Doshi-Velez}

\ShortHeadings{Mitigating the Effects of Non-Identifiability on Inference for BNN+LV}{Yacoby, Pan, and Doshi-Velez}
\firstpageno{1}

\begin{document}

\doparttoc 
\faketableofcontents 
\part{} 

\title{Mitigating the Effects of Non-Identifiability on Inference for Bayesian Neural Networks with Latent Variables}

\author{\name Yaniv Yacoby\thanks{Equal contribution} \email yanivyacoby@g.harvard.edu \\
	\name Weiwei Pan\samethanks \email weiweipan@g.harvard.edu \\
	\name Finale Doshi-Velez \email finale@seas.harvard.edu \\
       \addr John A. Paulson School of Engineering and Applied Sciences\\
       Harvard University\\
       Cambridge, MA 02138, USA
       }

\editor{Qiang Liu}

\maketitle

\begin{abstract}
Bayesian Neural Networks with Latent Variables (BNN+LVs) capture predictive uncertainty by explicitly modeling model uncertainty (via priors on network weights) and environmental stochasticity (via a latent input noise variable). In this work, we first show that BNN+LV suffers from a serious form of non-identifiability: explanatory power can be transferred between the model parameters and latent variables while fitting the data equally well. We demonstrate that as a result, in the limit of infinite data, the posterior mode over the network weights and latent variables is asymptotically biased away from the ground-truth. Due to this asymptotic bias, traditional inference methods may in practice yield parameters that generalize poorly and misestimate uncertainty. Next, we develop a novel inference procedure that explicitly mitigates the effects of likelihood non-identifiability during training and yields high-quality predictions as well as uncertainty estimates. We demonstrate that our inference method improves upon benchmark methods across a range of synthetic and real data-sets.
\end{abstract}


\section{Introduction}\label{sec:intro}
In machine learning fields such as active learning, Bayesian optimization, and reinforcement learning, 
one often needs to fit functions to labeled data that produce both high-quality predictions 
as well as appropriate quantification of predictive uncertainty. 
Often crucial in these applications is the task of identifying sources of predictive uncertainty, 
especially those that are reducible through additional data collection \citep{zhang_bayesian_2019, Henaff2019}.

In general, one can divide the sources of uncertainty in a prediction into two categories. 
\emph{Epistemic} uncertainty, or model uncertainty, 
comes from having insufficient data to determine the ``true'' predictor, and can be reduced with more observations.  
In contrast, \emph{aleatoric} uncertainty comes from the irreducible or unobservable stochasticity in the environment. 
Although in regression one typically assumes a simple fixed form for aleatoric uncertainty 
(for example, independently and identically sampled Gaussian output noise), 
in practice, environmental stochasticity can be a complex function of the input; 
that is, many downstream tasks require us to model heteroscedastic 
aleatoric uncertainty~\citep{Chaudhuri2017, Nikolov2018, griffiths_heteroscedastic_2019, Kendall2017}. 

Bayesian Neural Networks with Latent Variables (BNN+LVs) 
~\citep{wright1999bayesian, depeweg2018decomposition} 
meet the need for scalable predictive models that can capture complex forms of epistemic uncertainty
and heteroscedastic aleatoric uncertainty. 
In particular, a BNN+LV (Figure \ref{fig:dgm}) assumes the following data generation process:
\begin{align*}
    W &\sim p(W), \quad  z_n \sim p(z), \quad \epsilon_n \sim p(\epsilon), \\
    y_n &= f(x_n, z_n; W) + \epsilon_n, \quad n = 1, \ldots, N,
\end{align*}
where $\epsilon$ is a simple form of output noise, $W$ are the parameters of a neural network, 
and $z$ is an unobserved (latent) input variable associated with \emph{each} observation $(x, y)$, 
sampled independently of the input $x$. 
BNN+LVs explicitly (and separately) model aleatoric and epistemic uncertainties: 
the distribution over $W$ captures model uncertainty (epistemic uncertainty); 
together with the output noise $\epsilon$, the stochastic latent input $z$ captures aleatoric uncertainty.
Even if the output noise $\epsilon$ is assumed to be simple, 
because the latent input $z$ is passed through the neural network alongside the input $x$, 
it can model complex, heteroscedastic noise patterns in data~\citep{depeweg2018decomposition}. 
The presence of both $\epsilon$ and $z$ in the model allows us to further explain aleatoric noise 
as arising from a combination of random observation error and latent stochastic input, 
which can represent white noise or unobserved but meaningful covariates. 
Previous work has demonstrated the usefulness of BNN+LV's decomposition of predictive uncertainty 
(into epistemic and aleatoric uncertainties) in downstream applications like active learning and safe reinforcement learning 
\citep{depeweg2018decomposition}, where one needs to carefully balance the risks and rewards of gathering new data. 

\begin{figure}[t!]
\begin{center}
  \begin{tikzpicture}
    \node[obs] (x) {$x_n$};
    \node[obs] (y) [below=of x] {$y_n$};
    \node[latent] (z) [right=of x] {$z_n$};
    \node[latent] (w) [left=of y] {$W$};
    \edge {x, z, w} {y};
    \plate[inner sep=0.1cm] {} {(x) (y) (z)} {$N$};
  \end{tikzpicture}
\end{center}
\caption{\textbf{Bayesian Neural Network with Latent Variables (BNN+LV).} The outputs $y$ are explained using a function, parameterized by $W$, of the observed inputs $x$ and a latent variable $z$.}
\label{fig:dgm}
\end{figure}

For this model, the goal of inference is to draw samples from the posterior of the weights given the data, 
$p(W | \mathcal{D})$, in order to compute a Monte Carlo estimate of the posterior predictive distribution,
\begin{align*}
\begin{split}
p(y^*| x^*, \mathcal{D})  &= \int p(y^* | x^*, W) p(W | \mathcal{D}) dW,
\end{split}
\end{align*}
which evaluates the probability of a new outcome $y^*$ given a new input $x^*$.
Since $p(W | \mathcal{D})$ requires marginalizing out 
$z_1, \dots, z_N$ from $p(W, z_1, \dots, z_N | \mathcal{D})$
(the posterior over all unknown variables), it is intractable to approximate directly.
In practice, one therefore approximates the joint posterior distribution of the weights and latent inputs given the data, 
$p(W, z_1, \dots, z_N | \mathcal{D})$, and integrates out the latent inputs $z_1, \dots, z_N$ empirically
(by drawing samples from the joint posterior, but only keeping samples over the weights).

In this work, we show that this approach yields poor quality posterior predictives in practice.
We provide a theoretical characterization explaining the problem 
and present a new inference method to mitigate it.
Specifically, we show that in the limit of infinite data, 
the mode of the \emph{true} joint posterior $p(W, z_1, \dots, z_N | \mathcal{D})$ is asymptotically biased away
from the ground-truth parameters that generated the data.
As a result, empirically marginalizing out the latent inputs from an \emph{approximation}
of the joint posterior results in an approximation of $p(W | \mathcal{D})$ that explain the data poorly.
We then propose a practical approximate inference method, 
grounded in our theoretical characterization of the asymptotic bias,
that forces the inferred posterior to satisfy our modeling assumptions. 
Our contributions are:

\paragraph{A. (True Posterior Mode) We prove that in the limit of infinite data, the weights $W$ and latent inputs $z_1, \dots, z_N$ at the mode of BNN+LV joint posterior are asymptotically biased.}
In the case that $z$ represents meaningful latent covariates, 
one would hope that $p(W, z_1, \dots, z_N | \mathcal{D})$ encodes valuable information about these latent covariates. 
However, in this work, we show that in the limit of infinite data, 
the BNN+LV joint posterior mode is not located at the ground-truth $W, z_1, \dots, z_N$ that generated the data.
Specifically, we prove that the BNN+LV's likelihood is non-identifiable and use this non-identifiability to show that,
for any set of ground-truth $W, z_1, \dots, z_N$, 
there exists an alternative set of weights $\widehat{W}, \widehat{z}_1, \dots \widehat{z}_N$
that is scored as more likely under the posterior. 
This alternative set of parameters violates our modeling assumption that $x$ and $z$ are independent:
$\widehat{z}_1, \dots \widehat{z}_N$ memorized the inputs, $x_1, \dots, x_N$,
and as a result, $\widehat{W}$ parameterizes a function that explains the observed data poorly and generalizes poorly. 
This result has two important implications:
\begin{itemize}
\item For any downstream application in which we wish to interpret the inferred latent variables $z_1, \dots, z_N$, 
it is meaningless to look at the mode of the joint posterior $p(W, z_1, \dots, z_N | \mathcal{D})$;
instead, one must look at the mode of the posterior marginal $p(z_1, \dots, z_N | \mathcal{D})$ or conditional $p(z_1, \dots, z_N | \mathcal{D}, W)$ (given a specific $W$ of interest).
\item The mode of the joint posterior $p(W, z_1, \dots, z_N | \mathcal{D})$ is asymptotically biased, 
and may thus bias imperfect approximations of the joint posterior, $q(W, z_1, \dots, z_N | \mathcal{D})$,
towards approximations that violate our generative modeling assumptions, 
just like $\widehat{W}, \widehat{z}_1, \dots \widehat{z}_N$.
As a result, empirically marginalizing out $z_1, \dots, z_N$ from this approximation
may yield an approximation $q(W | \mathcal{D})$ of $p(W | \mathcal{D})$ that explains the data poorly
(leading to contribution B).
\end{itemize}
We note that the latter implication is caused by the \emph{approximation} of the joint posterior.
Given a tractable method that directly infers $p(W | \mathcal{D})$ without approximation, 
the resultant posterior predictive may explain the data well.
In other words, while the joint posterior mode is asymptotically biased (contribution A), 
the mode of the marginal posterior $p(W | \mathcal{D})$ 
may still parameterize the data generating function in the limit of infinite data 
(i.e. the consistency of the BNN+LV true posterior predictive is an open problem). 
We emphasize that this paper focuses on pathologies caused by 
mean-field variational approximations of the joint posterior $p(W, Z | \mathcal{D})$ (contribution B).

\paragraph{B. (Approximate Posterior) We empirically demonstrate that mean-field variational inference is prone to capturing high mass regions of the joint posterior (near the biased mode), corresponding to posterior predictives that explain the data poorly, generalize poorly and misestimate uncertainty.}
Specifically, the weights and latent inputs inferred by
mean-field variational inference correspond to models that, like the weights and latent inputs of the joint posterior mode,
violate our generative modeling assumptions---that the latent variables are independent of the inputs.
The resultant approximation of $p(W | \mathcal{D})$ yields 
a posterior predictive explains the data poorly and generalizes poorly.

\paragraph{C. (Method) We propose a new variational family for BNN+LV that mitigates the effect of the joint posterior bias.}
Our proposed variational family filters out models
that do not satisfy our modeling assumptions---that the latent inputs $z$ are independent of the inputs $x$ --
thereby mitigating the effects of the asymptotic bias on mean-field variational inference. 
Since our proposed variational family is intractable to use directly,
we re-formulate variational inference with this family as a constraint optimization problem using a proxy objective.
On a range of synthetic and real data-sets, 
we empirically show that posterior predictives learned via our method, 
Noise Constrained Approximate Inference (NCAI), perform significantly better: 
models trained this way consistently recover posterior predictives with properties 
(generalization and uncertainty estimates) more similar to those of the ground-truth.


\section{Related Work}\label{sec:related}

\paragraph{Models for heteroscedastic regression.}
In the standard Bayesian Neural Network (BNN) model for regression (e.g. ~\citet{mackay1992practical,neal2012bayesian}), one generally assumes that the irreducible noise (aleatoric uncertainty) in the data is identically and independently distributed.  However, many real-world tasks~\citep{Kendall2017, depeweg2018decomposition} require more complex forms of aleatoric uncertainty.  In particular, one may need to relax the assumption that the noise is identically distributed---not only may the variance of the noise depend on the input (heteroscedasticity) but the form of the distribution may also change depending on the input.
Works that consider more complex noise models take two main forms.  The first considers a predictor of the form $y = f(x;W) + \epsilon(x)$, where the output noise $\epsilon$ is a stochastic function of the input $x$ (e.g. ~\citet{kou2015probabilistic, bauza2017probabilistic}).  These ``output noise'' models have a long history in the Gaussian process (GP) literature (e.g. ~\citet{le2005heteroscedastic, wang2012Gaussian, kersting2007most}) and have been formulated more recently for BNNs (e.g. ~\citet{Kendall2017, gal2016uncertainty}).  Such models are appropriate, for example, when one believes that aleatoric uncertainty is solely rooted in observational error of the output (the input is measured without noise and there are no unobserved explanatory variables), and that this error varies across the input domain. For example, $y$ may represent noisy sensor readings at a region $x$, and some regions may be more error-prone than others.

However, assuming such an additive noise structure can be restrictive: one must fix a specific family of distributions for the output noise. For example, for the noise distribution, one commonly chooses a zero-mean Gaussian family with input-dependent variance; this choice assumes that the observation noise is symmetrically distributed. In this paper, we focus on an alternative approach to modeling irreducible noise, in which we explicitly consider a latent input variable $z$ for each input $x$, representing either white noise or meaningful but unobserved covariates, in addition to i.i.d. output noise $\epsilon$: $y=f(x,z;W)+\epsilon$.  
This ``latent input noise'' model encapsulates the ``output noise'' model as a special case and allows us to capture arbitrarily complex noise patterns while assuming simple distributions for $z$ and $\epsilon$. Furthermore, since the ``latent input noise'' model decomposes the source of noise into observation error and latent stochastic input, this model is more appropriate when, based on domain or task-specific knowledge, one wants to explicitly account for latent factors that affect the output. For example, $y$ may represent patient response to treatment given measured factors, $x$, including BMI, as well as latent factors, $z$, such as stress level, which is either not measured or cannot be measured in practice. In this case, $x$ represents stable characteristics of the patient (e.g. BMI does not vary too much from day to day), while $z$ represents transient characteristics of the patient (stress can vary wildly as a function of external factors, e.g. work-stress, family emergencies). Since in this scenario, $z$ cannot be meaningfully predicted from $x$, we assume that $x$ and $z$ are independent, and that $z$ is randomly sampled every time the patient is observed. During inference, one would ideally infer both the function parameters $W$ as well as the latent factors $z$ that explain the treatment outcome $y$ for a given patient $x$. 

\paragraph{Inference challenges for latent variable models.}
A related ``latent input noise''  model is the frequentist mixed effects model. 
Although the literature in this area is rich, most works only consider cases where the function $f$ is linear. 
In some of these cases, it has been shown that the MLE of the parameters $W$ of $f$ and the latent variables $z$ is inconsistent due to the ``non-vanishing effects" of the random variables $z$ (as the number of observations grows, so does the number of latent variables $z$ that need to be estimated). 
That is, inferring jointly the model parameters and the latent factors could yield asymptotically biased estimates. 
This problem is known as the ``incidental parameters problem''~\citep{neyman_consistent_1948, lancaster_incidental_2000}. 
While for specific models, it is possible to get a consistent MLE estimator for the parameters $W$
by first marginalizing out the latent variable $z$~\citep{kiefer1956}
there are no works to our knowledge that address the consistency of estimators of 
$W$ and $z$ in the case that $f$ is an arbitrarily non-linear function. 

In the case of Bayesian mixed effects models, there are a number of works that 
demonstrate the frequentist consistency of the posterior $p(W, z_1, \dots, z_N | \mathcal{D})$ 
as both the number of inputs and the number of occurrences of 
each latent variable approach infinity~\citep{baghishani_asymptotic_2012}. 
There are, however, fewer works that examine the consistency of  $p(W, z_1, \dots, z_N | \mathcal{D})$ 
when the number of occurrences of each latent variable is fixed~\citep{Wang2017}. 
In fact, it is generally not known whether or not the joint posterior $p(W, z_1, \dots, z_N | \mathcal{D})$ 
or the marginal posterior $p(W | \mathcal{D})$ over $W$ (derived by integrating out $z$) 
is consistent for arbitrary non-linear functions $f$. 

\paragraph{Bayesian latent variable models for non-linear regression.}
For Bayesian ``latent input noise'' models where the function $f$ is nonlinear, 
there are a number of works that place a Gaussian Process prior over $f$
(e.g. ~\citet{lawrence2007hierarchical, mchutchon2011Gaussian, damianou2014variational}). 
In contrast, there are only a few works on ``latent input noise'' models that place a BNN prior over 
$f$~\citep{wright1999bayesian, depeweg2018decomposition}. 
While GP latent variable models have been shown to successfully capture complex noise patterns in 
low-dimensional regression data, for high-dimensional data with non-local structure 
(e.g. images, natural language) it is more natural to apply models like BNN+LV whose 
priors are over flexible parametric forms of $f$.
However, in this work, we show that the flexibility gained by adding latent input noise variables to 
BNNs presents new challenges for inference. 

\paragraph{Non-identifiability in deep probabilistic models.} 
Due to their complexity, deep probabilistic models are often non-identifiable---that is, there exist several different sets of parameters that all explain the observed data equally well.
For example, the weights of a BNN can be permuted while still parameterizing the same function 
(known as ``weight-space symmetry''~\cite{pourzanjani_improving_2017}),
and the latent space of a Variational Autoencoder~\citep{kingma_auto-encoding_2013} can be transformed while still explaining the observed data equally well (e.g. ~\cite{Locatello2018}).
In such scenarios, it is has been shown that the undesirable effects of non-identifiability can be mitigated
by modifying the model itself (e.g. ~\cite{pourzanjani_improving_2017,Khemakhem2019}),
or by specifying additional model selection criteria~\citep{Zhao2018}. 

To our knowledge, we are the first to describe non-trivial likelihood non-identifiability 
that occurs in BNN+LV and how this non-identifiability impacts inference (particularly variational inference).  
The non-identifiability we characterize in this paper is different than the previously studied
weight-space symmetry of BNNs, and thus presents different challenges for inference;
the posterior multi-modality caused by the weight-space symmetry has been empirically shown 
to slow the convergence for MCMC methods and lead to poor variational approximations
~\citep{pourzanjani_improving_2017, papamarkou_challenges_2019}.
In contrast, the BNN+LV likelihood non-identifiability we characterize is between the weights and latent variables.
As we show in this work, it causes the posterior mode of the posterior $p(W, z_1, \dots, z_N | \mathcal{D})$ to be asymptotically biased.
These problems have not been explicitly considered in previous work likely because 
the impacts of non-identifiability on inference are often attributed to general optimization difficulties.  
Based on our analysis of non-identifiability in BNN+LV models, 
we propose modifications to the mean-field variational family that 
explicitly mitigate the effects of likelihood non-identifiability.


\section{Background and Notation} \label{sec:background}
Let $\mathcal{D}=\{(x_1,y_1), \ldots (x_N, y_N) \}$ be a data-set of $N$ 
observations from the true data distribution $p(y | x) p(x)$, 
in which each input $x_n \in \mathbb{R}^D$ is a $D$-dimensional vector and each output $y_n\in \mathbb{R}^L$ is $L$-dimensional.
Let $I$ denote the identity matrix and let capital letters denote sets of variables, 
e.g. $X = \{ x_n \}_{n=1}^N$, $Y = \{ y_n \}_{n=1}^N$, $Z = \{ z_n \}_{n=1}^N$. 
See Appendix \ref{sec:notation} for a summary of the notation.

\paragraph{A Bayesian Neural Network (BNN).} A BNN assumes a predictor of the form $y = f(x; W) + \epsilon$,
where $f$ is a neural network parametrized by $W$ and $\epsilon \sim \mathcal{N}(0, \sigma^2_\epsilon \cdot I)$ 
is a normally distributed noise variable. 
It places a prior $p(W)$ on the network parameters; 
given a data-set $\mathcal{D}$, we can apply Bayesian inference to compute the 
posterior distribution $p(W | \mathcal{D})$ over $W$ or the posterior predictive distribution $p(y^* | x^*, \mathcal{D})$
of over outputs $y^*$ given a new input $x^*$. 
As commonly done, we assume $p(W) = \mathcal{N}(0, \sigma^2_w \cdot I)$. 

\paragraph{A Bayesian Neural Network with Latent Variables (BNN+LV).} A BNN+LV enables more flexible noise distributions for BNNs by introducing a latent variable $z_n \sim \mathcal{N}(0, \sigma_z^2 \cdot I)$ for each observation $(x_n, y_n)$ ~\citep{depeweg2018decomposition}. It assumes the following data generation process (Figure \ref{fig:dgm}):
\begin{equation}\label{eqn:gen_model}
\begin{split}
    W &\sim p(W), \quad  z_n \sim p(z), \quad \epsilon_n \sim \mathcal{N}(0, \sigma_\epsilon^2 \cdot I), \\
    y_n &= f(x_n, z_n; W) + \epsilon_n, \quad n = 1, \ldots, N.
\end{split}
\end{equation}
In this model, $z$ is independent from $x$, and can represent either white noise or meaningful latent explanatory variables. 
When $f$ is non-linear, BNN+LV is able to model heteroscedastic noise by transforming $z$.  
Inference for BNN+LVs involves approximating the posterior distribution, 
\begin{align}
p(W, Z | \mathcal{D}) &\propto p(W) \cdot \prod\limits_n p(y_n | x_n, z_n, W) p(z_n),
\label{eq:true-posterior}
\end{align}
(where $Z = \{z_1,\ldots,z_N\}$),
over both network weights $W$ and the latent input $z_n$ for each observation $x_n$. 
When $Z$ represents meaningful latent variables, we may infer the latent information $z_n$ for each input $x_n$ by marginalizing $p(W, Z | \mathcal{D})$ over $W$. For a new input $x^*$, the posterior predictive is given 
by the expected likelihood under the posterior of $W$, and the prior of $z$~\citep{depeweg2018decomposition}:
\begin{align}
\begin{split}
p(y^*| x^*, \mathcal{D}) 
&= \int p(y^* | x^*, W) p(W| \mathcal{D}) dW \\
&= \iint p(y^* | x^*, z^*, W) p(z^*) dz^* p(W| \mathcal{D}) dW. 
\end{split}
\label{eq:posterior-predictive}
\end{align}
Note that $z^*$ is sampled from the prior $p(z)$ to compute the posterior predictive for a new input. 
This is because, in BNN+LVs, 
the form of environmental stochasticity modeled by the latent input $z$ does not change between train and test time.  
As an example, suppose that the latent input for our model is $z\sim \mathcal{N}(0, 1)$. 
Given an observation $(x_n, y_n)$ in the training data, we may infer the values of $z_n$ 
that is likely to have generated $y_n$ given $x_n$, i.e. we compute the posterior $p(z_n | x_n, y_n, W)$. 
Note that the posterior $p(z_n | x_n, y_n, W)$ will generally not be concentrated around $z_n = 0$ like the prior
(e.g. if the sampled noise $z_n$ is equal to 2 then the posterior $p(z_n | x_n, y_n, W)$ should concentrate near 2). 
However, given a new input $x^*$ for which we want to make a prediction 
(i.e. we are asked to predict $y^*$ rather than being given the target), 
what we've inferred about the input noise for $x_n$ is irrelevant to the prediction task for $x^*$, 
since the latent input $z^*$ for the new input $x^*$ is generated randomly from $p(z)$ and is independent of $z_n$. 

\paragraph{Inference for BNN+LVs.}
Our inference goal is to draw samples from $p(W | \mathcal{D})$, 
so we can compute a Monte-Carlo estimate of the posterior predictive (Equation \ref{eq:posterior-predictive}).
While asymptotically exact, MCMC methods are generally impractical for BNNs with large architectures 
trained over large data-sets; as such, we focus on variational inference. 
However, for BNN+LV, it is intractable to approximate $p(W | \mathcal{D})$ directly using variational inference
(by minimizing $D_\text{KL} \lbrack q_\phi(W | \mathcal{D}) || p(W | \mathcal{D}) \rbrack$), 
since doing so requires an intractable marginalization of $z_1, \dots, z_N$ (see Appendix \ref{sec:intractable-marginalization-of-z}). 
Instead, ~\citet{depeweg2018decomposition} advocate for approximating the posterior over all unobserved variables,
$p(W, Z | \mathcal{D})$.
One can then easily sample from $p(W | \mathcal{D})$ by sampling from $p(W, Z | \mathcal{D})$ 
and disregarding the samples over $Z$. 

As commonly done in the BNN literature (e.g. ~\cite{blundell2015weight,depeweg2018decomposition,Foong2019}), 
we approximate the true posterior with a fully factorized Gaussian over network weights and latent variables:
\begin{equation}
\begin{split}
q_\phi(Z, W | \mathcal{D}) &= q_\phi(Z | \mathcal{D}) \cdot q_\phi(W | \mathcal{D}) \\
&= \prod\limits_n q_\phi(z_n | x_n, y_n) \cdot \prod\limits_i q_\phi(w_i) \\
&= \prod\limits_n \mathcal{N} (z_n | \mu_{z_n}, \sigma_{z_n}^2 \cdot I) \cdot \prod\limits_i \mathcal{N} (w_i | \mu_{w_i}, \sigma_{w_i}^2),
\end{split} \label{eq:mfvf}
\end{equation}
where $\phi$ is the set of variational parameters $\{\mu_{z_n}, \sigma_{z_n}^2\}_{n=1}^N \cup \{\mu_{w_i}, \sigma_{w_i}^2\}_{i=1}^I$, over which we minimize a choice of divergence between the $q_\phi(Z, W | \mathcal{D})$ and the true posterior $p(W, Z | \mathcal{D})$.  We choose the commonly used KL-divergence, yielding the following evidence lower bound:
\begin{align}
\begin{split}
\text{ELBO}(\phi) &= \E{q_\phi(Z, W | \mathcal{D})} \lbrack p(Y | X, W, Z) \rbrack 
- D_{\text{KL}}[q_\phi(W | \mathcal{D})\|p(W)] 
- D_{\text{KL}}[q_\phi(Z | \mathcal{D}) \| p(Z)].
\label{eq:mf-elbo}
\end{split}
\end{align}
Maximizing $\text{ELBO}(\phi)$ over $\phi$ is equivalent to minimizing the KL-divergence of our approximate and true posteriors.

\paragraph{Uncertainty decomposition in BNN+LV.} 
Following \citet{depeweg2018decomposition}, we quantify the overall uncertainty in the 
posterior predictive using entropy, $\mathbb{H}\left[p(y_*|x_*)\right]$. 
We compute the aleatoric uncertainty due to $z$ and $\epsilon$ 
by taking the expectation of $\mathbb{H}\left[p(y_*|W, x_*)\right]$ with respect to $W$:
\begin{align}
\mathbb{E}_{q_\phi(W | \mathcal{D})} \left[ \mathbb{H}\left[p(y_*|W, x_*)\right]\right]. 
\label{eq:aleatoric}
\end{align}
We then quantify the epistemic uncertainty due to $W$ by computing the difference between total and aleatoric uncertainties: 
\begin{align}
\mathbb{H}\left[p(y_*|x_*)\right] - \mathbb{E}_{q_\phi(W | \mathcal{D})} \left[ \mathbb{H}\left[p(y_*|W, x_*)\right]\right].
\label{eq:epistemic}
\end{align}

\section{Asymptotic Bias of the BNN+LV Posterior Modes} \label{sec:model}

In this section, we prove that in the limit of infinite data,
due to structural non-identifiability in the BNN+LV likelihood, 
and due to the non-vanishing effect of the prior over the latent inputs,
the mode of the BNN+LV joint posterior $p(W, Z | \mathcal{D})$ is asymptotically biased 
towards parameters that explain the observed data poorly and generalize poorly.
We do this by first characterizing a number of non-trivial ways in which 
network weights $W$ and latent input $z$ in BNN+LV models can be non-identifiable under the likelihood. 
Then, we prove that due to this non-identifiability, for any ground-truth set of weights $W^\text{true}$ 
and corresponding ground-truth latent variables $Z^\text{true} = \{z_n^\text{true} \}_{n=1}^N$, 
there exists an alternative sets of weights and latent variables
that are scored higher under the posterior as the number $N$ of observations grows. Furthermore, the functions 
parametrized by the alternate set of weights explain the observed data poorly and generalize poorly to new data. 
We note that this is unlike the case of BNNs without latent variables:
BNNs are also non-identifiable under the likelihood 
(e.g. one can permute the weights and retain the same function), but the posterior modes 
parameterize functions that recover the data generating function as the number of observations increases,
and the posterior predictive of BNNs nonetheless concentrates around the ground-truth function, 
under mild assumptions~\citep{lee_consistency_2000}. 
Our results have two notable consequences: 
\begin{enumerate}
\item \textbf{Interpretation of the Latent Inputs:} For any downstream task in which we wish to interpret the inferred latent inputs $Z$, 
one should never summarize $p(W, Z | \mathcal{D})$ with its mode;
instead, one may use the mode to summarize $p(Z | \mathcal{D})$
or $p(Z | \mathcal{D}, W)$ (given a specific $W$ of interest). 
\item \textbf{Approximate Inference:} As we show in Section \ref{sec:inference-pathologies},
the asymptotic bias of the mode of the joint posterior $p(W, Z | \mathcal{D})$ exacerbates 
the bias in our mean-field approximation of the joint posterior,
resulting in approximations of $p(W | \mathcal{D})$ that generalize poorly. 
\end{enumerate}

For intuition and clarity, we begin by characterizing likelihood non-identifiability and posterior bias for a single node of a BNN+LV and then generalize these results to a 1-layer BNN+LV. 
We assume the model is well-specified. 

\subsection{Asymptotic Bias of 1-Node BNN+LV Posterior Modes}

\paragraph{Non-Identifiability.} 
Consider univariate output generated by a single hidden-node neural network with LeakyReLU activation. For simplicity, we study a case with zero network biases, unit output weights, and additive input noise:
\begin{align*}
f(x, z; W) = \max \left \{W (x + z), \alpha W (x + z)\right \}, 
\end{align*}
where $0 < \alpha < 1$.
For any non-zero constant $C$, the pair $\widehat{W}^{(C)} = W/C$, $\widehat{z}^{(C)} = (C - 1) x + Cz$ reconstructs the observed data equally well:
\begin{equation*}
\begin{split}
\max \left \{W (x + z), \alpha W (x + z)\right \} = 
\max \left \{\widehat{W}^{(C)}\left(x + \widehat{z}^{(C)}\right), \alpha \widehat{W}^{(C)} \left(x + \widehat{z}^{(C)}\right)\right \}.
\end{split}
\end{equation*}
Now, suppose that the output is observed with Gaussian noise: $y \sim \mathcal{N}(f(x, z; W), \sigma^2_\epsilon)$. Then the true values of the parameter $W$ and the latent input noise $z$ are equally likely as $\widehat{W}^{(C)}$ and $\widehat{z}^{(C)}$ under the likelihood: $p(y| f(x, z, W), \sigma^2_\epsilon) = p(y| f(x, \widehat{z}^{(C)}, \widehat{W}^{(C)}), \sigma^2_\epsilon)$. 
We show in Theorem \ref{thm:post} that for this model, \emph{the posterior over the model parameter and latent inputs $W, Z$ is biased away from the ground-truth towards parameters that generalize poorly as the sample size grows, regardless of the choice of $W$ prior}.  

\begin{restatable}[Asymptotic Bias of the Posterior Mode of 1-Node BNN+LV]{theorem}{ThmSingleNode} \label{thm:post}~ 
Fix any $W \in \mathbb{R}$ and any bounded prior $p_W(W)$ on $W$. Suppose that inputs $\{x_1, \ldots, x_N\}$ are sampled i.i.d from $p_x$ with finite first and second moments $\mu_x, \sigma^2_x$, and that $\{z_1, \ldots, z_N\}$ are sampled i.i.d from $\mathcal{N}(0, \sigma^2_z)$, where $\mu_x^2 + \sigma^2_x > \sigma^2_z$. There exist a non-zero $\mathcal{C}$ such the probability that the scaled values $(\widehat{W}^{(C)}, \{\widehat{z}^{(C)}_n\}_{n=1}^N)$ are more likely than $\left(W, \{z_n\}_{n=1}^N\right)$ under the posterior approaches 1 as $N \to \infty$, for every $C \in (\mathcal{C}, 1)$.  
\end{restatable}
The proof of Theorem \ref{thm:post} is in Appendix \ref{sec:1node}. 
The theorem says that in the limit of infinite data, with probability approaching $1$, 
the posterior over $W, Z$ is biased away from the ground-truth model parameters. 
As such, the function corresponding to these alternative parameters generalizes poorly 
(that is, since $W$ and $\widehat{W}^{(C)}$ parameterize models with different slopes, 
$p(y | x, W) \neq p(y | x, \widehat{W}^{(C)})$, and $\widehat{W}^{(C)}$ thereby explains new data poorly). 
Furthermore, since these results hold for any bounded, data-independent prior over the weights,
they suggests that the bias in the joint posterior cannot be removed via model selection
(e.g. through a clever selection of priors). 


Next, using the exact same mechanism, 
we characterize non-identifiability and the resultant bias in the posterior for a 1-layer BNN+LV.

\subsection{Asymptotic Bias of 1-Layer BNN+LV Posterior Modes} \label{sec:1layer-bnnlv}

\paragraph{Non-Identifiability.}
Sources of non-identifiability increase when $f$ is a neural network.  Consider a single-output neural network, $f$, that takes as input $x$ and $z$ (represented as a concatenated vector in $\mathbb{R}^{2D}$) and has a single hidden layer containing $H$ hidden nodes. At the output node, we fix the activation to be the identity. Thus, the activation of the output node is computed as
\begin{align*}
a^\text{out} = (a^\text{hidden})^\top W^{\text{out}} + b^\text{out},
\end{align*}
where $W^{\text{out}}$ is a $H$-dimensional weight vector, $b^\text{out}$ is the bias and $a^\text{hidden}$ is the $H$-dimensional vector of activations of the hidden nodes. We can further expand $a^\text{hidden}$ as
\begin{align*}
a^\text{hidden} =g \left(W^x x + W^z z + b^\text{hidden}\right),
\end{align*}
where $W^x$ and $W^z$ are weight matrices in $\mathbb{R}^{H\times D}$, $b^\text{hidden}$ is a $H$-dimensional bias vector and $g$ is the activation function, applied element-wise. We characterize ways that the models parameters and the latent input variables $z$ are non-identifiable given a set of observed data. For any choice of diagonal matrix $S \in \mathbb{R}^{D\times D}$, vector $U\in \mathbb{R}^D$, and any factorization $W^z = RT$ where $T$ is in $\mathbb{R}^{D\times D}$, we can express $a^\text{hidden}$ in two \emph{equivalent} ways:
\begin{equation*}
\begin{split}
a^\text{hidden} &= g \left(W^x x + W^z z + b^\text{hidden}\right) = g \left(\widehat{W}^x x + \widehat{W}^z \widehat{z}  + \widehat{b}^\text{hidden}\right)
\end{split}
\end{equation*}
by setting:
\begin{equation} 
\begin{split} 
\widehat{W}^x = W^x + W^zS,\quad \widehat{W}^z =R, \quad 
\widehat{z} = Tz - TSx - U, \quad \widehat{b}^\text{hidden} = b + RU. \label{eq:dependence}
\end{split}
\end{equation}
Suppose that the output is observed with Gaussian noise: $y \sim \mathcal{N}(f(x, z; W), \sigma^2_\epsilon)$, 
where $W$ denotes the set of weights $(W^{\text{out}}, W^x, W^z, b^\text{out}, b^\text{hidden})$. 
Then, by only observing outputs generated by this network given observed input, 
one cannot identify ground-truth parameter and latent variable values under the likelihood: 
$p(y| f(x, z, W), \sigma^2_\epsilon) = p(y| f(x, \widehat{z}, \widehat{W}), \sigma^2_\epsilon)$. 
Just like in the case of a network with a single node, we show in Theorem \ref{thm:post_1layer} that 
\emph{the joint posterior is biased away from the ground-truth parameters, regardless of the choice of weight prior}.
\begin{restatable}[Asymptotic Bias of the Posterior Mode of 1-Layer BNN+LV]{theorem}{ThmSingleLayer} \label{thm:post_1layer}~ 
Fix any set of parameters $W$ and any bounded prior $p_W(W)$ on $W$.
Suppose that $\{x_1, \ldots, x_N\}$ is sampled i.i.d from $p_x$, with finite first and second moments, and that $\{z_1, \ldots, z_N\}$ are sampled i.i.d from $\mathcal{N}(0, \sigma^2_z \cdot I)$. There exists an alternate set of parameters $\left(\widehat{W}, \{\widehat{z}_n\}_{n=1}^N \right)$ such that the probability that these alternative parameters are valued as more likely than $\left(W, \{z_n\}_{n=1}^N \right)$ under the posterior approaches 1 as $N \to \infty$. 
\end{restatable}
The proof for Theorem \ref{thm:post_1layer} is in Appendix \ref{sec:1layer_proof}. 
As in the 1-node BNN+LV case, the bias in the joint posterior cannot be removed via model selection
(e.g. through a clever selection of bounded, data-independent priors over the weights and latent variables). 

Note that the form of non-identifiability in 1-layer BNN+LV models is not limited to the cases we describe above. In Appendix \ref{sec:y_dep}, we analyze another form of non-identifiability, in which the latent variable compensates for any scaling of $W^{\text{out}}$ by encoding the training outputs $y$. In practice, we find that poor posterior predictive distributions are always associated with the latent variable encoding for either the input $x$ or the output $y$, both of which we quantify by measuring the mutual information of the inferred $z$'s and the training data (see Section \ref{sec:exp}). Lastly, we note that our characterization of non-identifiability can be easily extended to multi-layer networks, which have, at the very least, the types of non-identifiability we describe above at the input layer. 

We next demonstrate that the weights and latent inputs most likely under the joint posterior
violate our generative modeling assumptions---an insight that we will exploit in Section \ref{sec:method} 
to develop a new method to mitigate the effects of the asymptotic bias on approximate inference.

\begin{figure*}[p!]
    \centering
    
    \begin{subfigure}[t]{0.48\textwidth}
        \centering
        \small        
        \includegraphics[width=1.0\textwidth]{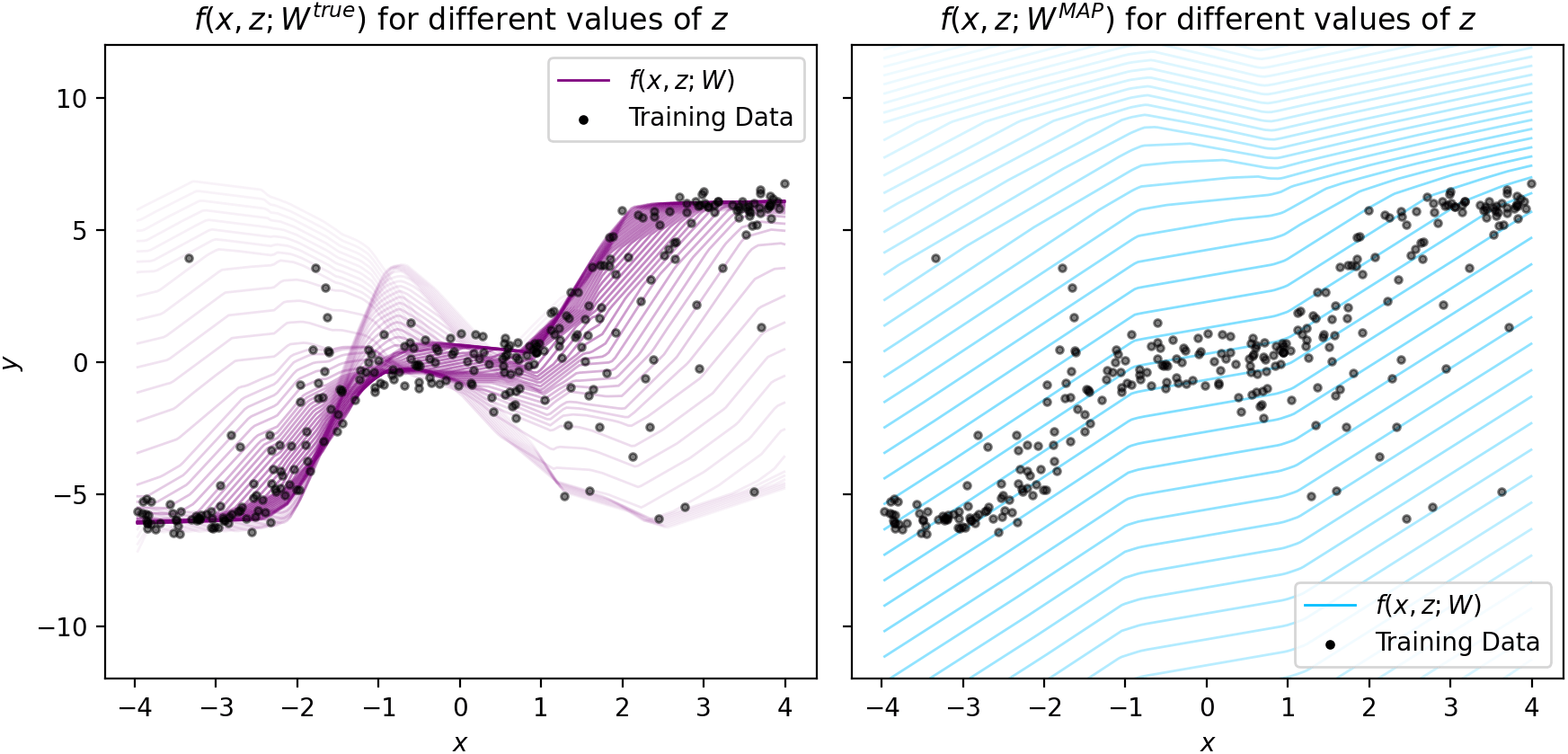} 
        \includegraphics[width=1.0\textwidth]{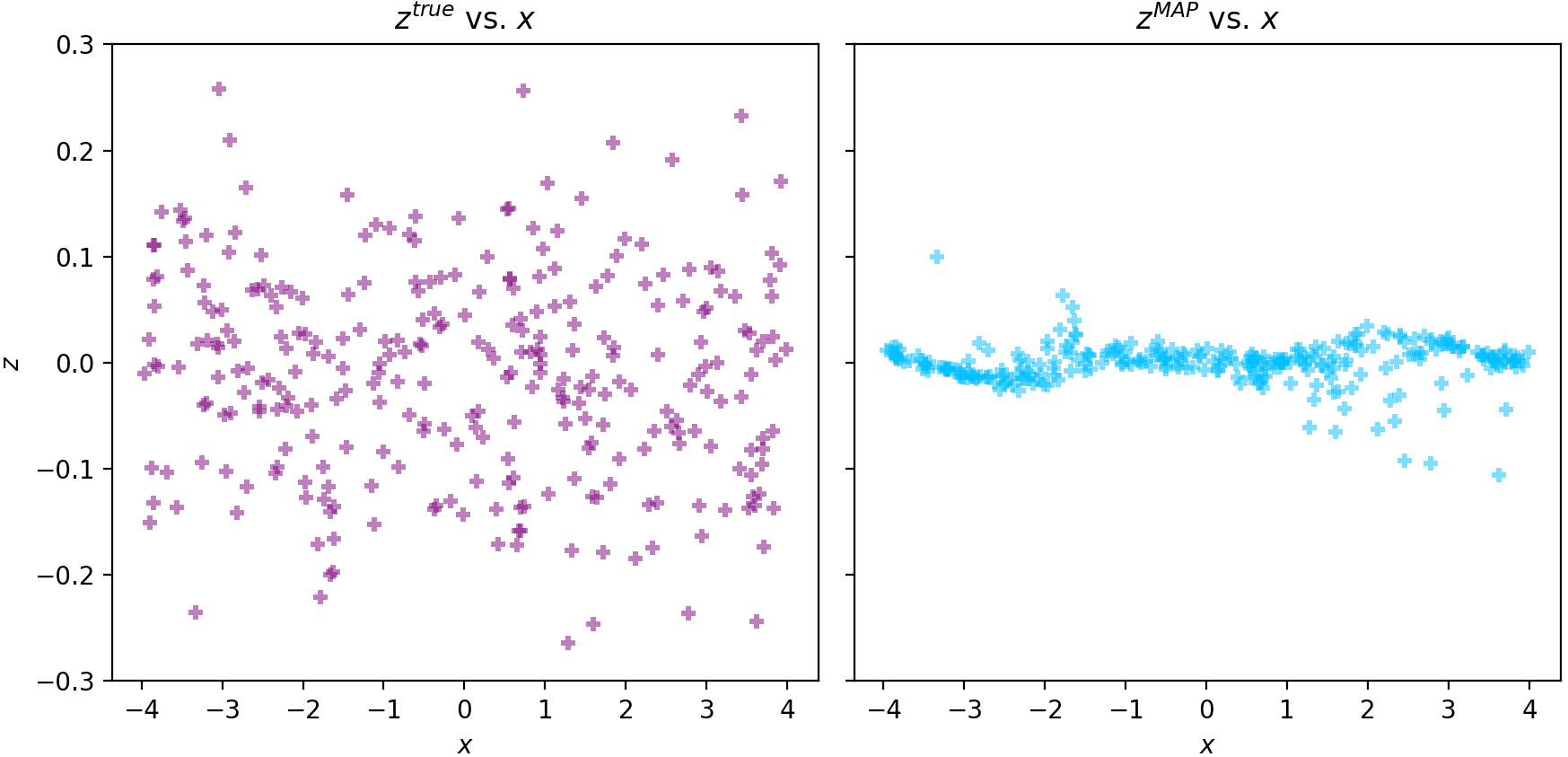} 
        \caption{\textbf{Heavy-Tail:} true (purple) vs. MAP (blue)}
    \end{subfigure}%
    ~ 
    \begin{subfigure}[t]{0.48\textwidth}
        \centering
        \small
        \includegraphics[width=1.0\textwidth]{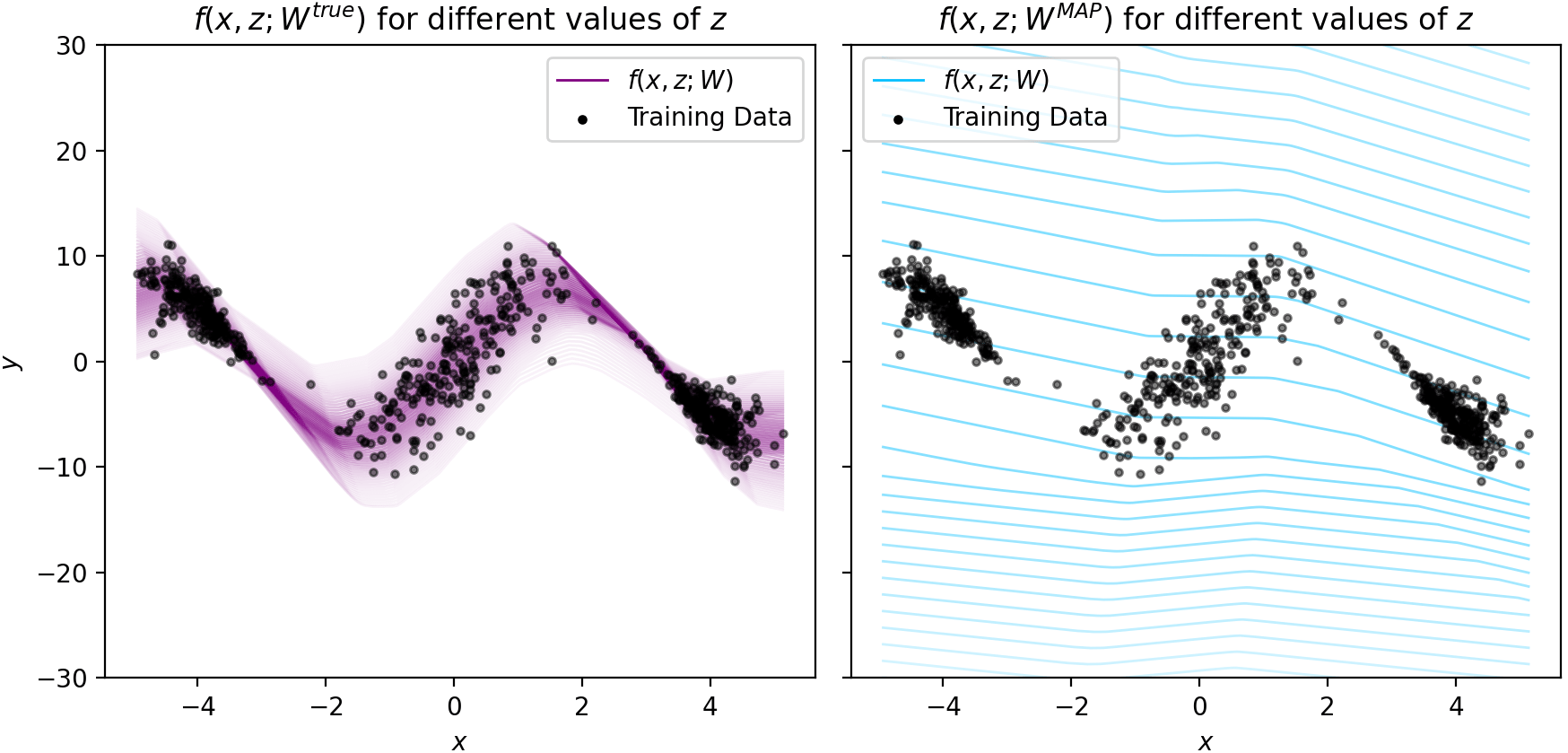}
        \includegraphics[width=1.0\textwidth]{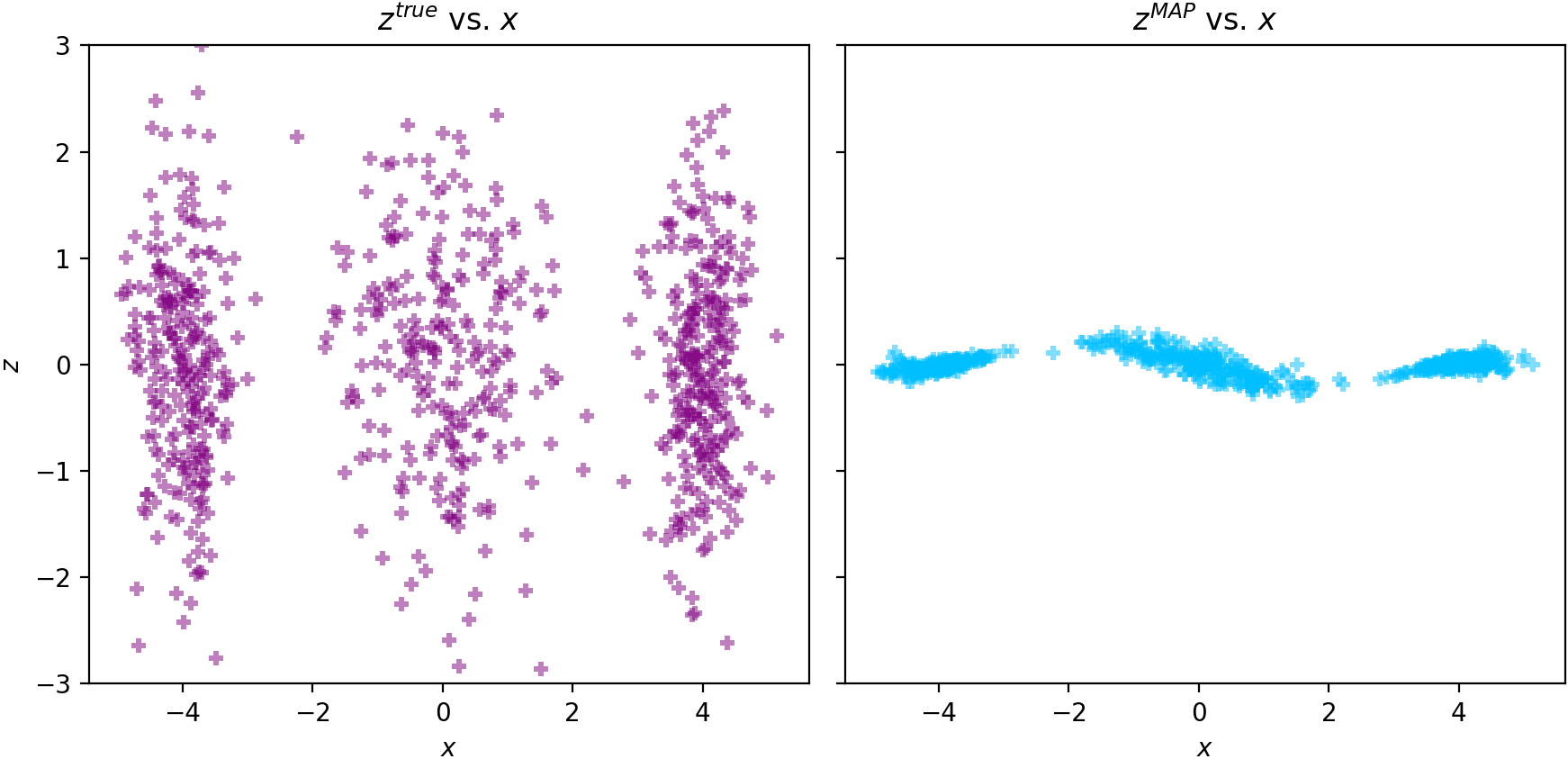}  
        \caption{\textbf{Depeweg:} true (purple) vs. MAP (blue)}
    \end{subfigure}%
    
    \vspace{0.4cm}
    \begin{subfigure}[t]{0.48\textwidth}
        \centering
        \small
        \includegraphics[width=1.0\textwidth]{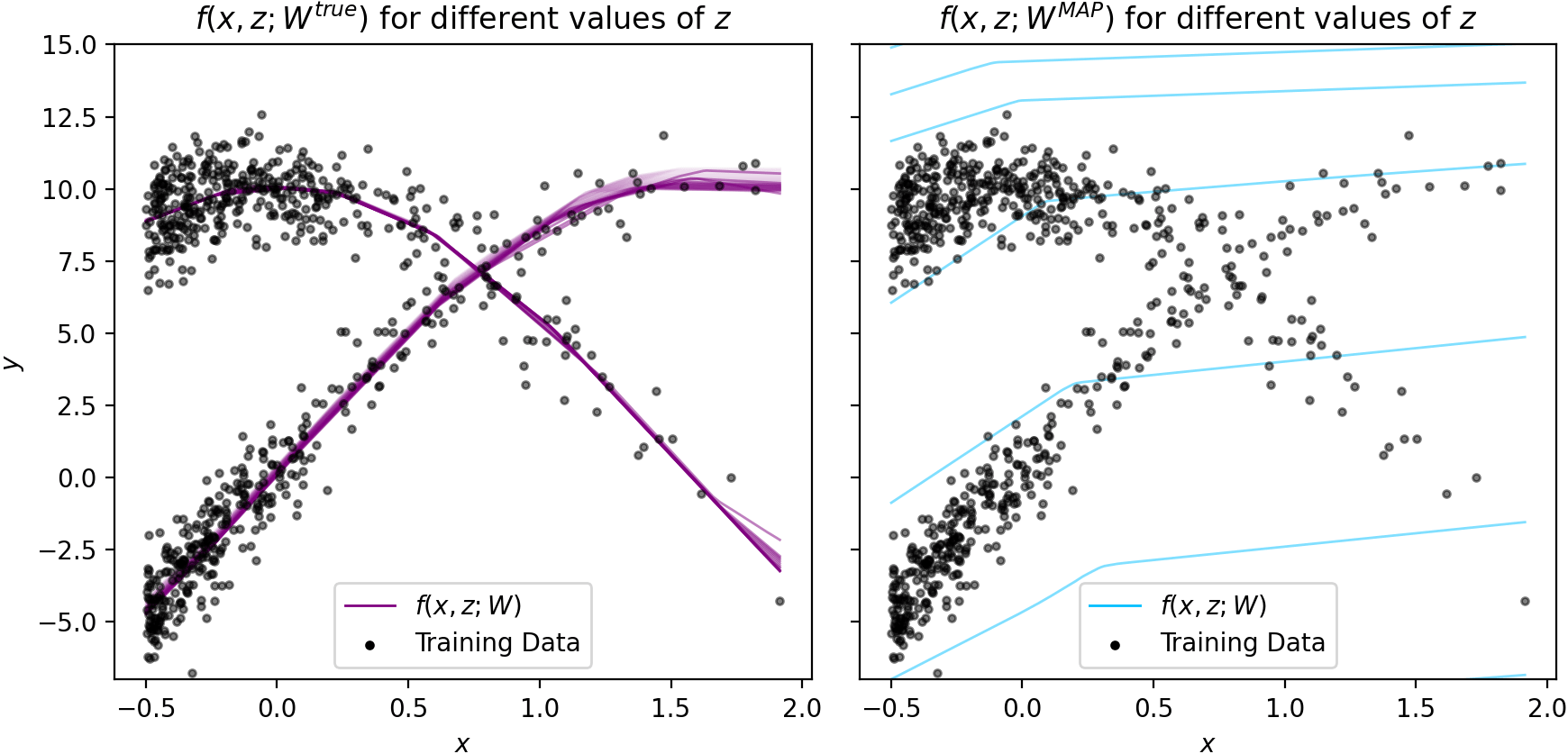}
        \includegraphics[width=1.0\textwidth]{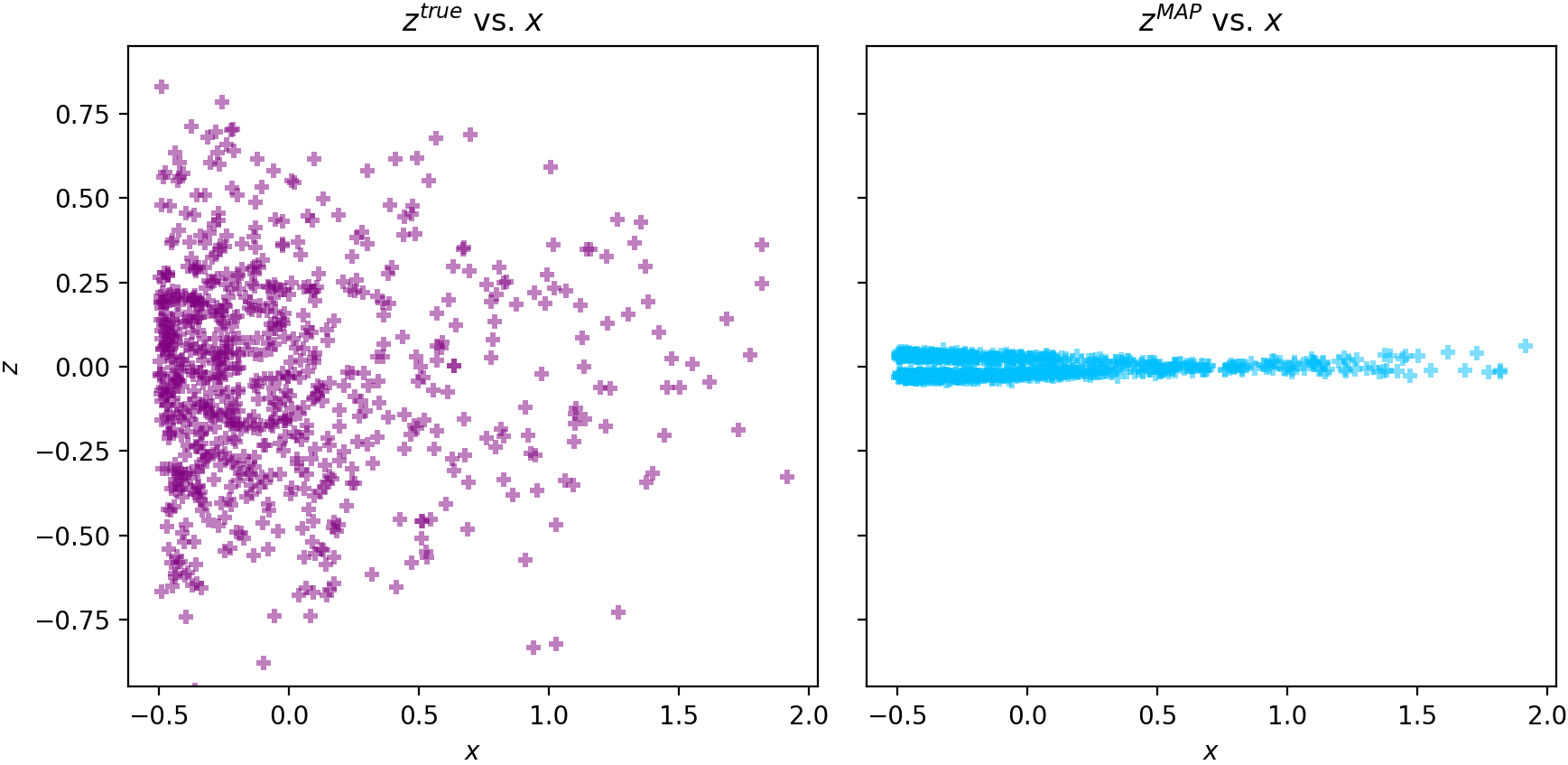}  
        \caption{\textbf{Bimodal:} true (purple) vs. MAP (blue)}
    \end{subfigure}%
    ~
    \begin{subfigure}[t]{0.48\textwidth}
    \footnotesize
    \begin{tabular}{l||c|c}
    & \multicolumn{2}{c}{log posterior of} \\
    & $W^\text{true}, Z^\text{true}$ & $W^\text{MAP}, Z^\text{MAP}$ \\ \hline \hline
    \textbf{Heavy-Tail} & $-311.2 \pm 12.9$  & $\bm{53.0 \pm 0.4}$   \\ \hline
    \textbf{Depeweg}    & $-1719.5 \pm 20.7$ & $\bm{-938.7 \pm 0.2}$  \\ \hline
    \textbf{Bimodal}    & $-2516.1 \pm 22.3$ & $\bm{-1560.2 \pm 0.0}$ \\ 
    \end{tabular}
    \caption{Comparison of log-posterior evaluated at the true parameter $W^\text{true}, Z^\text{true}$ vs. the
    MAP parameters $W^\text{MAP}, Z^\text{MAP}$. The posterior is significantly higher at the MAP.}
    \label{tab:map}
    \end{subfigure}%
    
    \caption{\textbf{At the MAP of the joint posterior, $W$ explains the data poorly and $Z$ memorizes the data.} Top-row of (a)-(c): We compare the ground-truth function (purple) vs. function learned via MAP estimate over $W, Z$ (blue) (with $\sigma^2_w = 10.0$). The functions are plotted for every value of $z$ on $\lbrack -3 \sigma_z, 3 \sigma_z \rbrack$ with opacity proportional to $p(z)$. Bottom row of (a)-(c):  $Z^\text{true}$ and $Z^\text{MAP}$ are plotted vs. $X$. The MAP estimated parameters (i) have \emph{significantly} higher log-posterior probability than the ground-truth, (ii) $Z^\text{MAP}$ memorizes the data (violating the assumptions in Section \ref{sec:assumption-violation}), and (iii) as a result, the learned functions explain the observed data poorly and over-estimate aleatoric uncertainty.}
    \label{fig:map-pathology}
\end{figure*}

\subsection{The Joint Posterior Mode Violates Generative Modeling Assumptions} \label{sec:assumption-violation}

Generally, when assuming a probabilistic model, we hope that in the limit of infinite data,
inference recovers model parameters that satisfy our generative modeling assumptions. 
For example, when assuming a BNN, we hope that in the limit of infinite data,
the MLE and MAP over the weights parameterize the true function that generated the data.
For the BNN+LV model, however, the parameters given by the MAP of the joint posterior 
$W^\text{MAP}, Z^{\text{MAP}}$ do not satisfy our generative modeling assumptions. 
To illustrate this, we compare $W^\text{MAP}, Z^{\text{MAP}}$ 
with the ground-truth data generating parameters, $W^\text{true}, Z^{\text{true}}$;
under our ground-truth data-generating process from Equation \ref{eqn:gen_model}, 
$Z^{\text{true}}$ satisfies two modeling assumptions:
\begin{enumerate}
\item[] \textbf{Assumption 1:} $z^\text{true}$ is independent of $x$---that is, $p(z, x)$ factorizes.
\item[] \textbf{Assumption 2:} $z^\text{true}$ is distributed like the prior $p(z)$. 
\end{enumerate}
In contrast, using our characterization of likelihood non-identifiability from Equation \ref{eq:dependence}, 
in the limit of infinite data,
the posterior always prefers an alternative set of parameters $\widehat{W}, \widehat{Z}$
that violate the above modeling assumptions.
Specifically in Equation \ref{eq:dependence}, each $\widehat{z}_n$ becomes directly dependent on the inputs $x_n$,
(or indirectly dependent on $x_n$ through $y_n$---see Appendix \ref{sec:y_dep}), 
thereby violating assumption 1. 
Next, since our characterization, $\widehat{z}_n$ is dependent on $x_n$ (which may not be drawn from a Gaussian),
$\widehat{z}_1, \dots, \widehat{z}_N$ may not be distributed like the prior, thereby violating assumption 2.

\paragraph{The two assumptions are independent of one another.} 
We note that violating one assumption does not necessarily imply violating the other.
Consider for instance $\widehat{z}_1, \dots, \widehat{z}_N$ that are distributed like $p(z)$ 
(and hence do not violate assumption 2),
but that are dependent on $x_1, \dots, x_N$ 
(e.g. $\widehat{z}_1, \dots, \widehat{z}_N$ are sorted such that small 
$\widehat{z}$'s are paired with small $x$'s and vice versa),
thus violating assumption 1.
Alternatively, consider $\widehat{z}_1, \dots, \widehat{z}_N$ that are independent of the $x$'s, 
but that are not distributed like the $p(z)$ 
(e.g. $\widehat{z}_1, \dots, \widehat{z}_N$ are still distributed like a Gaussian but with a different variance),
thereby only violating assumption 1. 
We also note that, while there may be other modeling assumption that the joint posterior mode violates,
these are the two assumptions we found to be important empirically
and defer investigation of other modeling assumption to future work. 

\paragraph{} On three synthetic data-sets (for which we know the ground-truth),
we next empirically demonstrate that under the joint posterior mode,
the weights parameterize a function that explains the observed data poorly,
and the latent inputs violate the above modeling assumptions.
In Section \ref{sec:inference-pathologies} we demonstrate empirically that,
since the mean-field variational posterior is closer to the MAP than to the ground-truth solution,
it similarly violates these two modeling assumptions
and therefore results in posterior predictives that explain the data poorly and generalize poorly.
Then, in Section \ref{sec:method}, we use this insight to propose a new inference method
that enforces these modeling assumptions explicitly.

\subsection{Empirical Demonstration of Asymptotic Bias of Joint Posterior Mode} \label{sec:empirical-MAP}

In Figure \ref{fig:map-pathology}, we empirically demonstrate that 
(1) the posterior mode $W^\text{MAP}, Z^\text{MAP}$,  
is not located at the ground-truth parameters $W^\text{true}, Z^\text{true}$,
(2) that $W^\text{MAP}$ parametrizes functions that generalize poorly 
and misestimate aleatoric uncertainty, by putting mass where there is no data,
and (3) that $Z^\text{MAP}$ violates the two modeling assumptions from Section \ref{sec:assumption-violation}. 
In these experiments, we optimize for the MAP using gradient descent, 
selecting the best solution across 9 random initializations / 1 initialization at the ground-truth $W^\text{true}, Z^\text{true}$.
In Table \ref{tab:map}, we see that the observed data has a significantly higher log posterior probability 
under the MAP solution $W^\text{MAP}, Z^\text{MAP}$ 
than it does under the ground-truth parameters $W^\text{true}, Z^\text{true}$,
confirming that indeed $W^\text{true}, Z^\text{true} \neq W^\text{MAP}, Z^\text{MAP}$. 
In Figure \ref{fig:map-pathology}, we visualize the predictive distribution corresponding to 
$W^\text{MAP}$ for data drawn from the same distribution as the training data. 
For each of the three data-sets, we see that the  $W^\text{MAP}$ parametrizes a function 
that puts mass where there is no data (i.e. over-estimates aleatoric uncertainty),
thereby explaining the observed data poorly.
Correspondingly, we also see that the $Z^\text{MAP}$ memorized the data 
and exhibits a distribution different than the prior, 
violating the two modeling assumptions from Section \ref{sec:assumption-violation}. 

We next show that because mean-field variational inference prefers solutions closer to the MAP than to the ground-truth, 
it suffers from the same issues at the MAP---mean-field VI yields approximations of $p(W | \mathcal{D})$ that explain the observed data poorly because they violate the generative modeling assumptions.

\section{Effect of Asymptotic Bias of the BNN+LV Posterior Mode on Variational Inference} \label{sec:inference-pathologies}

Whereas in Section \ref{sec:model} we focused on the asymptotic bias of the BNN+LV joint posterior mode,
in this section we unfold the consequence of this asymptotic bias on mean-field variational inference. 
We show that because solutions returned by mean-field variational inference
are in practice closer to the MAP than they are to the ground-truth parameters, 
they violate the assumptions made in the generative model.
As a result, they correspond to posterior predictives that under-fit the observed data and misestimate uncertainty.

\paragraph{Mean-field VI returns posterior predictives that under-fit the data and violate generative modeling assumptions.} 
We follow the standard practice of:
initializing the parameters of the mean-field variational family (Equation \ref{eq:mfvf}) 
using a random initialization (described in Appendix \ref{sec:exp_details}),
maximizing the ELBO (in Equation \ref{eq:mf-elbo}), 
and selecting models with the highest validation log-likelihood over 10 random restarts.
In Figure \ref{fig:1d-qualitative} (blue column), 
we see that traditional inference posterior predictives that under-fit the data and misestimate the noise.
Furthermore, the figure shows that the latent variables recovered from mean-field VI exhibit strong dependence on the data
(i.e. memorized the data), thereby violating our generative modeling assumptions (Section \ref{sec:assumption-violation}). 
Whereas in Figure \ref{fig:1d-qualitative} we offer qualitative results, 
in Section \ref{sec:exp} we demonstrate that this pathology occurs on a variety of synthetic and real data-sets
both qualitatively and quantitatively.  
We next argue that traditional inference results models that under-fit the data and violate our generative modeling assumptions because, in practice, traditional inference yields posterior distributions closer to the MAP than to the ground-truth. 

\definecolor{FixedGT}{RGB}{190, 130, 190}
\definecolor{GT}{RGB}{219, 184, 219}
\definecolor{MAP}{RGB}{132, 192, 255}

\begin{table}[t]
  \begin{subtable}[t]{1.0\textwidth} 
  \centering
  \small
  \begin{tabular}{l||c|c|c|c|c}
  & \textbf{Draw 1} & \textbf{Draw 2} & \textbf{Draw 3} & \textbf{Draw 4} & \textbf{Draw 5} \\ \hline \hline
  \textbf{Lowest ELBO} &         \cellcolor{FixedGT} Fixed @ GT &  \cellcolor{FixedGT} Fixed @ GT &  \cellcolor{FixedGT} Fixed @ GT &  \cellcolor{FixedGT} Fixed @ GT &  \cellcolor{FixedGT} Fixed @ GT \\ \hhline{~-----}
  \multicolumn{1}{c||}{\multirow{3}{*}{$\bm{\bigg\downarrow}$}} &        Random &      Random &          \cellcolor{GT} GT &          Random &      \cellcolor{GT} GT \\ \hhline{~-----}
  &          \cellcolor{GT} GT &          \cellcolor{GT} GT &      Random &      \cellcolor{GT} GT &          Random  \\ \hhline{~-----}
  &            $\text{NCAI}_{\lambda=0}$ &        \cellcolor{MAP} MAP &        $\text{NCAI}_{\lambda=0}$ &        $\text{NCAI}_{\lambda=0}$ &         $\text{NCAI}_{\lambda=0}$  \\ \hhline{~-----}
  \textbf{Highest ELBO} &  \cellcolor{MAP} MAP &        $\text{NCAI}_{\lambda=0}$ &        \cellcolor{MAP} MAP &        \cellcolor{MAP} MAP &       \cellcolor{MAP} MAP \\ \hline
  \end{tabular}
         
  \caption{Bimodal data-set}
  \label{fig:bimodal-ranking}
  \end{subtable}
  
  \begin{subtable}[t]{1.0\textwidth} 
  \centering
  \small
  \begin{tabular}{l||c|c|c|c|c}
  & \textbf{Draw 1} & \textbf{Draw 2} & \textbf{Draw 3} & \textbf{Draw 4} & \textbf{Draw 5} \\ \hline \hline
  \textbf{Lowest ELBO} &        \cellcolor{FixedGT} Fixed @ GT &  \cellcolor{FixedGT} Fixed @ GT &  \cellcolor{FixedGT} Fixed @ GT &  \cellcolor{FixedGT} Fixed @ GT &  \cellcolor{FixedGT} Fixed @ GT \\ \hhline{~-----}
  \multicolumn{1}{c||}{\multirow{3}{*}{$\bm{\bigg\downarrow}$}} &        \cellcolor{GT} GT &          \cellcolor{GT} GT &          \cellcolor{GT} GT &          \cellcolor{GT} GT &         \cellcolor{GT} GT \\ \hhline{~-----}
  &         \cellcolor{MAP} MAP &        \cellcolor{MAP} MAP &        \cellcolor{MAP} MAP &        \cellcolor{MAP} MAP &        \cellcolor{MAP} MAP  \\ \hhline{~-----}
  &            Random &      Random &      Random &      Random &      Random  \\ \hhline{~-----}
  \textbf{Highest ELBO} &  $\text{NCAI}_{\lambda=0}$ &        $\text{NCAI}_{\lambda=0}$ &        $\text{NCAI}_{\lambda=0}$ &        $\text{NCAI}_{\lambda=0}$ &        $\text{NCAI}_{\lambda=0}$ \\ \hline
  \end{tabular}
         
  \caption{Heavy-Tail data-set}
  \label{fig:heavytail-ranking}
  \end{subtable}
  
  \begin{subtable}[t]{1.0\textwidth} 
  \centering
  \small
  \begin{tabular}{l||c|c|c|c|c}
  & \textbf{Draw 1} & \textbf{Draw 2} & \textbf{Draw 3} & \textbf{Draw 4} & \textbf{Draw 5} \\ \hline \hline
  \textbf{Lowest ELBO} &        \cellcolor{FixedGT} Fixed @ GT & \cellcolor{FixedGT} Fixed @ GT &      Random &    \cellcolor{FixedGT}  Fixed @ GT &          \cellcolor{FixedGT} Fixed @ GT \\ \hhline{~-----}
  \multicolumn{1}{c||}{\multirow{3}{*}{$\bm{\bigg\downarrow}$}} &        Random &      Random &  \cellcolor{FixedGT} Fixed @ GT &  Random &  Random \\ \hhline{~-----}
  &          NCAI &        \cellcolor{GT}  GT &         \cellcolor{GT} GT &         \cellcolor{GT} GT &      \cellcolor{GT} GT  \\ \hhline{~-----}
  &           \cellcolor{GT} GT &        \cellcolor{MAP} MAP &        \cellcolor{MAP} MAP &        NCAI &         \cellcolor{MAP} MAP  \\ \hhline{~-----}
  \textbf{Highest ELBO} &  \cellcolor{MAP} MAP &        NCAI &        NCAI &        \cellcolor{MAP} MAP &        NCAI \\ \hline
  \end{tabular}
         
  \caption{Depeweg data-set}
  \label{fig:depeweg-ranking}
  \end{subtable}
  
  \caption{\textbf{Mean-field VI returns solutions closer to the MAP than to the ground-truth parameters.} For each of the tables above (each corresponding to a different synthetic data-set, detailed in Appendix \ref{sec:datasets}), and for each initialization scheme (Section \ref{sec:inference-pathologies}), we initialize and run mean-field VI 10 times, selecting the run with the highest ELBO. We then sort the random initialization from lowest to highest ELBO. These tables show that, both across different data-sets, as well as across different draws from the same data-generating process, MAP initialization results in a higher ELBO than initialization at the ground-truth (GT). Even more problematically, it results in a higher ELBO than the ELBO evaluated directly at the ground-truth (Fixed @ GT).}
    \label{tab:elbo-rankings}
\end{table}

\paragraph{Mean-field VI may return solutions closer to the MAP than to the ground-truth.}
To show that mean-field VI typically returns solutions closer to the MAP than to the ground-truth,
we show that when initialized at the MAP, mean-field VI yields a higher ELBO than when fixed at the ground-truth, 
or when initialized at the ground-truth (i.e. the ELBO prefers models initialized at the MAP).
Correspondingly, MAP-initialized inference also yields posterior predictives that under-fit the data. 
This experiment suggests that in practice, the optima of the ELBO commonly found via gradient descent
share more characteristics with the problematic MAP of the joint posterior
than with the ground-truth data-generating parameters.

In this experiment, 
we initialize the variational parameters using each of the following schemes (after which we optimize the ELBO):
\begin{itemize}
\item Ground-Truth (GT): We initialize the variational means to $W^\text{true}, Z^\text{true}$. Holding the variational means fixed, we optimize the variational variances until convergence. 
\item MAP: We compute the MAP of $p(W, Z | \mathcal{D})$ via gradient descent, selecting the best of 10 random restarts (1 initialized at $W^\text{true}, Z^\text{true}$, and 9 with $W^\text{true}$ and with $Z$ sampled randomly from the prior). We then initialize the variational means to $W^\text{MAP}, Z^\text{MAP}$. Holding the variational means fixed, we optimize the variational variances until convergence. 
\item Random: We randomly initialize the variational parameters.
\item $\text{NCAI}_{\lambda=0}$: For completeness, we even use the initialization we propose later in this work (in Section \ref{sec:method}).
\end{itemize}
For each of the above initialization schemes, we initialize and run mean-field VI 10 times, 
selecting the models with the highest ELBO.
We also compare the ELBO optimized from each of the above initializations
to the ELBO evaluated at the ground-truth initialization (``Fixed @ GT'').
We sort the resultant ELBOs from lowest to highest,
and repeat this whole experiment 5 times (for 5 draws of data-sets from each synthetic data-generating process).

Table \ref{tab:elbo-rankings} shows the results.
If we were to \emph{only} consider the two ground-truth initializations (``Fixed @ GT'' and ``GT'')
and the MAP initialization, we find that we find that \emph{across all} three synthetic data-sets
\emph{and across all} 5 draws of each of these data-sets,
the MAP initialization yields a higher ELBO
(i.e. the relative ordering of blue, light purple and dark purple in \emph{all three} tables is always the same).
Correspondingly, inference initialized at the MAP yields posterior predictives that fit the data poorly.
For instance, in the Bimodal data-set (Figure \ref{fig:bimodal-global-optima-viz}),
in draws 2, 3 and 4, the posterior predictive places mass where there is no data when $x > 0.5$,
in draw 1, the model places a little more mass above the observed data at $x = 0.75$,
and in draw 5, the posterior predictive mean is significantly biased relative to the ground-truth. 
This result helps us explain why in practice, mean-field VI for BNN+LV yields poor quality models:
it shows that there exist several different sets of variational parameters $\phi$ 
that are preferred by the ELBO over the ground-truth parameters, 
but that retain undesirable properties of the joint posterior mode. 

Of course, in practice we cannot use the ground-truth initializations,
and we do not want to use the MAP initialization (for reasons mentioned here),
so we are left with only one option: random initialization.
However, as already shown in Figure \ref{fig:1d-qualitative},
random initialization also yield poor results in practice.
So in practice, what initialization should we use?
In Section \ref{sec:method}, we propose a novel and practical initialization, $\text{NCAI}_{\lambda=0}$,
which we find to be empirically effective.
We include this initialization, as well as the random initialization, in Table \ref{tab:elbo-rankings}
because these two initializations will help us explore whether
these undesirable optima of the ELBO are local or global, discussed next.

\begin{figure*}[p]
\vspace{-0.5cm}

  \begin{subfigure}[t]{1.0\textwidth}
\centering
\small
\begin{tabular}{l||c|c|c|c|c}
\textbf{Initialization} & \textbf{Draw 1} & \textbf{Draw 2} & \textbf{Draw 3} & \textbf{Draw 4} & \textbf{Draw 5} \\ \hline \hline
\textbf{Fixed @ GT} & $-3662.098$ & $-3663.944$ & $-3699.185$ & $-3671.299$ & $-3698.769$ \\ \hline
\textbf{Ground-Truth}   & $-2117.360$ & $-2102.342$      & $-2151.512$ & $-2127.975$ & $-2152.788$ \\ \hline
\textbf{Random}         & $-2150.645$ & $-2120.855$      & $-2146.312$ & $-2139.177$ & $-2150.549$ \\ \hline
$\bm{\textbf{NCAI}_{\lambda=0}}$           & $-2093.275$ & $\bm{-2086.503}$ & $-2115.208$ & $-2104.083$ & $-2113.348$ \\ \hline
\textbf{MAP} & $\bm{-2087.035}$ & $-2088.970$ & $\bm{-2105.432}$ & $\bm{-2083.557}$ & $\bm{-2082.548}$ \\ 
\end{tabular}
         
         \caption{The highest ELBO (of 10 restarts) using different initialization schemes. MAP initialization results in the highest ELBO in 4 of the 5 draws.}
         \label{fig:bimodal-global-optima-tab}
    \end{subfigure}
    
     \begin{subfigure}[t]{1.0\textwidth}
     	\vspace{0.5cm}
     
     	\begin{subfigure}[t]{0.01\textwidth}
        \centering
        \small
    	\rotatebox[origin=l]{90}{\hspace{65pt}$y$}
   	\end{subfigure}
    	 \begin{subfigure}[t]{0.320\textwidth}
    	 \includegraphics[width=\textwidth]{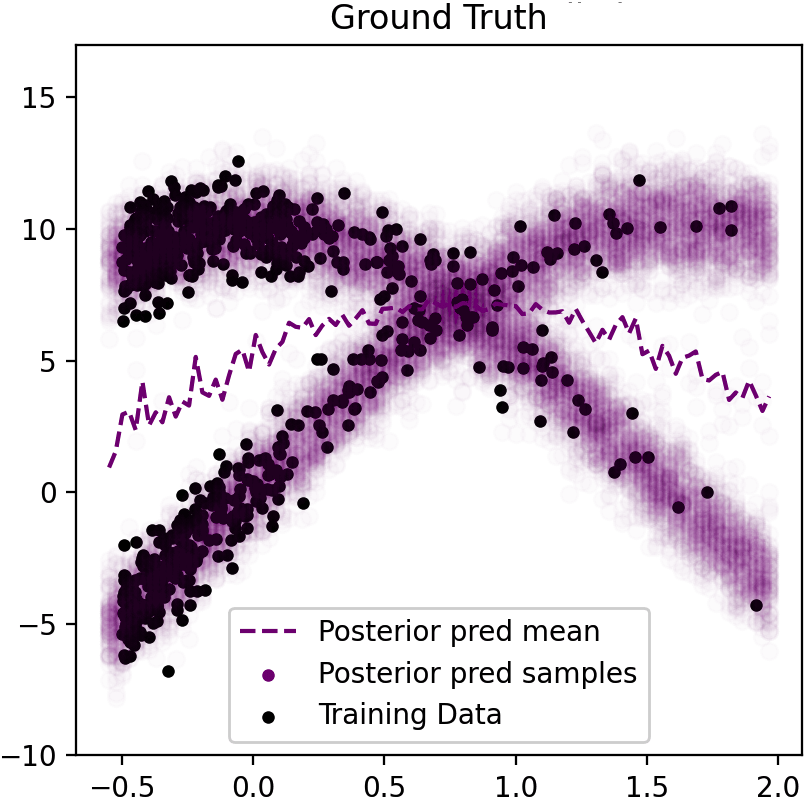}
   	 \end{subfigure}
	 \begin{subfigure}[t]{0.321\textwidth}
    	 \includegraphics[width=\textwidth]{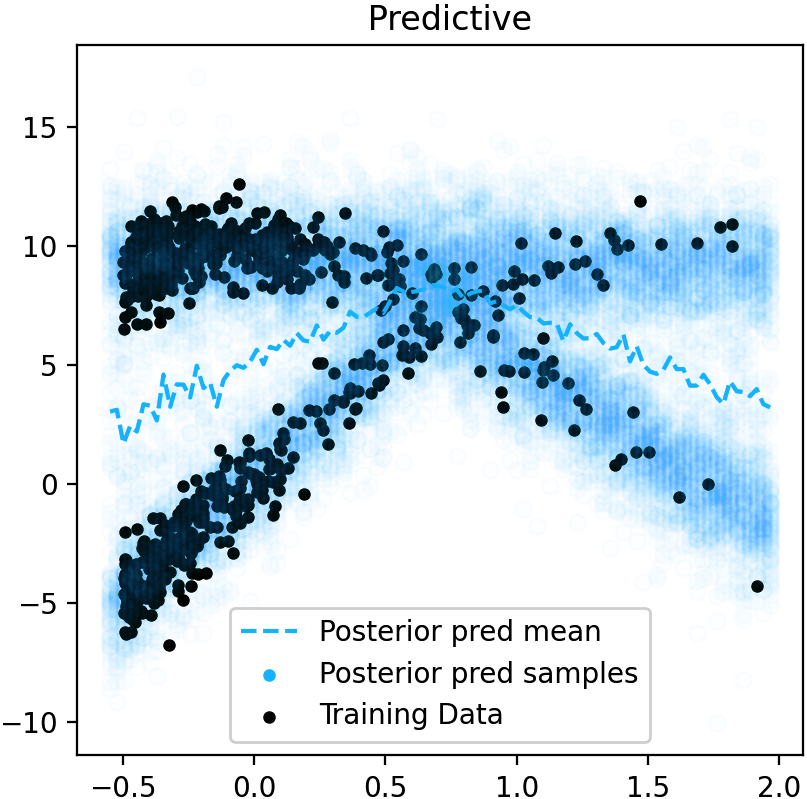}
   	 \end{subfigure}
	 \begin{subfigure}[t]{0.324\textwidth}
    	 \includegraphics[width=\textwidth]{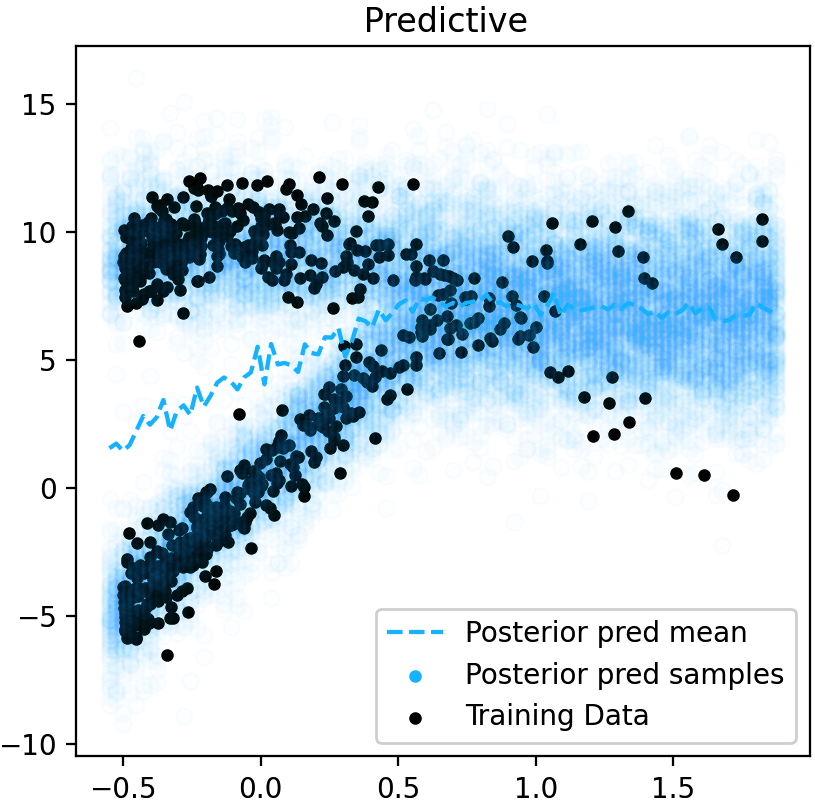}
   	 \end{subfigure}
	
	\begin{subfigure}[t]{0.01\textwidth}
        \centering
        \small
   	\end{subfigure}
	 \begin{subfigure}[t]{0.320\textwidth}
        \centering
        \small
        \vspace{-0.3cm}
        $x$
   	 \end{subfigure}
	 \begin{subfigure}[t]{0.321\textwidth}
        \centering
        \small
        \vspace{-0.3cm}
        $x$
   	 \end{subfigure}
	 \begin{subfigure}[t]{0.324\textwidth}
        \centering
        \small
        \vspace{-0.3cm}
        $x$
   	 \end{subfigure}
	
	\begin{subfigure}[t]{0.01\textwidth}
        \centering
        \small
   	\end{subfigure}
	 \begin{subfigure}[t]{0.320\textwidth}
	 \center
    	 Fixed @ GT
   	 \end{subfigure}
	 \begin{subfigure}[t]{0.321\textwidth}
	 \center
    	 Draw 1
   	 \end{subfigure}
	 \begin{subfigure}[t]{0.324\textwidth}
	 \center
    	 Draw 2
   	 \end{subfigure}
	 
	 \vspace{0.65cm}
	 
	 \begin{subfigure}[t]{0.01\textwidth}
        \centering
        \small
    	\rotatebox[origin=l]{90}{\hspace{65pt}$y$}
   	\end{subfigure}
	 \begin{subfigure}[t]{0.321\textwidth}
    	 \includegraphics[width=\textwidth]{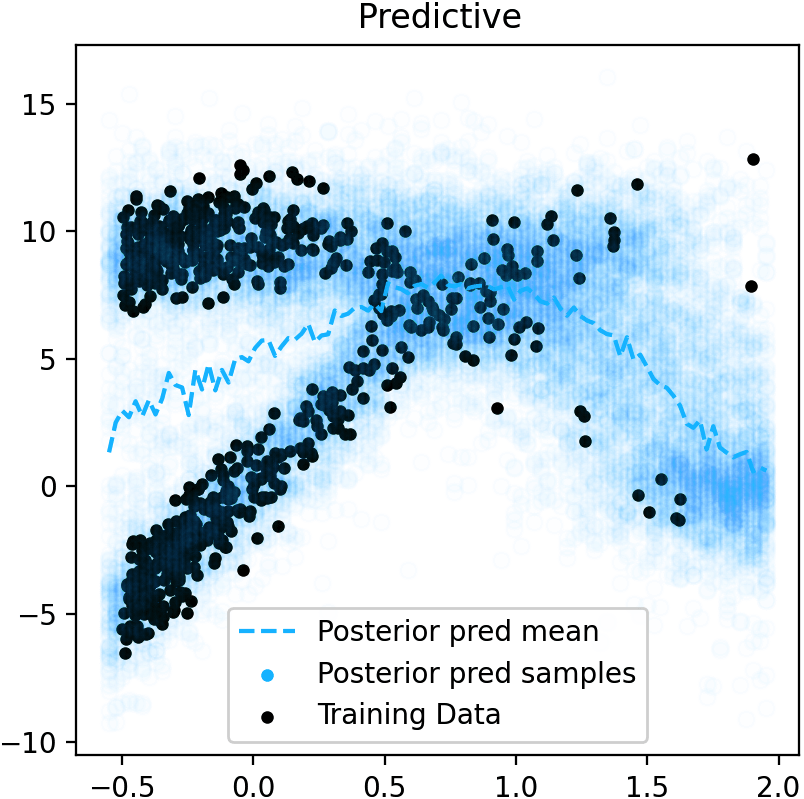}
   	 \end{subfigure}
	 \begin{subfigure}[t]{0.320\textwidth}
    	 \includegraphics[width=\textwidth]{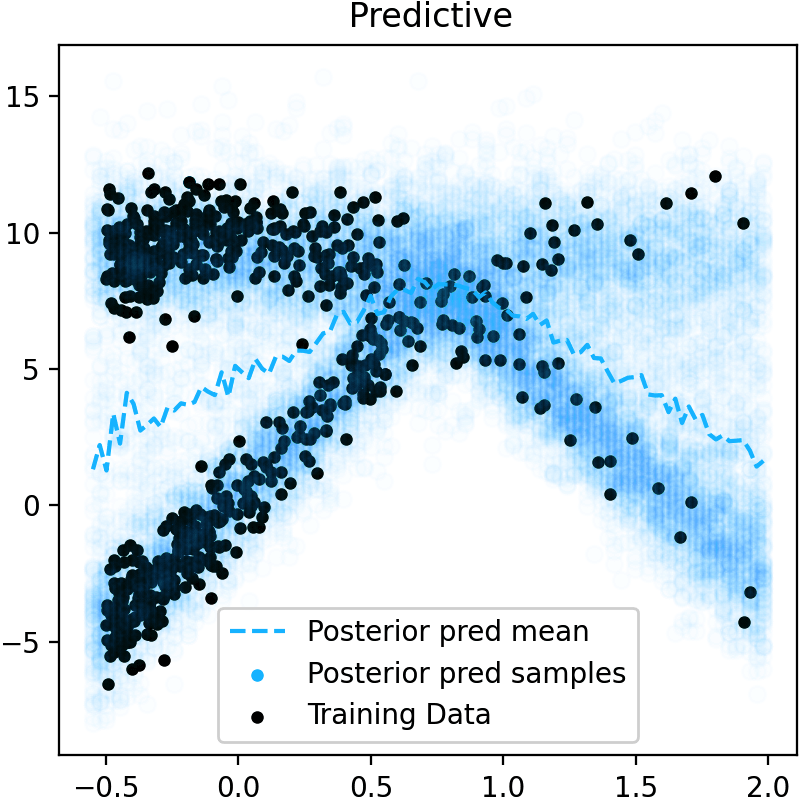}
   	 \end{subfigure}
	 \begin{subfigure}[t]{0.324\textwidth}
    	 \includegraphics[width=\textwidth]{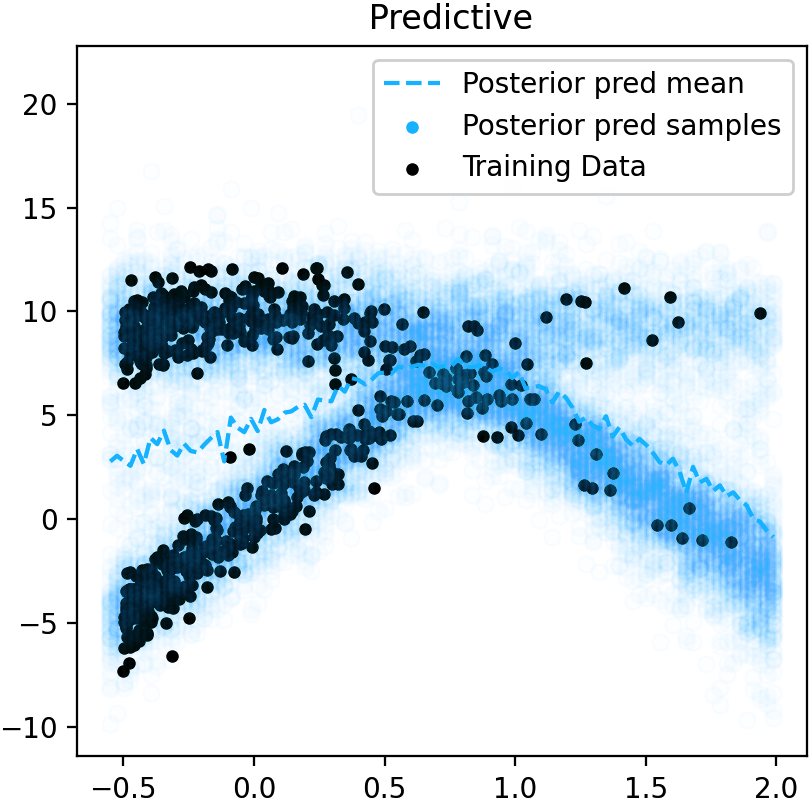} 
   	 \end{subfigure}
	 
	\begin{subfigure}[t]{0.01\textwidth}
        \centering
        \small
   	\end{subfigure}
	 \begin{subfigure}[t]{0.321\textwidth}
        \centering
        \small
        \vspace{-0.3cm}
        $x$
   	 \end{subfigure}
	 \begin{subfigure}[t]{0.320\textwidth}
        \centering
        \small
        \vspace{-0.3cm}
        $x$
   	 \end{subfigure}
	 \begin{subfigure}[t]{0.324\textwidth}
        \centering
        \small
        \vspace{-0.3cm}
        $x$
   	 \end{subfigure}
	 
	 \begin{subfigure}[t]{0.01\textwidth}
        \centering
        \small
   	\end{subfigure}
	 \begin{subfigure}[t]{0.331\textwidth}
	 \center
    	 Draw 3
   	 \end{subfigure}
	 \begin{subfigure}[t]{0.320\textwidth}
	 \center
    	 Draw 4
   	 \end{subfigure}
	 \begin{subfigure}[t]{0.324\textwidth}
	 \center
    	 Draw 5
   	 \end{subfigure}
	 
	  \vspace{0.35cm}
	  
	 \caption{Visualization of the posterior predictive, resulting from optimizing the MAP-initialized ELBO. Relative to the ground-truth (purple), posterior predictives learned via MAP-initialization do not explain the observed data well; they place mass where there is no data.}
	 \label{fig:bimodal-global-optima-viz}
    \end{subfigure}

    \caption{\textbf{Mean-field VI returns solutions closer to the MAP than to the ground-truth parameters.} We draw the Bimodal data-set (\cite{depeweg2018decomposition}, detailed in Appendix \ref{sec:datasets}) 5 times. The table displays the highest ELBO (of 10 restarts) using different initialization schemes, and the figures show the fit of the posterior predictive from MAP-initialized VI. In 4 of 5 draws of the data, the MAP initialization yields the highest ELBO, and corresponds to poor fits on the training data (blue) relative to the ground-truth (purple)---the learned posterior predictives all place mass where there is no data.}
    \label{fig:bimodal-global-optima}
\end{figure*}

\paragraph{Global vs. local optima of the ELBO.}
While we have shown that MAP-initialized inference yields higher ELBOs than the ELBO at the ground-truth,
and returns posterior predictives that explain the observed data poorly, 
does this phenomenon occur at the local or global optima of the ELBO?
That is, even the \emph{best} imperfect approximation of the joint posterior under the ELBO
would be biased towards the undesirable posterior mode?
We conjecture here that, for data-sets for which the ELBO is a loose bound to the log marginal likelihood, 
the global optima would exhibit the same properties as the MAP,
while for data-sets in which the ELBO is tight, 
only local optima will exhibit these properties. 

The Bimodal data-set is one for which we expect the ELBO to be loose,
and therefore for the global optima of the ELBO to be biased towards the MAP.
In this data-set, the ground-truth function uses $z$ as a 
binary indicator to select which function to generate. 
As such, the true posterior of $z$ given the ground-truth function $p(Z | \mathcal{D}, W^\text{true})$ 
is highly skewed and poorly approximated by a mean-field Gaussian. 
This causes the gap between the ELBO and the true marginal likelihood, 
$\mathbb{E}_{p(Z)} \lbrack p(Y | X, W^\text{true}, Z) \rbrack$, to be particularly high. 
The results in Table \ref{tab:elbo-rankings} confirm that for the Bimodal data-set,
the global optima of the ELBO may be problematic,
since the MAP-initialization results in the highest ELBO across all initializations,
whereas for the other data-sets, there exist other initializations that yield higher ELBOs, 
e.g. $\text{NCAI}_{\lambda=0}$ (which, as we show in Section \ref{sec:exp}, also results in a better fit).

\section{Mitigating Inference Challenges Caused by Model Non-Identifiability} \label{sec:method}

In Section \ref{sec:model} we showed that the BNN+LV posterior mode is asymptotically biased towards 
functions that generalize poorly.
We furthermore showed (in Section \ref{sec:assumption-violation}) that the parameters preferred by the posterior 
violate two modeling assumptions---(1) that the latent variable $z$ is independent of the input $x$
and (2) that the latent variable $z$ is drawn from a normal distribution.
More surprisingly, in Section \ref {sec:inference-pathologies} we showed that empirically, 
mean-field VI returns solutions that violate the same modeling assumptions,
and as a consequence, that the resultant posterior predictives generalize poorly and misestimate uncertainty.
This leads us to hypothesize that when these modeling assumptions are satisfied, 
the the resultant posterior predictive will generalize well and provide appropriate estimations of uncertainty. 
In this section, we therefore develop a method to enforce these modeling assumptions 
\emph{explicitly} during variational inference. 
We call this method Noise-Constrained Approximate Inference (NCAI), 
since it restricts the variational family to treat the latent input variable $z$ as identically distributed and independent of $x$. 

Our method consists of two steps: first, an intelligent model-assumption satisfying initialization,
and second, variational inference using a variational family constrained to satisfy the modeling assumptions 
from Section \ref{sec:assumption-violation}. 

\paragraph{Step 1: Model-Satisfying Initialization.}  
Since local optima are a major concern in BNN+LV inference, 
we start with settings of the variational parameters $\phi$ that 
satisfy the properties implied by our generative model (Equation \ref{eqn:gen_model}).  
We initialize the variational means $\mu_{w_i}$ of the weights 
(except for weights associated with the input noise) with those of a deterministic neural network trained on the same data,
based on the observation that a neural network is often able to capture the trend of the data (but not the uncertainty). 
We then initialize the variational means $\mu_{z_n}$ of the latent noise to $0$, 
to ensure that in the early stages of optimization, the model is forced to explain the data using $W$
(as opposed to by memorizing it with $Z$). 
Lastly, we initialize all variational variances randomly.

\paragraph{Step 2: Inference with Noise-Constrained Variational Family.} 
We further ensure that the two key modeling assumptions---that the noise variables 
$z$ are drawn \emph{independently} of $x$ and i.i.d from the \emph{prior} $p(z)$---remain satisfied 
during training by restricting our variational family.
Specifically, we construct a variational family by filtering out distributions from the mean-field variational family 
(Equation \ref{eq:mfvf}) that do not obey our modeling assumptions:
\begin{align}
\mathcal{Q} &= \{ q_\phi(W, Z | \mathcal{D}) : \underbrace{I_\phi(x; z) = 0}_{\text{assumption 1}} \text{ and } \underbrace{D\lbrack q_\phi(z) || p(z) \rbrack = 0}_{\text{assumption 2}} \}, 
\label{eq:constrained-mf-vf}
\end{align}
Here, $I_\phi(x; z)$ quantifies the statistical dependence between the  $x$'s and $z$'s 
under the posterior,
\begin{align}
I_\phi(x; z) &= D \lbrack q_\phi(z | x) p(x) || q_\phi(z) p(x) \rbrack,
\label{eq:dep}
\end{align}
where $q_\phi(z | x)$ is the approximate posterior with $y$ marginalized out
(since we only have a single $y$ associated with every $x$, $q_\phi(z | x)$ is approximated with $q_\phi(z | x, y)$.)
$D\lbrack q_\phi(z) || p(z) \rbrack$ quantifies the ``distance'' between the 
approximated aggregated posterior $q_\phi(z)$ and the prior $p(z)$.
The aggregated posterior is the posterior $q_\phi(z | x, y)$ marginalized over the observed data,
approximated as follows~\citep{Makhzani2015}:
\begin{align}
q_\phi(z) &= \E{p(x, y)} \left\lbrack q_\phi(z | x, y) \right\rbrack \approx \frac{1}{N} \sum\limits_{n=1}^N q_\phi(z_n | x_n, y_n),
\label{eq:aggregated-posterior}
\end{align}
We note that while each posterior $q_\phi(z | x, y)$ can be an arbitrary distribution, 
the aggregate posterior $q_\phi(z)$ must recover the prior when inference is exact;
that is, when  $q_\phi(z | x, y) = p(z | x, y)$, we have that 
$q_\phi(z) = \mathbb{E}_{p(x, y)} [q_\phi(z | x, y)] \approx \mathbb{E}_{p(x, y)} [p(z | x, y)]= p(z)$.

Together, these two constraints explicitly enforce both assumptions, respectively. 
We emphasize that since both constraints are satisfied by the true posterior
(i.e. when $q_\phi(W, Z) = p(W, Z | \mathcal{D})$),
these constraints are not at odds with the model---they simply help select a posterior approximation that retains the desired properties of the original model. 
Moreover, since the two constraints are orthogonal (i.e. satisfying one does not imply satisfying the other---see Section \ref{sec:assumption-violation}), both are needed. 

Now, using this noise-constrained mean-field variational family, we perform variational inference:
\begin{align}
\mathrm{argmin}_{q_\phi \in \mathcal{Q}} D_{\text{KL}}[q_\phi(W, Z | \mathcal{D})\|p(W, Z | \mathcal{D})].
\label{eq:constrained-vi}
\end{align}
As with standard variational inference, 
once we have the variational approximation that minimizes Equation \ref{eq:constrained-vi},
we use it to compute a Monte-Carlo estimate of the posterior predictive in Equation \ref{eq:posterior-predictive}:
\begin{align*}
\begin{split}
p(y^*| x^*, \mathcal{D}) 
&= \iint p(y^* | x^*, z^*, W) p(z^*) dz^* p(W| \mathcal{D}) dW \\
&\approx \frac{1}{S} \sum\limits_{s=1}^S p(y^* | x^*, z^{(s)^*}, W^{(s)}), \quad  z^{(s)^*} \sim p(z^*), \quad W^{(s)}, Z^{(s)} \sim q_\phi(W, Z | \mathcal{D}).
\end{split}
\end{align*}

\paragraph{Variational Inference with the Noise Constrained Variational Family.}
Since performing inference over the constrained $\phi$-space is challenging,
we equivalently re-write Equation \ref{eq:constrained-vi} as:
\begin{align}\label{eqn:framework}
\begin{split}
\mathrm{argmin}_\phi D_\text{KL}\lbrack q_\phi(W, Z | \mathcal{D}) || p(W, Z | \mathcal{D}) \rbrack  \quad &\text{s.t} \\
&I_\phi(x; z) = 0,\\
&D\lbrack q_\phi(z) || p(z) \rbrack = 0.
\end{split}
\end{align}
and solve Equation \ref{eqn:framework} by gradient descent on the Lagrangian:
\begin{align}\label{eqn:obj}
\begin{split}
\mathcal{L}_{\text{NCAI}}(\phi) = -\text{ELBO}(\phi) &+ \lambda_1 \cdot I_\phi(x; z) 
+  \lambda_2 \cdot D\lbrack q_\phi(z) || p(z) \rbrack.
\end{split}
\end{align}
We emphasize that even though in practice, using our proposed variational family requires solving a constraint optimization problem (similar to posterior regularization~\citep{zhu2014bayesian}), 
the theoretical justification of our method nevertheless does not deviate 
from the standard application of variational inference with a specific choice of variational family.

\paragraph{Empirical properties of the constraints.} 
We note that even though the two constraints are theoretically satisfied by the true posterior,
in practice, the constraints cannot be minimized to $0$ completely.
Firstly, even given the best mean-field posterior approximation $\phi^* = \mathrm{argmin}_\phi -\text{ELBO}(\phi)$,
the aggregated posterior $q_\phi(z)$ (used in both constraints) may not equal $p(z)$ due to approximation error.
Secondly, it is not possible to estimate $I_\phi(x; z)$ without bias:
$I_\phi(x; z)$ depends on $q_\phi(z | x)$, which is estimated by integrating out $y$ from $q_\phi(z | x, y)$,
and since only one $y$ is observed for every $x$, $q_\phi(z | x)$ is simply estimated with $q_\phi(z | x, y)$
(see Appendix \ref{sec:mixz} for details).
While $q_\phi(z | x)$ should approximate $q_\phi(z)$ (i.e. $I_\phi(x; z) = 0$),
in general $q_\phi(z | x, y)$ does not approximate $q_\phi(z)$ (i.e. $D\lbrack q_\phi(z | x, y) p(x, y) || q_\phi(z) p(x, y) \rbrack > 0$).
As such, even with no approximation error (i.e. $q_\phi(z | x, y) = p(z | x, y)$), 
replacing $q_\phi(z | x)$ with $q_\phi(z | x, y)$ in Equation \ref{eq:dep} may yield $I_\phi(x; z) > 0$.

These properties of the constraints mean that in practice, 
we cannot solve the Lagrangian from Equation \ref{eqn:obj}. 
As such, we instead relax the equality constraints by optimizing Equation \ref{eqn:obj},
for fixed lambdas, selected to maximize validation log-likelihood,
as commonly done in the generative modeling literature~\citep{Zhao2018}.
These properties of the constraints also have two implications on the evaluation of NCAI (Section \ref{sec:exp}):
(1) quantitatively one should not expect NCAI to minimize both constraints all the way to 0, 
(2) qualitatively the means of $q_\phi(z | x, y)$, while not independent of $(x, y)$, 
should still exhibit less dependence than simply having memorized the data.
In Section \ref{sec:exp}, we show that using NCAI, the constraints are better satisfied than when using
vanilla mean-field VI (MFVI), and as a consequence the learned posterior predictives generalize better. 

\paragraph{Differentiable and Easy-to-Optimize Proxies for the Two Constraints.}
While the KL-divergence in both constraints may seem like a differentiable, easy-to-optimize choice,
we find that in practice it is in fact not.
In Appendix~\ref{sec:training-with-ncai}, we describe why mutual information in assumption 
(1) is intractable to compute using KL-divergence;
we empirically show how traditional divergences (e.g. Jensen-Shannon, reverse/forward-KL divergences)
in assumption (2) can all be trivially minimized by inflating the variational variances;
finally, we describe our choices of differential and tractable proxies for these constraints that do not suffer
from these issues. 
To encourage the aggregated posterior to match the prior over $z$, 
we propose a proxy based on the Henze-Zirkler statistical test for Gaussianity~\citep{hztest1990}. 
To minimize $I_\phi(x; z)$, we instead propose a proxy that penalizes the linear correlation between $x$ and $z$,
and that penalizes non-linear correlation between $x$ and $z$ by penalizing 
the linear correlation between $y$ (which depends non-linearly on $x$) and $z$.
See Appendix~\ref{sec:training-with-ncai} for the full definitions and justifications of both proxies.
While empirically effective (later shown in Section \ref{sec:exp}), 
these proxies are nonetheless somewhat heuristic in nature.
We therefore consider this method as a demonstration of validity of our theoretical analysis.

\begin{figure*}[p]
    \centering
    
    \begin{subfigure}[t]{0.01\textwidth}
        \centering
        \small
    \end{subfigure}%
    ~ 
    \begin{subfigure}[t]{0.19\textwidth}
        \centering
        \small
        \hspace{6.0mm}BNN
        \vspace{1mm}
    \end{subfigure}%
    ~ 
    \begin{subfigure}[t]{0.373\textwidth}
        \centering
        \small
        \hspace{8.5mm}BNN+LV with MFVI
        \vspace{1mm}    
    \end{subfigure}
   ~ 
    \begin{subfigure}[t]{0.373\textwidth}
        \centering
        \small
        \hspace{8.0mm}BNN+LV with $\text{NCAI}_\lambda$
        \vspace{1mm}    
    \end{subfigure}
    
    \begin{subfigure}[t]{0.01\textwidth}
        \centering
        \small
        \rotatebox[origin=l]{90}{\hspace{22pt}Goldberg}
    \end{subfigure}%
    ~ 
    \begin{subfigure}[t]{0.19\textwidth}
        \centering
        \includegraphics[width=1.0\textwidth]{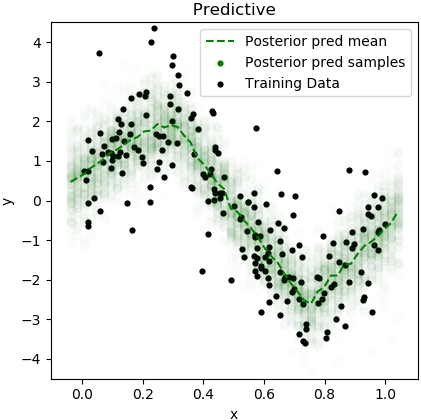} 
    \end{subfigure}%
    ~ 
    \begin{subfigure}[t]{0.373\textwidth}
        \centering
         \includegraphics[width=1.05\textwidth]{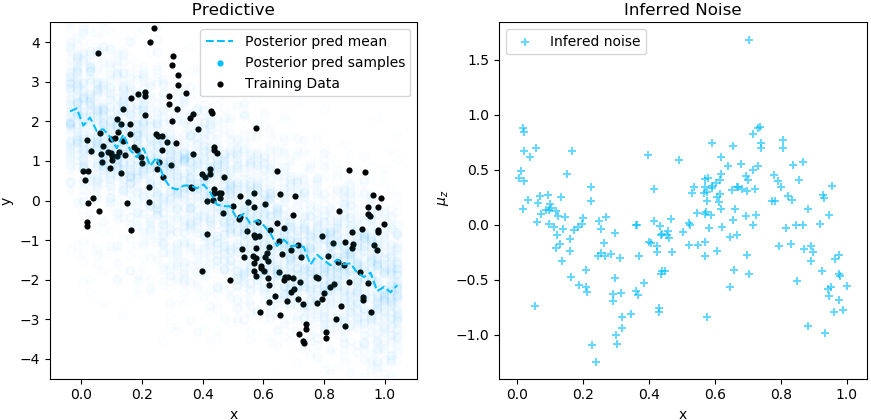} 
    \end{subfigure}
   ~ 
    \begin{subfigure}[t]{0.373\textwidth}
        \centering
         \includegraphics[width=1.05\textwidth]{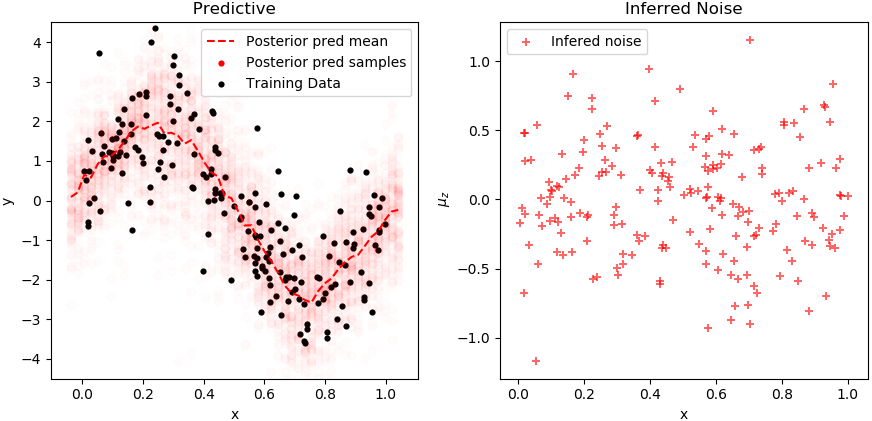}
    \end{subfigure}
    
     \begin{subfigure}[t]{0.01\textwidth}
        \centering
        \small
        \rotatebox[origin=l]{90}{\hspace{31pt}Lidar}
    \end{subfigure}%
    ~ 
    \begin{subfigure}[t]{0.19\textwidth}
        \centering
        \includegraphics[width=1.0\textwidth]{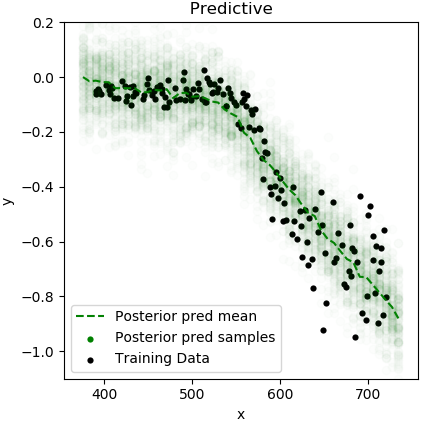} 
    \end{subfigure}%
    ~ 
    \begin{subfigure}[t]{0.373\textwidth}
        \centering
         \includegraphics[width=1.05\textwidth]{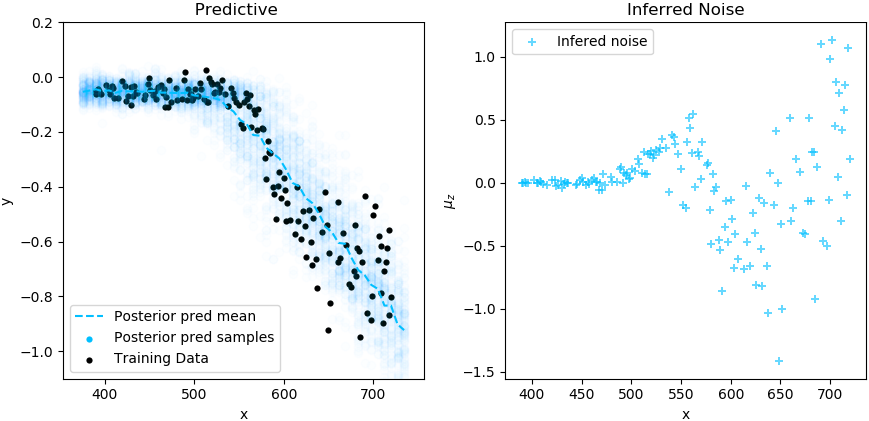} 
    \end{subfigure}
   ~ 
    \begin{subfigure}[t]{0.373\textwidth}
        \centering
         \includegraphics[width=1.05\textwidth]{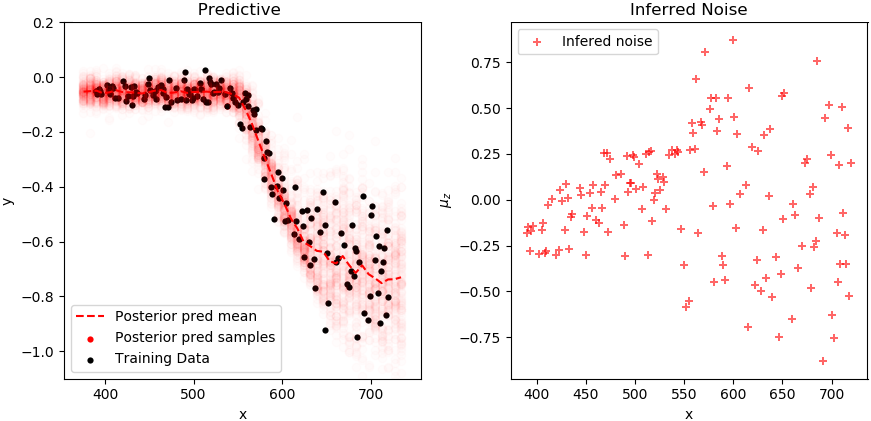}
    \end{subfigure}
    
     \begin{subfigure}[t]{0.01\textwidth}
        \centering
        \small
        \rotatebox[origin=l]{90}{\hspace{21pt}Williams}
    \end{subfigure}%
    ~ 
     \begin{subfigure}[t]{0.19\textwidth}
        \centering
        \includegraphics[width=1.0\textwidth]{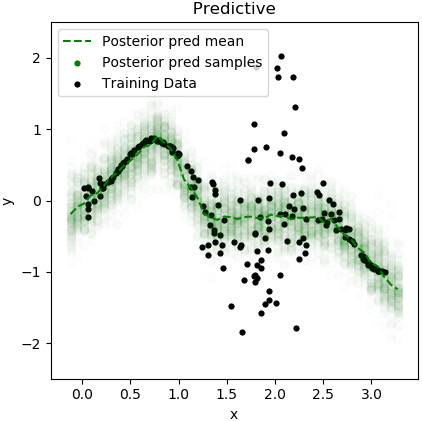} 
    \end{subfigure}%
    ~ 
    \begin{subfigure}[t]{0.373\textwidth}
        \centering
         \includegraphics[width=1.05\textwidth]{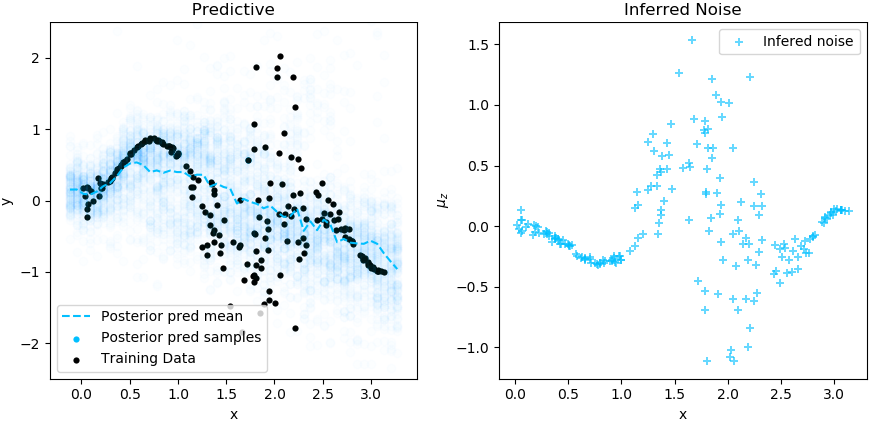} 
    \end{subfigure}
   ~ 
    \begin{subfigure}[t]{0.373\textwidth}
        \centering
         \includegraphics[width=1.05\textwidth]{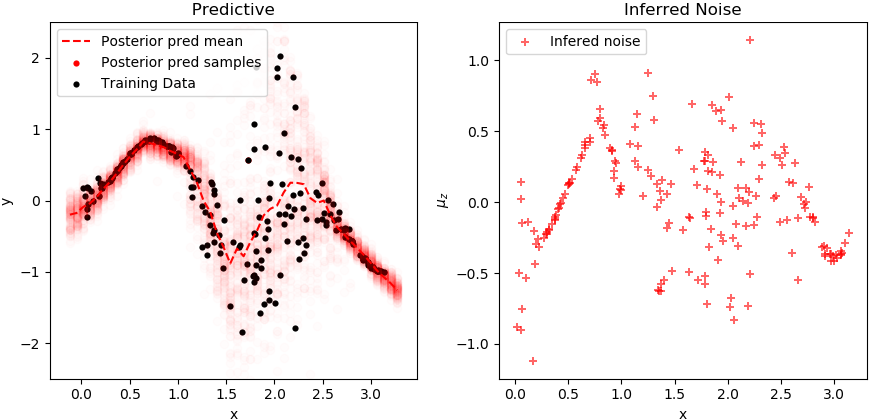}
    \end{subfigure}
    
   \begin{subfigure}[t]{0.01\textwidth}
        \centering
        \small
        \rotatebox[origin=l]{90}{\hspace{30pt}Yuan}
    \end{subfigure}%
    ~ 
    \begin{subfigure}[t]{0.19\textwidth}
        \centering
        \includegraphics[width=1.0\textwidth]{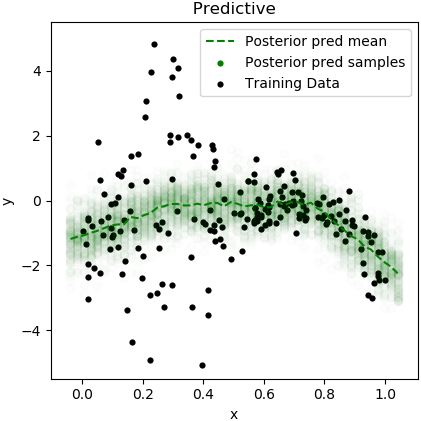} 
    \end{subfigure}%
    ~ 
    \begin{subfigure}[t]{0.373\textwidth}
        \centering
         \includegraphics[width=1.05\textwidth]{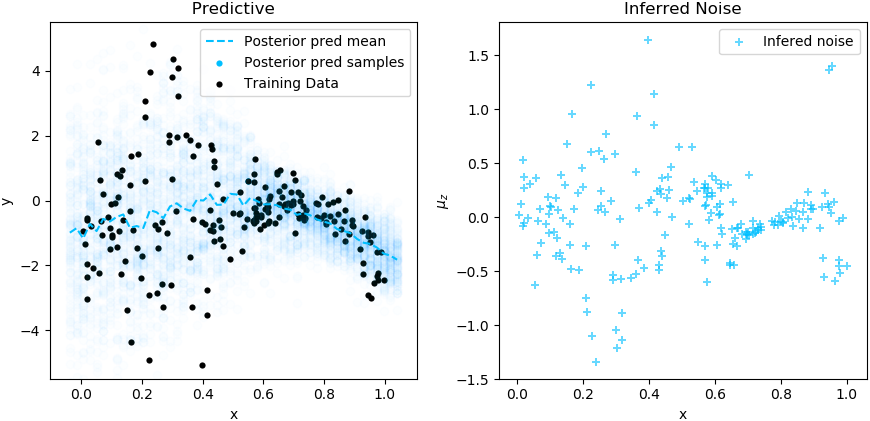} 
    \end{subfigure}
   ~ 
    \begin{subfigure}[t]{0.373\textwidth}
        \centering
         \includegraphics[width=1.05\textwidth]{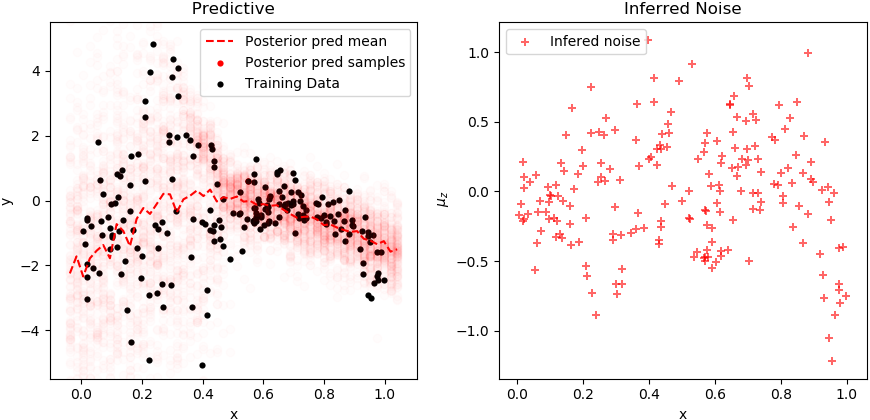}
    \end{subfigure}
   
   \begin{subfigure}[t]{0.01\textwidth}
        \centering
        \small
        \rotatebox[origin=l]{90}{\hspace{17pt}Heavy-Tail}
    \end{subfigure}%
    ~ 
     \begin{subfigure}[t]{0.19\textwidth}
        \centering
        \includegraphics[width=1.0\textwidth]{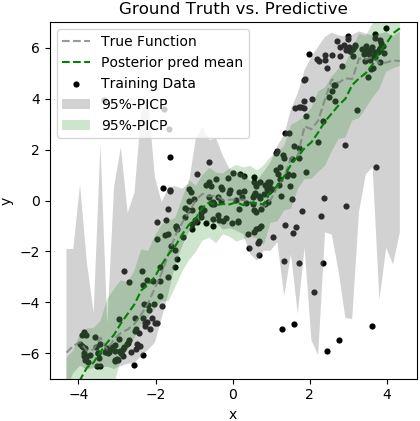} 
    \end{subfigure}%
    ~ 
    \begin{subfigure}[t]{0.373\textwidth}
        \centering
         \includegraphics[width=1.05\textwidth]{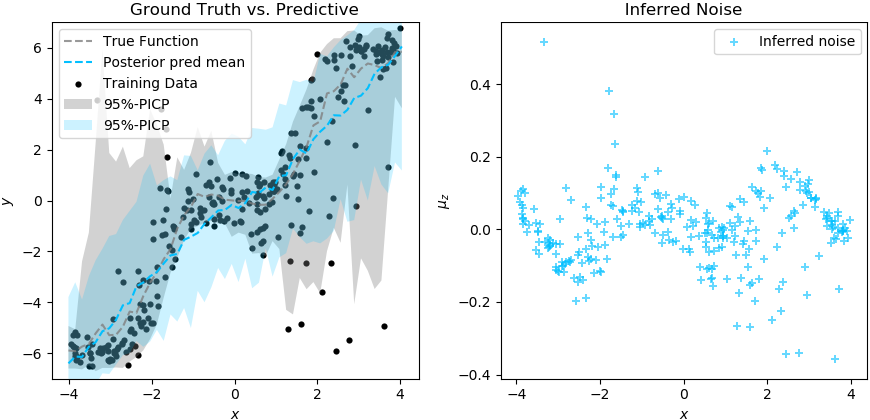} 
    \end{subfigure}
   ~ 
    \begin{subfigure}[t]{0.373\textwidth}
        \centering
         \includegraphics[width=1.05\textwidth]{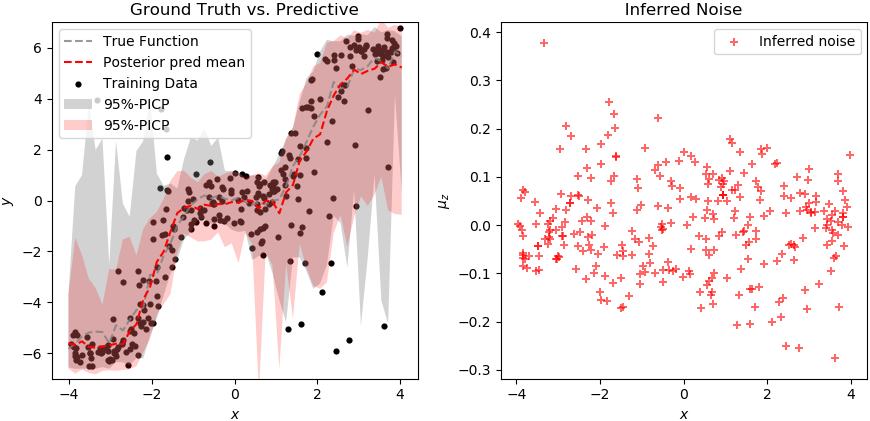}
    \end{subfigure}
    
   \begin{subfigure}[t]{0.01\textwidth}
        \centering
        \small
        \rotatebox[origin=l]{90}{\hspace{22pt}Depeweg}
    \end{subfigure}%
    ~ 
    \begin{subfigure}[t]{0.19\textwidth}
        \centering
        \includegraphics[width=1.0\textwidth]{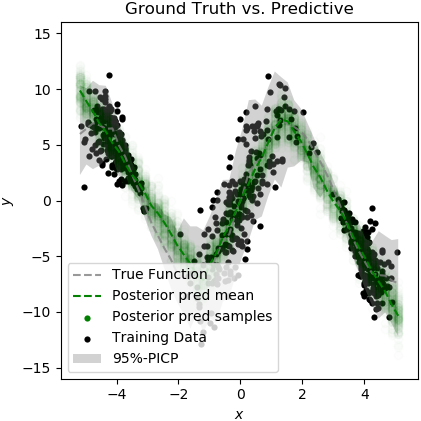} 
    \end{subfigure}%
    ~ 
    \begin{subfigure}[t]{0.373\textwidth}
        \centering
         \includegraphics[width=1.05\textwidth]{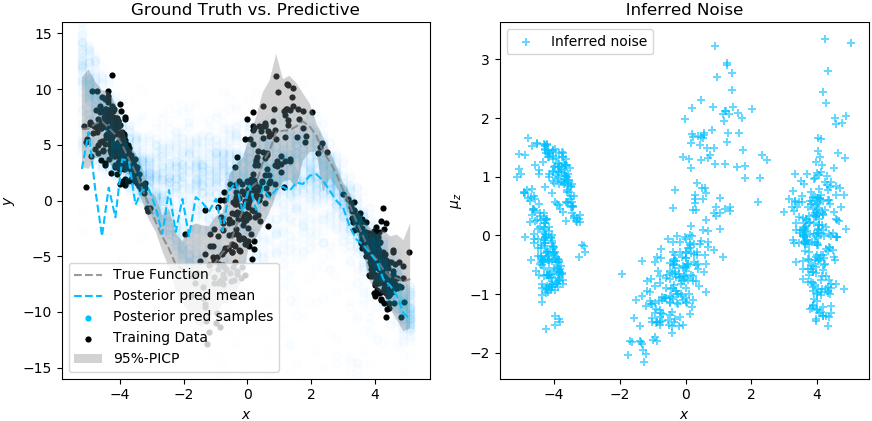} 
    \end{subfigure}
   ~ 
    \begin{subfigure}[t]{0.373\textwidth}
        \centering
         \includegraphics[width=1.05\textwidth]{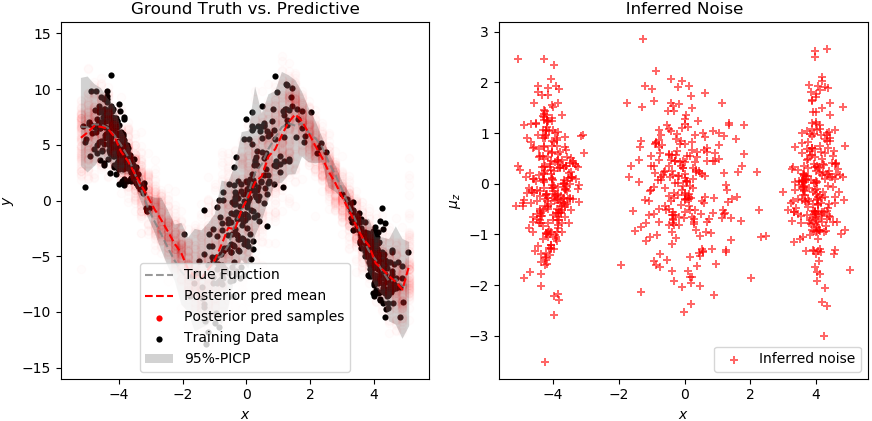}
    \end{subfigure}
    
    \caption{\textbf{Comparison of posterior predictives.} BNN (green) captures trend but underestimates variance; BNN+LV with mean-field VI (blue) captures more variance, but infers $z$'s that are dependent on the data, shown in scatter plot. BNN+LV with $\text{NCAI}_\lambda$ (red) best captures heteroscedasticity and infers $z$'s that best resemble white noise (shown in scatter plot). For the bottom two rows, since the true function is known, it is visualized in gray.}
    \label{fig:1d-qualitative}
\end{figure*}


\section{Experiments} \label{sec:exp}

In this section, we demonstrate that on a wide range of data-sets,
mean-field variational inference for BNN+LV suffers from the theoretical issues we identify in Section \ref{sec:model}. 
We also demonstrate that NCAI improves the quality of approximate inference for this class of models. That is, 
we show that latent inputs inferred by NCAI better satisfy modeling assumptions, and as a result, 
the learned posterior predictives are closer to the ground-truth data-generating models---they generalize better and do not misestimate uncertainty.

\subsection{Setup} \label{sec:exp-setup}

\paragraph{Data-sets.}
We consider 5 synthetic data-sets that are frequently used in heteroscedastic regression literature 
(Goldberg, Williams, Yuan, Depeweg and Heavy Tail),
as well as 6 real data-sets with different patterns of epistemic and aleatoric uncertainty
(Lidar, Yacht, Energy Efficiency, Airfoil, Abalone, Wine Quality Red),
all of which are described in Appendix \ref{sec:datasets}. 
Each data-set is split into 5 random train/validation/test sets. 
For every split of each data-set, each method is evaluated on the best-learned posterior predictive 
(according validation log-likelihood) out of 10 random restarts (see Appendix \ref{sec:exp_details} for details). 

\paragraph{Experimental Setup.} We use neural networks with LeakyReLU activations with $\alpha=0.01$ in all experiments. 
We set the prior variances $\sigma_z^2, \sigma_w^2$ using empirical Bayes
and grid-search over remaining hyper-parameters. 
For optimization, we use Adam \citep{kingma2014adam} with a learning rate of $0.01$,
train for 30,000 epochs (and verify convergence). 
Lastly, for each method, we select the best hyper-parameters using the 
average log-likelihood on the validation set. 
Full details are in Appendix \ref{sec:exp_setup}. 

\paragraph{Baselines.} We compare NCAI on BNN+LV with unconstrained mean-field VI~\citep{blundell2015weight}. We also compare selecting constraint strength parameters, $\lambda_1, \lambda_2$, of NCAI through cross-validation (denoted $\text{NCAI}_{\lambda}$) against fixing $\lambda_1, \lambda_2$ at zero (denoted $\text{NCAI}_{\lambda=0}$).  Finally, we compare the performance of BNN+LV (for all inference methods) with that of a BNN. 

\paragraph{Evaluation.} 
We evaluate the learned posterior predictives for quality of fit using test average log-likelihood, RMSE, 
and calibration of the posterior predictives 
(using the $95\%$ Prediction Interval Coverage Probability (PICP), and the $95\%$ Mean Prediction Interval Width (MPIW)). 
We also check whether they satisfy the two generative modeling assumptions from Section \ref{sec:assumption-violation}.
Specifically, we check whether $z$ is independent of $x$ under the posterior using mutual information 
(estimated via a non-parametric nearest neighbor method ~\citep{kraskov2004}),
and we check that $z$ is Gaussian under the posterior predictive using the Henze-Zirkler test-statistic for normality
(as well as using Jensen-Shannon, and forward/reverse KL divergences between the aggregated posterior and the prior). 
We note that due to the reasons mentioned in Appendix \ref{sec:training-with-ncai}---that traditional variances are trivially minimized by inflating the variational variances---we primarily focus our analysis on the Henze-Zirkler metric during evaluation,
even though it is also used in our objective (Equation \ref{eq:full-ncai-objective}).
Details about evaluation metrics are found in Appendix \ref{sec:eval}.

\begin{table*}[t!]
\setlength{\tabcolsep}{3pt}
\small
\centering
\begin{tabular}{l||ccccc}
\multicolumn{6}{c}{\emph{Mutual Information between $x$ and $z$ (Synthetic Data)}}\\
Inference                          & Heavy Tail         & Goldberg           & Williams           & Yuan               & Depeweg             \\
\hline \hline
\textbf{MFVI}                                & $0.243 \pm 0.079$  & $0.229 \pm 0.113$  & $0.982 \pm 0.121$  &               $\bm{0.24 \pm 0.129}$   & $0.428 \pm 0.04$    \\ 
\hline
\textbf{$\text{NCAI}_{\lambda=0}$}  & $0.051 \pm 0.049$  & $\bm{0.02 \pm 0.024}$   & $\bm{0.519 \pm 0.091}$  & $0.283 \pm 0.112$  & \bm{$0.032 \pm 0.017$}   \\
\textbf{$\text{NCAI}_\lambda$}      & $\bm{0.036 \pm 0.04}$   & $0.046 \pm 0.067$  & $\bm{0.519 \pm 0.091}$  & $0.283 \pm 0.112$  & \bm{$0.032 \pm 0.017$}
\end{tabular}
\vskip0.1cm
\caption{Comparison of  \emph{mutual information} between $z$ and $x$ on synthetic data-sets ($\pm$ std). Across all but one of the data-sets, $\text{NCAI}_\lambda$ training infers $z$'s that have the least mutual information in comparison mean-field VI (MFVI). Additional evaluations of model assumption satisfaction are in Appendix \ref{sec:results_tab}.}
\label{table:synth_metrics}
\normalsize
\end{table*}

\begin{table*}[t!]
\setlength{\tabcolsep}{3pt}
\small
\centering
\begin{tabular}{l||ccccc}
\multicolumn{6}{c}{\emph{Henze-Zirkler Test-Statistic (Synthetic Data)}}\\
Inference                                    & Heavy Tail         & Goldberg           & Williams           & Yuan               & Depeweg             \\
\hline \hline
\textbf{MFVI}                     & $4.701 \pm 5.439$            & $0.918 \pm 0.41$                   & \bm{$6.445 \pm 2.818$}  & \bm{$5.252 \pm 5.607$}  & $6.408 \pm 2.439$   \\ 
\hline
\textbf{$\text{NCAI}_{\lambda=0}$}  & $7.137 \pm 5.436$  & $0.621 \pm 0.234$              & $7.248 \pm 2.598$  & $8.091 \pm 5.185$  & $\bm{0.792 \pm 0.357$}   \\
\textbf{$\text{NCAI}_\lambda$}      & \bm{$0.027 \pm 0.011$}  & \bm{$0.026 \pm 0.038}$  & $7.248 \pm 2.598$  & $8.091 \pm 5.185$  & $\bm{0.792 \pm 0.357$}  
\end{tabular}
\vskip0.1cm
\caption{Comparison of model assumption satisfaction (in terms of the HZ metric) on synthetic data-sets ($\pm$ std). Across all but one data-sets, $\text{NCAI}_\lambda$ training infers $z$'s that are more Gaussian (lowest HZ) relative to mean-field VI (MFVI).}
\label{table:synth_metrics2}
\end{table*}

\begin{table*}[t!]
\setlength{\tabcolsep}{1.5pt}
\scriptsize
\centering
\begin{tabular}{ll||ccccc}
\multicolumn{6}{c}{\emph{Test Log-Likelihood (Synthetic Data)}}\\
Model & Inference                                    & Heavy Tail               & Goldberg            & Williams            & Yuan                & Depeweg              \\
\hline \hline
\textbf{BNN} & \textbf{MFVI}                       & $-2.47 \pm 0.083$        & $-1.055 \pm 0.08$   & $-1.591 \pm 0.417$  & $-2.846 \pm 0.346$  & $-2.306 \pm 0.059$   \\
\textbf{BNN+LV} & \textbf{MFVI}                     & $-1.867 \pm 0.078$       & $-1.026 \pm 0.056$  & $-1.033 \pm 0.156$  & $-1.278 \pm 0.164$  & $-2.342 \pm 0.048$   \\ 
\hline
\textbf{BNN+LV} & \textbf{$\text{NCAI}_{\lambda=0}$}  & $-1.481 \pm 0.018$       & $\bm{-0.962 \pm 0.040}$  & $\bm{-0.414 \pm 0.184}$  & $\bm{-1.211 \pm 0.083}$  & $\bm{-1.973 \pm 0.049}$   \\
\textbf{BNN+LV} & \textbf{$\text{NCAI}_\lambda$}      & $\bm{-1.426 \pm 0.042}$  & $-0.963 \pm 0.041$  & $\bm{-0.414 \pm 0.184}$  & $\bm{-1.211 \pm 0.083}$  & $\bm{-1.973 \pm 0.049}$  \\
\end{tabular}
\vskip0.1cm
\caption{Comparison of \emph{test log-likelihood} on synthetic data-sets ($\pm$ std). For all data-sets, BNN+LV trained with NCAI outperforms BNN+LV and BNN trained with mean-field VI (MFVI).}
\label{table:synth_gen_ll}
\end{table*}
\normalsize

\begin{table*}[t!]
\begin{adjustwidth}{-.5in}{-.5in}  
\setlength{\tabcolsep}{1.5pt}
\ssmall
\centering
\begin{tabular}{ll||ccccccc}         
\multicolumn{7}{c}{\emph{\normalsize{Test Log-Likelihood (Real Data)}}}\\                                                                                             
Model & Inference                                   & Abalone                   & Airfoil                   & Energy                   & Lidar                    & Wine                      & Yacht                    \\ \hline \hline
\textbf{BNN} & \textbf{MFVI}                      & $-1.248 \pm 0.153$        & $-0.995 \pm 0.143$        & \bm{{$1.281 \pm 0.171$}} & $-0.31 \pm 0.069$        & $-1.143 \pm 0.027$        & $0.818 \pm 0.187$        \\
\textbf{BNN+LV}  & \textbf{MFVI}                   & $-0.843 \pm 0.071$        & $-0.512 \pm 0.083$        & $0.573 \pm 0.288$        & $0.129 \pm 0.131$        & $-1.709 \pm 0.22$         & $0.638 \pm 0.121$        \\ \hline
\textbf{BNN+LV}  &  \textbf{$\text{NCAI}_{\lambda=0}$} & \bm{{$-0.831 \pm 0.086$}} & \bm{{$-0.462 \pm 0.056$}} & $0.862 \pm 0.138$        & \bm{{$0.269 \pm 0.107$}} & $-1.147 \pm 0.025$        & \bm{{$0.832 \pm 0.077$}} \\
\textbf{BNN+LV}  &  \textbf{$\text{NCAI}_\lambda$}     & \bm{{$-0.831 \pm 0.086$}} & \bm{{$-0.462 \pm 0.056$}} & $0.898 \pm 0.452$        & $0.263 \pm 0.11$         & \bm{{$-0.849 \pm 0.038$}} & \bm{{$0.832 \pm 0.077$}}
\end{tabular}
\end{adjustwidth}
\caption{Comparison of \emph{test log-likelihood} on real data-sets ($\pm$ std). For all but one of the data-sets, BNN+LV trained with NCAI outperforms BNN+LV and BNN trained with mean-field VI (MFVI).}
\label{table:real_gen_ll}
\end{table*}
\normalsize

\subsection{Results}

\paragraph{The BNN+LV joint posterior prefers models that do not explain the observed data well.}
In Figure \ref{fig:map-pathology}, we compare the ground-truth function and 
the function learned via the MAP estimate over $W, Z$.
We show that in comparison to the true parameters $W^\text{true}, Z^\text{true}$, 
the MAP-inferred parameters $W^\text{MAP}, Z^\text{MAP}$
are (1) scored \emph{significantly} higher under the log-posterior, 
(2) have $Z^\text{MAP}$ that violate our modeling assumptions (Section \ref{sec:assumption-violation}),
and (3) have $W^\text{MAP}$ that generalize poorly and misestimates aleatoric uncertainty.

\paragraph{NCAI better satisfies generative modeling assumptions than mean-field VI.}
As shown in Tables \ref{table:synth_metrics} and \ref{table:real_metrics1}, 
across all synthetic and real data-sets, NCAI learns posterior predictives
for which $z$ and $x$ are significantly less dependent
under the posterior than those learned via mean-field VI, thereby satisfying assumption 1.
As shown in Tables \ref{table:synth_metrics2} and \ref{table:real_metrics2}, across all synthetic and real data-sets, 
NCAI learns posterior predictives for which the aggregated posterior better matches the prior, 
thereby satisfying assumption 2.
Furthermore, on synthetic data, Figure \ref{fig:1d-qualitative} confirms qualitatively from 
scatter plots of variational posterior mean of $z$ vs. $x$ that NCAI better satisfies modeling assumptions.
Additional evaluation metrics of model assumption satisfaction 
can be found in Appendix \ref{sec:results_tab} (described in Appendix \ref{sec:eval}). 

\paragraph{Learned posterior predictives perform better when generative model assumptions are satisfied.}
On all 1-dimensional data-sets, 
Figure \ref{fig:1d-qualitative} shows a qualitative comparison of the posterior predictive distributions of 
BNN+LV trained with $\text{NCAI}_\lambda$ compared with benchmarks.
We see that, as expected, BNNs underestimate the posterior predictive uncertainty, 
whereas BNN+LV with mean-field VI improves upon the BNN in terms of log-likelihood 
by expanding posterior predictive uncertainty nearly symmetrically about the predictive mean. 
The predictive distribution obtained by BNN+LV trained with NCAI, however, 
captures the asymmetry of the observed heteroscedasticity. 
Furthermore, its predictive mean better captures the overall trend and its predictive uncertainty is better-calibrated.

Across all synthetic and real data-sets (apart from Energy Efficiency), 
when modeling assumptions are satisfied (i.e. when NCAI is used), 
the learned posterior predictives have higher average log-likelihood 
(Tables \ref{table:synth_gen_ll}, \ref{table:real_gen_ll} for synthetic and real, respectively).
On Energy Efficiency, the BNN performs best in terms of test log-likelihood,
but drastically underestimates the uncertainty in the data;
specifically, the $95\%$-PICP and MPIW show that BNN has a small predictive interval width 
that only covers about $80\%$ of the data, whereas NCAI overs about $94\%$ 
of the data (see Table \ref{res:ee} details). 
This is because the BNN, when properly trained, is able to capture the trends in the data 
but tends to underestimate the variance (log-likelihood and calibration)---this tendency
is especially apparent in the presence of heteroscedastic noise. 

\paragraph{Selecting between $\text{NCAI}_{\lambda=0}$ and $\text{NCAI}_{\lambda>0}$.}
Generally, we observe that on data-sets in which the noise is roughly symmetric around the posterior predictive mean
(as in the Goldberg, Yuan, Williams, Lidar, and Depeweg data-sets) 
$\text{NCAI}_{\lambda=0}$ and $\text{NCAI}_{\lambda>0}$ perform comparably well on average test log-likelihood.
However, when the noise is skewed around the posterior predictive mean (like in the HeavyTail data-set), 
we find that $\text{NCAI}_{\lambda>0}$ out-performs $\text{NCAI}_{\lambda=0}$.
This is because $\text{NCAI}_{\lambda=0}$ first fits the variational parameters of the weights to best capture
as best as possible, often fitting a function that represents the mean.
After the warm-start, when training with respect to the variational parameters of the $z$'s,
the uncertainty is increased about the mean to best capture the data, often in a way that does not significantly alter the parameters of the weights, thereby resulting in a posterior predictive with symmetric noise.

\subsection{Application: Uncertainty Decomposition}
By explicitly modeling sources of epistemic uncertainty, $W$, and aleatoric noise, $z$ and $\epsilon$, 
the BNN+LV model can decompose the uncertainty in its posterior predictive distribution. 
This decomposition can improve performance on downstream tasks that rely on exploiting uncertainty in data. 
For example, \cite{depeweg2018decomposition} shows that accurate decomposition improves active-learning with 
BNN+LV in the presence of complex noise; the authors also formulate a new `risk-sensitive criterion' 
for safe model-based RL based on the decomposition of predictive uncertainties in BNN+LV. 

Following \cite{depeweg2018decomposition}, we quantify the uncertainty in the posterior predictive using entropy 
(see Equations \ref{eq:aleatoric} and \ref{eq:epistemic}). 
Using Hamiltonian Monte Carlo (HMC) \citep{Neal2012} as the ``gold-standard'' approximation of the true posterior, 
we compare the uncertainty decomposition learned by BNN+LV and NCAI with that learned by HMC.
Figure \ref{fig:decomp} shows that like HMC, 
NCAI has appropriately high total and aleatoric uncertainties at $x$'s for which there is a 
high variance in $y$, as well as high epistemic uncertainty for $x$'s near the boundary of the data.
In contrast, BNN+LV trained via mean-field VI does not.
This is evidence that our method learns a decomposition closer to that given by the ``ground-truth'' posterior predictive.

We see that BNN+LV likelihood non-identifiability negatively impacts the accuracy of uncertainty decompositions: BNN+LV trained with mean-field VI, while able to reconstruct training data well, produces inaccurate uncertainty decompositions. In contrast, NCAI consistently produces aleatoric and epistemic uncertainties that align well with those produced by HMC (more details in Appendix \ref{sec:uncertainty}).



\section{Discussion}\label{sec:discuss}

\paragraph{Non-identifiability negatively impacts inference in theory and practice.} 
In Section \ref{sec:model} we show that BNN+LV models are meaningfully non-identifiable 
under the likelihood and, as a consequence, the posterior mode is asymptotically biased towards models 
that both explain the data poorly and generalize poorly,
regardless of the choices of priors.
We argue that approximations of the joint posterior via mean-field VI
are negatively impacted by this asymptotic bias, resulting in posterior predictives that will similarly 
explain the data poorly, generalize poorly and misestimate predictive uncertainty.
In Section \ref{sec:exp} we demonstrate the negative effect on mean-field VI using a variety of synthetic and real data-sets.

\paragraph{Enforcing modeling assumptions explicitly during training mitigates the effects of non-identifiability.}
In Section \ref{sec:assumption-violation}, we show that model parameters that are scored as likely under the posterior 
often violate the generative modeling assumption---that the latent variable is generated i.i.d from the prior. 
Based on this analysis,
we develop a two-step method, NCAI, that explicitly enforces these assumptions during training.
We demonstrate on both synthetic and observed data that in enforcing these assumptions explicitly, 
we recover posterior predictives that generalize better and do not misestimate predictive uncertainty.

\paragraph{Can one alter the original BNN+LV model to correct for the asymptotic bias of the joint posterior?}
The insights from this work suggest that this is likely not possible. 
As $N$ increases, any non-degenerate prior over $z$ (that is independent of $x$ under the generative process)
will have a non-vanishing effect on the joint posterior mode.
One may be tempted to construct a prior over $z$ that weakens as $N$ increases;
however, when this prior is sufficiently weak,
this reduces to the ``incidental parameters problem''~\citep{neyman_consistent_1948, lancaster_incidental_2000}.
This work therefore highlights the more general challenge of approximate inference for 
Bayesian models that have both global parameters and local latent variables.

\paragraph{Limitations and future work.}
In this work, we prove that the joint posterior $p(W, Z | \mathcal{D})$ is biased towards models 
that generalize poorly; however, it is not currently known whether the posterior predictive of BNN+LVs
(computed using the marginal of the weights $p(W | \mathcal{D})$) is consistent.
We hope to study this in future work.
Experimentally, we show that vanilla VI suffers from both local and global optima problems
that cause the learned models to generalize poorly.
In future work, we hope to extend our analysis to other inference methods, such as MCMC-based methods.
Although in this work we proposed a new training framework, NCAI, which out-performs naive VI, 
our framework still has several limitations.
Firstly, our training objective can be challenging to optimize:
our proposed intelligent initialization alone does not always recover good quality models,
and our proposed constraints that are trained jointly with the ELBO require additional optimization
tricks to avoid local optima. 
Secondly, as we show in Section \ref{sec:method}, some of the constraints cannot be estimated without bias,
and more generally, the constraints cannot be optimized directly, requiring tractable proxies. 
In future work, we hope to both explore proxies for our constraints that are more amenable to easy 
optimization.
As such, we regard the specific instantiation of NCAI in this paper as an extension of our theoretical analysis---as a link
between the asymptotic bias of the posterior mode and the solutions returned by mean-field VI---as opposed 
to a method to be readily deployed in safety-critical applications. 
We expect that the analysis provided in this paper is useful in diagnosing poor performance of other similar models.


\section{Conclusion}\label{sec:conclude}
In this paper we identify a key issue with a promising class of flexible latent variable models for Bayesian regression---that
model non-identifiability can bias the posterior mode towards model parameters that generalize poorly. 
By analyzing the sources of non-identifiability in BNN+LV models, 
we propose an approximate inference framework, NCAI, that explicitly enforces model assumptions during training. 
On synthetic and real data-sets with complex patterns and sources of uncertainty, 
we demonstrate that NCAI better recovers posterior predictives that 
generalize well and accurately estimate uncertainty relative to baselines.


\acks{YY acknowledges support from NIH 5T32LM012411-04 and from IBM Research. WP acknowledges support from
Harvard's IACS. We thank Melanie F. Pradier, Beau Coker and Jiayu Yao for helpful feedback and discussions.}



\appendix

\addcontentsline{toc}{section}{Appendix} 
\part{Appendix} 
\parttoc 

\section{Notation} \label{sec:notation}
\clearpage
\begin{center}
\renewcommand{\arraystretch}{1.3}
\begin{longtable}{l || p{0.6\textwidth}}
\textbf{Symbol} & \textbf{Definition} \\ \hline \hline
$N$ & Number of observations (or training points) \\ \hline
$(x_n, y_n)$ & The $n$th observed input and output, where $x_n \in \mathbb{R}^D, y_n \in \mathbb{R}^L$. \\ \hline
$(x^*, y^*)$ & A point from the test-set. \\ \hline
$z_n$ & The latent code corresponding to the $n$th observation. Since the latent variables are sampled i.i.d from the prior, $p(z_n) = p(z)$. \\ \hline
$z^*$ & The latent code corresponding to $x^*$. Again, $p(z^*) = p(z)$. \\ \hline
$\mathcal{D}$ & $\{ (x_n, y_n) \}_{n=1}^N$ \\ \hline
$X$ & $\{ x_n \}_{n=1}^N$ \\ \hline
$Y$ & $\{ y_n \}_{n=1}^N$ \\ \hline
$Z$ & $\{ z_n \}_{n=1}^N$ \\ \hline
$p(y | x) p(x)$ or $p(y | x; W^\text{true}) p(x)$ & The ground-truth data-generating distribution (that generated $\mathcal{D}$). \\ \hline
$f(\cdot, \cdot; W)$ & A neural network $f$ parameterized by weights $W$. \\ \hline
$W$ & The set of all neural network weights $w_i$. \\ \hline
$p(W, Z | \mathcal{D})$ & The BNN+LV joint posterior (Equation \ref{eq:true-posterior}). \\ \hline
$p(W | \mathcal{D})$ & The marginal posterior of the weights, $\int_Z p(W, Z | \mathcal{D}) dZ$ (intractable). \\ \hline
$q_\phi(Z, W | \mathcal{D}) = q_\phi(Z | \mathcal{D}) q_\phi(W | \mathcal{D})$ & The mean-field variational approximation of the joint posterior (Equation \ref{eq:mfvf}), parameterized by $\phi$. \\ \hline
$Z^\text{MAP}, W^\text{MAP}$ & The MAP of the true joint posterior: $Z^\text{MAP}, W^\text{MAP} = \mathrm{argmax}_{Z, W} p(Z, W | \mathcal{D})$ \\ \hline
$Z^\text{true}, W^\text{true}$ & The ground-truth data-generating weights $W^\text{true}$ and latent codes $Z^\text{true} = \{ z_n^\text{true} \}_{n=1}^N$ that produced the observed data. \\ \hline
$\widehat{Z}, \widehat{W}$ & The alternative set of latent variables $\widehat{Z} = \{ \widehat{z}_n \}_{n=1}^N$ and weights, defined in Equation \ref{eq:dependence}, that have a higher probability under the joint posterior. \\ \hline
$q_\phi(z)$ & The aggregated posterior (Equation \ref{eq:aggregated-posterior}). \\ \hline
$q_\phi(z | x)$ & The approximate posterior marginalized over $y$: $\int_y q_\phi(z | x, y) p(y | x) dy$. Since we only observe one $y$ for every $x$, we cannot obtain unbiased estimates of this distribution in practice. \\
\bottomrule
\caption{\textbf{Notation}}
\label{tab:notation}
\end{longtable}
\end{center}

\section{Asymptotic Bias of the BNN+LV Joint Posterior Mode}\label{sec:proof} 
\subsection{Asymptotic Bias of 1-Node BNN+LV Posterior Mode}\label{sec:1node}

Consider univariate output generated by a single hidden-node neural network with LeakyReLU activation:
\begin{equation} \label{eq:one_node}
\begin{split}
z &\sim \mathcal{N}(0, \sigma^2_z)\\
\epsilon & \sim \mathcal{N}(0, \sigma^2_\epsilon)\\
y &=  \max \left \{W (x + z), \alpha W (x + z)\right \} + \epsilon
\end{split}
\end{equation}
where $\alpha$ is a fixed constant in $(0, 1)$. 

\ThmSingleNode*

\begin{proof}

We assume the model in Equation \ref{eq:one_node}.
We denote the prior on $W$ as $p_W$ and suppose that it is bounded, 
we denote the prior on $z$ as $p_z$,
and we suppose that $p_x$, the distribution over the input $x$, 
has bounded first and second moments $\mu_x$ and $\sigma^2_x$. 
For any non-zero constant $0 < C < 1$, we define 
\begin{align*}
\widehat{W}^{(C)} = W/C, \quad \widehat{z}_n^{(C)} = (C - 1) x_n + Cz_n,
\end{align*}
and we define $D_N^{(C)}$ to be the difference between the scaled parameters 
$(\widehat{W}^{(C)}, \widehat{z}_1^{(C)}, \ldots \widehat{z}_N^{(C)})$ 
and the ground-truth parameters $(W, z_1, \ldots z_N)$ under the joint posterior $\log p(W, Z|\mathcal{D})$. 

For any set of parameters $(W, z_1, \ldots z_N)$,
we show that we can find alternate parameters $(\widehat{W}^{(C)}, \widehat{z}_1^{(C)}, \ldots \widehat{z}_N^{(C)})$, 
that are scored as more likely under the log-posterior $\log p(W, \{z_n\} | \mathcal{D})$.
To do this, we first show that the alternative parameters are scored as equally likely under the log-likelihood,
and then we show that the alternate parameters are scored as more likely under the log-prior.

Since we have that 
\begin{equation*}
\begin{split}
\max \left \{W^{\text{true}} (x + z^{\text{true}}_n), \alpha W^{\text{true}} (x + z^{\text{true}}_n)\right \}  
= \max \left \{\widehat{W}^{(C)} (x + \widehat{z}_n^{(C)}), \alpha \widehat{W}^{(C)} (x + \widehat{z}_n^{(C)})\right \},
\end{split}
\end{equation*}
we see that the alternate parameters are as likely as the original parameters under the likelihood; that is,
\begin{align*}
\prod_n p(y_n | x_n, z^{\text{true}}_n, W^{\text{true}}) = \prod_n p\left(y_n | x_n, \widehat{z}_n^{(C)}, \widehat{W}^{(C)}\right). 
\end{align*}
Since the likelihood under both sets of parameters is equal, 
the difference between the log-posterior, $D_N^{(C)}$,
is simply the difference of the two sets of parameters under the log-prior:
\begin{align*}
D_N^{(C)} &= \log p_W(\widehat{W}^{(C)}) - \log p_W(W) + \sum_{n=1}^N \log p_z(\widehat{z}_n^{(C)}) - \log p_z(z_n).
\end{align*}

We now show that as $N \to \infty$, the probability that $(\widehat{W}^{(C)}, \widehat{z}_1^{(C)}, \ldots \widehat{z}_N^{(C)})$ is valued higher than $(W, z_1, \ldots z_N)$ in the posterior approaches 1. That is,
\begin{align*}
\lim_{N\to \infty} \mathrm{Pr}\left[D_N^{(C)} > 0\right]
= \lim_{N\to \infty} \mathbb{E}_{p_x p_z} \left\lbrack \mathbb{I} (D_N^{(C)} > 0) \right\rbrack = 1.
\end{align*}
Since we are only interested in the sign of $D_N^{(C)}$, we demonstrate, equivalently, that 
\begin{align*}
\lim_{N\to \infty} \mathbb{E}_{p_x p_z} \left\lbrack \mathbb{I} \left(\frac{1}{N} \cdot D_N^{(C)} > 0\right) \right\rbrack = 1.
\end{align*}

Expanding $\frac{1}{N} \cdot D_N^{(C)}$, we obtain
\begin{align*}
\frac{1}{N} \cdot D_N^{(C)} &= \underbrace{\frac{1}{N} \cdot \left( \log p_W(\widehat{W}^{(C)}) - \log p_W(W) \right)}_{L_N^{(C)}} + \underbrace{\frac{1}{N} \sum_{n=1}^N \log p_z(\widehat{z}_n^{(C)}) - \log p_z(z_n)}_{R_N^{(C)}}.
\end{align*} 
We note that as $N \to \infty$, the term $L_N^{(C)}$ approaches $0$. 
Thus, it suffices to analyze the sign of  $R_N^{(C)}$; that is, we compute $\mathrm{Pr}[R_N^{(C)} > 0]$.

Let $\mathbb{E} \left[R_N^{(C)}\right]$ and $\mathbb{V} \left[R_N^{(C)}\right]$ denote the mean and variance of the random variable $R_N^{(C)}$. We next show that $\mathbb{E} \left[R_N^{(C)}\right]$ is positive, allowing us to use Cantelli's Inequality to bound the probability
that $R_N^{(C)}$ is negative under $p_x$ and $p_z$ as follows:
\begin{align*}
\mathrm{Pr}\left[R_N^{(C)} < 0\right]
&= \mathrm{Pr}\left[R_N^{(C)} - \mathbb{E} \left[R_N^{(C)}\right]  < -\mathbb{E} \left[R_N^{(C)}\right] \right] \\
&= 1 - \mathrm{Pr}\left[R_N^{(C)} - \mathbb{E} \left[R_N^{(C)}\right]  \geq -\mathbb{E} \left[R_N^{(C)}\right]\right] \\
&\leq \frac{\mathbb{V} \left[R_N^{(C)}\right]}{\mathbb{V} \left[R_N^{(C)}\right] + \mathbb{E} \left[R_N^{(C)}\right]^2}. 
\end{align*}
As $N \to \infty$, we show that this upper bound on $\mathrm{Pr}\left[R_N^{(C)} < 0\right]$ approaches $0$.

We start by expanding $R_N^{(C)}$ as follows:
\begin{align*}
R_N^{(C)} &= \frac{1}{N} \sum_{n=1}^N \log p_z(\widehat{z}_n^{(C)}) - \log p_z(z_n) \\
&= -\frac{1}{2 \sigma^2_z} \cdot \frac{1}{N} \sum_{n=1}^N (\widehat{z}_n^{(C)})^2 - z_n^2  \\
&= \frac{1}{2 \sigma^2_z} \cdot \frac{1}{N} \sum_{n=1}^N z_n^2 - (\widehat{z}_n^{(C)})^2 \\
2 \sigma^2_z \cdot R_N^{(C)} &= \frac{1}{N} \sum_{n=1}^N z_n^2 - (\widehat{z}_n^{(C)})^2 \\
&= \frac{1}{N} \sum_{n=1}^N z_n^2 - ((C - 1) x_n + Cz_n)^2  \\
&= \frac{1}{N} \sum_{n=1}^N (2 C - C^2 - 1) \cdot x_n^2 + (1 - C^2) \cdot z_n^2 + (2C - 2C^2) \cdot x_n z_n \\
&= (2 C - C^2 - 1) \cdot \frac{1}{N} \sum_{n=1}^N x_n^2 
+ (1 - C^2) \cdot \frac{1}{N} \sum_{n=1}^N z_n^2 
+ (2C - 2C^2) \cdot \frac{1}{N} \sum_{n=1}^N x_n z_n 
\end{align*}

The variance of $R_N^{(C)}$ can be computed as follows:
\begin{align*}
\mathbb{V} \left[R_N^{(C)}\right] 
= \frac{\mathbb{V} \left[R_1^{(C)}\right]}{N}.
\end{align*}
Therefore, as $N$ increases, the variance approaches $0$ at a rate of $1 / N$.

The mean of $R_N^{(C)}$ can be written as:
\begin{align*}
2 \sigma^2_z \cdot \mathbb{E} \left[R_N^{(C)}\right] 
&= (2 C - C^2 - 1) \cdot \mathbb{E}_{p_x} \left[x^2\right] 
+ (1 - C^2) \cdot \mathbb{E}_{p_z} \left[z^2\right]
+ (2C - 2C^2) \cdot \mathbb{E}_{p_x p_z} \left[x z\right] \\
&= (2 C - C^2 - 1) \cdot (\sigma^2_x + \mu_x^2)
+ (1 - C^2) \cdot \sigma^2_z
+ (2C - 2C^2) \cdot \mu_x \cdot 0 \\
&= (2 C - C^2 - 1) \cdot (\sigma^2_x + \mu_x^2)
+ (1 - C^2) \cdot \sigma^2_z
\end{align*}
Thus, the mean is positive when
\begin{align*}
0 < \frac{\sigma^2_x + \mu_x^2 - \sigma^2_z}{\sigma^2_x + \mu_x^2 + \sigma^2_z} < C < 1,
\end{align*}
which is satisfied whenever $\sigma^2_x + \mu_x^2 > \sigma^2_z$. 

Now, by Cantelli's Inequality, we have that:
\begin{align*}
\mathrm{Pr}\left[R_N^{(C)} < 0\right] 
\leq \frac{\mathbb{V} \left[R_N^{(C)}\right]}{\mathbb{V} \left[R_N^{(C)}\right] + \mathbb{E} \left[R_N^{(C)}\right]^2}
= \frac{\mathbb{V} \left[R_1^{(C)}\right]}{\mathbb{V} \left[R_1^{(C)}\right] + N \cdot \mathbb{E} \left[R_N^{(C)}\right]^2},
\end{align*}
which shows that $\mathrm{Pr}\left[R_N^{(C)} < 0\right]$ approaches 0 as $N \to \infty$.

\end{proof}

\subsection{Asymptotic Bias of 1-Layer BNN+LV Posterior Mode}\label{sec:1layer_proof}

Let $W$ be the set of network weights $\{W^x, W^z, W^{\text{out}}, b^\text{hidden}, b^\text{out}\}$.
Consider a data-set generated from the following model:
\begin{equation} \label{eq:1-layer}
\begin{split}
W &\sim p(W) \\
z &\sim \mathcal{N}(0, \sigma^2_z \cdot I)\\
\epsilon & \sim N(0, \sigma^2_\epsilon)\\
a^\text{hidden} &=g \left(W^x x + W^z z + b^\text{hidden}\right),\\
y &=  (a^\text{hidden})^\top W^{\text{out}} + b^\text{out} + \epsilon
\end{split}
\end{equation}

\ThmSingleLayer*

\begin{proof}

The proof follows in the same fashion as the one for Theorem \ref{thm:post}.
We assume the model in Equation \ref{eq:1-layer}. 
For a given $W$, let $\widehat{W}$ denote the set 
$\{\widehat{W}^x, \widehat{W}^z, W^{\text{out}}, \widehat{b}^\text{hidden}, b^\text{out}\}$,
where
\begin{equation*}
\begin{split}
\widehat{W}^x = W^x + W^zS,\quad \widehat{W}^z =R, \quad
\widehat{z} = Tz - TSx - U, \quad \widehat{b}^\text{hidden} = b + RU,
\end{split}
\end{equation*}
for any choice of diagonal matrix $S \in \mathbb{R}^{D\times D}$, vector $U\in \mathbb{R}^D$, 
and any factorization $W^z = RT$ where $T \in \mathbb{R}^{D\times D}$.

For any set of parameters $(W, z_1, \ldots z_N)$,
we show that we can find alternate parameters $(\widehat{W}, \widehat{z}_1 \ldots \widehat{z}_N)$, 
that are scored as more likely under the log-posterior $\log p(W, \{z_n\} | \mathcal{D})$.
We define $D_N$ to be the difference between the scaled parameters 
$(\widehat{W}, \widehat{z}_1, \ldots \widehat{z}_N)$ 
and the ground-truth parameters $(W, z_1, \ldots z_N)$ under the joint posterior $\log p(W, Z|\mathcal{D})$. 
Since $f(x_n, z_n; W) = f(x_n,  \widehat{z}_n; \widehat{W})$ by construction, 
the alternate parameters are equally as likely as the original parameters under the likelihood:
$p(y_n | x_n, z_n, W) = p(y_n | x_n,  \widehat{z}_n, \widehat{W})$.
As such, $D_N$ is simply the difference of the two sets of parameters under the log-prior:
\begin{align*}
D_N &= \log p_W(\widehat{W}) - \log p_W(W) + \sum_{n=1}^N \log p_z(\widehat{z}_n) - \log p_z(z_n).
\end{align*}
We now show that as $N \to \infty$, the probability that $(\widehat{W}, \widehat{z}_1, \ldots \widehat{z}_N)$ is valued higher than $(W, z_1, \ldots z_N)$ in the posterior approaches 1. That is,
\begin{align*}
\lim_{N\to \infty} \mathrm{Pr}\left[D_N > 0\right]
= \lim_{N\to \infty} \mathbb{E}_{p_x p_z} \left\lbrack \mathbb{I} (D_N > 0) \right\rbrack = 1.
\end{align*}
Since we are only interested in the sign of $D_N$, we demonstrate, equivalently, that 
\begin{align*}
\lim_{N\to \infty} \mathbb{E}_{p_x p_z} \left\lbrack \mathbb{I} \left(\frac{1}{N} \cdot D_N > 0\right) \right\rbrack = 1.
\end{align*}

Expanding $\frac{1}{N} \cdot D_N$, we obtain
\begin{align*}
\frac{1}{N} \cdot D_N &= \underbrace{\frac{1}{N} \cdot \left( \log p_W(\widehat{W}) - \log p_W(W) \right)}_{L_N} + \underbrace{\frac{1}{N} \sum_{n=1}^N \log p_z(\widehat{z}_n) - \log p_z(z_n)}_{R_N}.
\end{align*} 
We note that as $N \to \infty$, the term $L_N$ approaches $0$. 
Thus, it suffices to analyze the sign of  $R_N$; that is, we compute $\mathrm{Pr}[R_N > 0]$.

Let $\diag(t)$ represent a diagonal matrix with diagonal elements $t_1, \dots, t_D$.
For simplicity, we set $U = 0$, $S = I$, $R = W^z T^{-1}$ and $T = \diag(t)$ (where $t_d \neq 0$), giving us:
\begin{align*}
\widehat{W}^x = W^x + W^z,\quad \widehat{W}^z = W^z T^{-1}, \quad
\widehat{z} = T \cdot z - T \cdot x, \quad \widehat{b}^\text{hidden} = b.
\end{align*}
Following the proof for Theorem \ref{thm:post}, it suffices to show that the mean of
$R_N$ is positive and that the its variance approaches $0$ at a rate of $1 / N$ in order 
to prove that $\mathrm{Pr}\left[R_N < 0\right]$ approaches $0$ as $N \to \infty$. 

For this model, $R_N$ is defined as:
\begin{align*}
R_N &= \frac{1}{N} \sum_{n=1}^N \log p_z(\widehat{z}_n) - \log p_z(z_n) \\
2 \sigma^2_z \cdot R_N &= \frac{1}{N} \sum_{n=1}^N z_n^\intercal z_n - (T z_n - T x_n)^\intercal (T z_n - T x_n) \\
&= \frac{1}{N} \sum_{n=1}^N \sum\limits_{d=1}^D -t_d^2 \cdot x_{n,d}^2 + 2 t_d^2 \cdot x_{n,d} \cdot z_{n,d} - t_d^2 z_{n,d}^2 + z_{n,d}^2 \\
&= \frac{1}{N} \sum_{n=1}^N \sum\limits_{d=1}^D -t_d^2 \cdot x_{n,d}^2 + 2 t_d^2 \cdot x_{n,d} \cdot z_{n,d} + (1 - t_d^2) \cdot z_{n,d}^2
\end{align*}
The mean of $R_N$ can be computed as follows:
\begin{align*}
 \mathbb{E} \left[R_N\right] 
 &= \sum\limits_{d=1}^D -t_d^2 \cdot \mathbb{E}_{p_{x_d}} \left\lbrack x^2 \right\rbrack + 2 t_d^2 \cdot \mathbb{E}_{p_{x_d}} \left\lbrack x \right\rbrack \cdot \mathbb{E}_{p_{z_d}} \left\lbrack z \right\rbrack + (1 - t_d^2) \cdot \mathbb{E}_{p_{z_d}} \left\lbrack z^2 \right\rbrack \\
 &= \sum\limits_{d=1}^D -t_d^2 \cdot \mathbb{E}_{p_{x_d}} \left\lbrack x^2 \right\rbrack + 2 t_d^2 \cdot \mathbb{E}_{p_{x_d}} \left\lbrack x \right\rbrack \cdot 0 + (1 - t_d^2) \cdot \mathbb{E}_{p_{z_d}} \left\lbrack z^2 \right\rbrack \\
&= \sum\limits_{d=1}^D -t_d^2 \cdot (\mu_{x_d}^2 + \sigma^2_{x_d}) + (1 - t_d^2) \cdot \sigma^2_{z_d} 
\end{align*}
where $p_{x_d}, p_{z_d}$ are the $d$th marginals of $p_x, p_z$,
and $\mu_{x_d}, \sigma^2_{x_d}$ are the mean and variance of the $d$th marginal of $x$. 
As such, for any $t_d$ that satisfies the following conditions,
\begin{align*}
0 < |t_d| < \sqrt{\frac{\sigma^2_z}{\mu_{x_d}^2 + \sigma^2_{x_d} + \sigma^2_z}},
\end{align*}
$\mathbb{E} \left[R_N\right]$ is positive. 
Just as in the proof for the 1-Node BNN+LV, the variance of $R_N$ can be computed as follows:
\begin{align*}
\mathbb{V} \left[R_N\right] 
= \frac{\mathbb{V} \left[R_1\right]}{N}.
\end{align*}
Therefore, as $N$ increases, the variance approaches $0$ at a rate of $1 / N$.
Now, by Cantelli's Inequality, we have that:
\begin{align*}
\mathrm{Pr}\left[R_N < 0\right] 
\leq \frac{\mathbb{V} \left[R_N\right]}{\mathbb{V} \left[R_N\right] + \mathbb{E} \left[R_N\right]^2}
= \frac{\mathbb{V} \left[R_1\right]}{\mathbb{V} \left[R_1\right] + N \cdot \mathbb{E} \left[R_N\right]^2},
\end{align*}
which shows that $\mathrm{Pr}\left[R_N^{(C)} < 0\right]$ approaches 0 as $N \to \infty$.

\end{proof}

\subsection{Additional Types of Non-Identifiability of 1-Layer BNN+LV Models}\label{sec:y_dep}
Assume that the activation function $g$ is invertible. Let $\{(x_n, y_n)\}_{n=1}^N$ be a set of observed data generated by the model parameters $W$. Define $\widehat{b}^\text{hidden}, b^\text{out}$ to be zero, $\widehat{W}^x$ to be the $H\times D$ zero matrix and $\widehat{W}^{\text{out}}$ to be the $1\times H$ matrix of consisting of $\frac{1}{DH}$ in all entries. Finally, let $\widehat{W}^z$ be a $H\times D$ matrix of 1's and let $\widehat{z}_n = g^{-1}(y_n)$. 

Then, we have that $y_n = f(x_n, \widehat{z}_n; \widehat{W})$. That is, the alternate set of model parameters $\widehat{W}^x$ reconstructs the observed data perfectly. We note that in this case, the latent noise variable $z$ is a function of the observed output $y$ and is hence dependent on the input $x$.

\section{Intractability of Marginalizing out $Z$} \label{sec:intractable-marginalization-of-z}

Instead of approximating the joint posterior $p(W, Z | \mathcal{D})$, 
one might argue that it is better to directly approximate the posterior marginal $p(W | \mathcal{D})$,
since only the posterior marginal is used in the posterior predictive (Equation \ref{eq:posterior-predictive}).
This allows us to bypass the aforementioned local optima caused by approximating $p(W, Z | \mathcal{D})$. 
However, marginalizing out $Z$ may be intractable:
\begin{align} 
\footnotesize
\begin{split}
\mathrm{argmin}_\phi D_\text{KL} \lbrack q_\phi(W | \mathcal{D}) || p(W | \mathcal{D}) \rbrack 
= \mathrm{argmin}_\phi \mathbb{E}_{q_\phi(W | \mathcal{D})} \lbrack \log \underbrace{\mathbb{E}_{p(Z)} \lbrack p(Y | X, W, Z) \rbrack}_{\text{intractable}} \rbrack - D_\text{KL} \lbrack q_\phi(W | \mathcal{D}) || p(W) \rbrack
\label{eq:marg-elbo}
\end{split}
\end{align}
In order to tractably approximate the above expectation over $Z$, 
it is natural to apply a variational lower bound, 
as commonly does in the Variational Autoencoder literature~\citep{kingma_auto-encoding_2013}:
\begin{align}
\log \mathbb{E}_{p(Z)} \lbrack p(Y | X, W, Z) \rbrack &\geq \mathbb{E}_{q_\phi(Z | \mathcal{D})} \lbrack \log p(Y | X, W, Z) \rbrack - D_\text{KL} \lbrack q_\phi(Z | \mathcal{D}) || p(Z) \rbrack.
\label{eq:ll-lower-bound}
\end{align}
By plugging Equation \ref{eq:ll-lower-bound} into Equation \ref{eq:marg-elbo}, 
we get the original mean-field ELBO in Equation \ref{eq:mf-elbo}. 
This derivation of the ELBO illuminates an additional issue:
since the marginal likelihood, $\mathbb{E}_{p(Z)} \lbrack p(Y | X, W, Z) \rbrack$, is lower-bounded,
it is no longer highest at the ground-truth weights (or at weights that parameterize equivalent functions).
As such, the $q_\phi(W | \mathcal{D})$ that maximizes the ELBO may not concentrate around functions that explain the data well.
In fact, the harder $p(Z | \mathcal{D})$ is to approximate using $q_\phi(Z | \mathcal{D})$, the larger the gap is between the bound 
and the true marginal likelihood, and the more $q_\phi(W | \mathcal{D})$ will concentrate around functions that explain the data poorly. 
We note that this issue occurs at the \emph{global} optima of the ELBO.
To alleviate this issue, in which the global optima of the ELBO prefer functions that explain the data poorly, 
one can use a more flexible variational family, leading to a tighter variational bound.
The challenge here, though, is that it is unknown on what types of data-sets it is important to model 
the posterior dependencies between the weights $W$ and latent inputs $Z$,
which result in computationally expensive estimates of the ELBO.

\FloatBarrier

\section{Choosing Differentiable Forms of the NCAI Objective} \label{sec:training-with-ncai}
Tractable training with NCAI depends on instantiating a differentiable form of the training Equation \ref {eqn:obj}. 
In the following, we consider computationally efficient proxies for the two constraints in the NCAI objective.

\subsection{Defining $I_\phi(x; z)$} \label{sec:mixz}

\paragraph{Limitations of directly using $I_\phi(x; z)$.}
If we use KL-divergence, we get the following estimator:
\begin{align}
  I_\phi(x; z) &= D_{\text{KL}}[q_\phi(z|x) p(x) \| q_\phi(z) p(x)] \\
  &= \E{q_\phi(z|x) p(x)} \left\lbrack \log \frac{q_\phi(z|x) p(x)}{q_\phi(z) p(x)} \right\rbrack \\
  &\approx \frac{1}{N} \sum\limits_{n=1}^N
  \E{q_\phi(z_n | x_n)} \left\lbrack \log \frac{q_\phi(z_n | x_n)}{\frac{1}{N} \sum\limits_{n=1}^N q_\phi(z_n | x_n, y_n)} \right\rbrack
\end{align}
where
\begin{align}
  q_\phi(z_n | x_n) &= \E{p(y_n | x_n)} \left\lbrack q_\phi(z_n | x_n, y_n) \right\rbrack 
\end{align}
The above estimator is intractable to compute for two reasons:
\begin{enumerate}
\item Estimating $q_\phi(z_n | x_n)$ requires samples from $p(y_n | x_n)$ (the posterior predictive).
To sample from $p(y_n | x_n)$ one already needs to have a good estimate of the posterior;
however, given a good estimate of $p(y_n | x_n)$, there would be no need to estimate $I_\phi(x; z)$ in the first place.
Alternatively, to sample from $p(y_n | x_n)$, one needs multiple observations of $y$ for the same $x$,
which we do not have.
As such we are forced to estimate $q_\phi(z_n | x_n)$ with the single observation we have $(x_n, y_n)$,
giving us $q_\phi(z_n | x_n, y_n)$ as the approximation.
As a result, it is not possible to estimate $I_\phi(x; z)$ without bias. 
\item The nested expectations above require too many samples for a low-variance estimator of 
the gradient of $I_\phi(x; z)$.
\end{enumerate}

\paragraph{Defining a computationally efficient, differentiable proxy.}
Instead of using a traditional divergence to estimate $I_\phi(x; z)$, we choose to minimize a proxy.
Minimizing a proxy, however, also comes with challenges.
There generally do not exist differentiable analytic forms for common measures of dependence 
such as mutual information, thus necessitating the use of upper/lower-bounds ~\citep{Dieng2018}.
Moreover, in using lower bounds (e.g. MINE~\citep{belghazi2018mine}), 
we found that even if the bound is tight for fixed variational parameters,
these estimators need to be re-trained even when 
the variational means are adjusted by a gradient step.
We found this is prohibitively expensive to compute in practice.

As such, as a proxy for $I_\phi(x; z)$, we penalize the correlation between $x$ and the variational means of $z$ 
as well as the correlation between $y$ and the variational means of $z$:
\begin{equation}
\begin{split}
&\lambda_3 \,\text{PC}(\{x_n\}, \{\mu_{z_n}\}) + \lambda_4\,  \text{PC}(\{y_n\}, \{\mu_{z_n}\})
\end{split}
\end{equation}
where $\text{PC}(\cdot, \{\mu_{z_n}\})$ is a measure of the average Pearson correlation 
between pairs of dimensions in $\{x_n\}$ or $\{y_n\}$ and the latent noise means $\{\mu_{z_n}\}$:
\begin{align}
\text{PC}(\{a_n\}, \{ b_n \}) &= \frac{1}{D} \sum\limits_{d=1}^D \frac{\text{Cov}(\{a_n^{(d)}\}, \{ b_n^{(d)} \})}{\sqrt{\mathbb{V}[\{a_n^{(d)} \}] \mathbb{V}[\{ b_n^{(d)} \}]}}.
\end{align}
Minimizing the correlation between $x$ and $z$ reduces the linear dependence 
between the input and the inferred noise, following the form of non-identifiability identified in Equation \ref{eq:dependence};
minimizing the correlation between $z$ and $y$ 
reduces the non-linear dependence between the input and the inferred noise, 
following the form of non-identifiability identified in Appendix \ref{sec:y_dep}.
The fact that both our constraint and the exact form of non-identifiability we characterize in Equation \ref{eq:dependence}
are both linear explains why this simple constraint is empirically effective.

\subsection{Defining $D\lbrack q_\phi(z) || p(z) \rbrack$} \label{eq:qzpz}

\paragraph{Limitations of directly using $D\lbrack q_\phi(z) || p(z) \rbrack$.}
Minimizing $D\lbrack q_\phi(z) || p(z) \rbrack$ with respect to the variational parameters is challenging
for several reasons: 
(1) common choices for divergences may be biased~\citep{He2019}, 
(2) traditional divergences such as forward/reverse-KL, Jensen-Shannon and MMD~\citep{Gretton2012}
are expensive to compute,
and lastly (3) as we show in this section,
these divergences can be artificially reduced by trivially inflating the variational variances. 

\paragraph{Defining a computationally efficient, differentiable proxy.} 
For these reasons, as a proxy for $D\lbrack q_\phi(z) || p(z) \rbrack$, we choose to penalize the 
Henze-Zirkler (HZ) differentiable non-parametric test-statistic for normality~\citep{hztest1990} 
applied to the set of latent noise means, $\{\mu_{z_n}\}_{n=1}^N$:
\begin{align}
\begin{split}
\text{HZ}(\{a_n\}_{n=1}^N) = \frac{1}{N} &\sum\limits_{i=1}^N \sum\limits_{j=1}^N \exp\left(-\frac{\beta^2}{2} (a_i - a_j)^\intercal \Sigma^{-1} (a_i - a_j) \right) \\
&- 2 (1 + \beta^2)^\frac{-k}{2} \sum\limits_{i=1}^N \exp\left( \frac{-\beta^2}{2 (1 + \beta^2)} (a_i - \mu)^\intercal \Sigma^{-1} (a_i - \mu) \right) \\
&+ N (1 + 2 \beta^2)^\frac{-k}{2}
\end{split}
\end{align}
where $\mu, \Sigma$ are the sample mean and covariance of $\{a_n\}_{n=1}^N$, 
and $\beta = \frac{1}{\sqrt{2}} \left( \frac{N (2 k + 1)}{4} \right)^{-k - 4}$.
This encourages the aggregated posterior $q_\phi(z)$ to be Gaussian.
We additionally penalize the $\ell_2$ penalty of the off-diagonal terms of the 
empirical covariance of the variational mean of the latent noise: 
\begin{align}
\lambda_1 \text{HZ}(\{\mu_{z_n}\}_{n=1}^N)
+  \lambda_2 \big\lVert \offdiag\; \widehat{\Sigma}\left(\{\mu_{z_n}\}_{n=1}^N\right) \big\rVert_2 
\end{align}
Lastly, we note that traditional divergences can be trivially minimized by setting each $q_\phi(z_n | x_n, y_n) = p(z)$ (a phenomenon known as posterior collapse~\citep{He2019}). Our HZ based constraint cannot be fooled in this way.

To demonstrate that our proposed instantiation cannot be trivially minimized by inflating the variational variances,
we performed the following experiment:
(1) we initialized the variational means uniformly inside a triangle, and the variational variances randomly,
(2) we minimized $D \lbrack q_\phi(z) || p(z) \rbrack$ using one of the aforementioned divergences.
Figure \ref{fig:div-pathology} shows $q_\phi(z)$ and the distribution of the posterior means before and after the optimization.
As the figure shows, only our Henze-Zirkler penalty optimizes the means to be Gaussian; in contrast,
the remaining divergences were trivially minimized by keeping the means near their initialized positions
and inflating the variational variances.
This suggests that a poor initialization of the posterior means $\mu_{z_n}$ cannot be corrected
without our Henze-Zirkler penalty.

\begin{figure*}[p]
    \centering
    \includegraphics[width=0.6\textwidth]{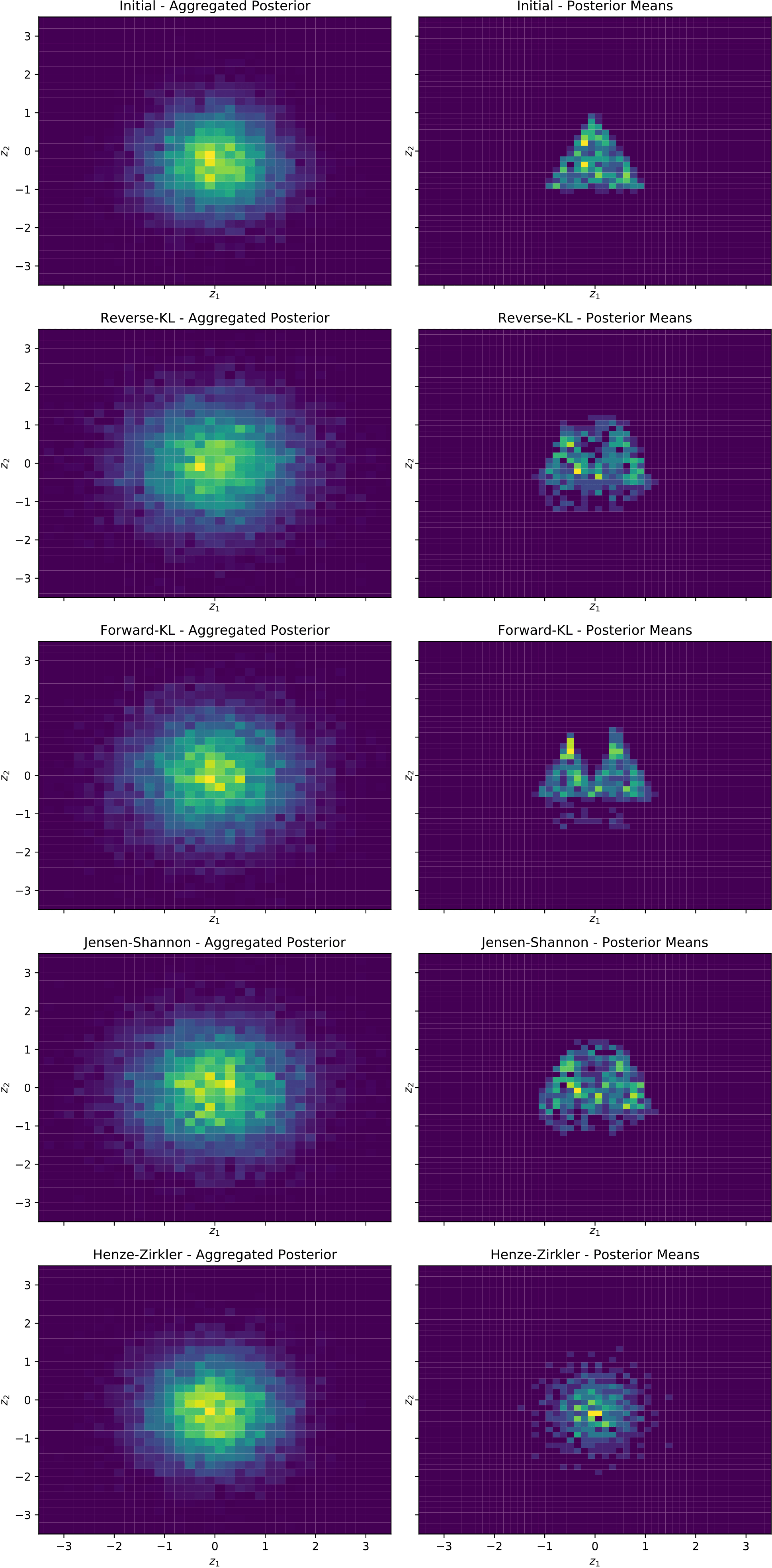} 
    \caption{{Left column: aggregated posterior $q_\phi(z)$.
    Right column: posterior means $\{ \mu_{z_1} \}_{n=1}^N$.
    The top row shows the initial variational parameters---the means of the posterior are uniformly distributed in a triangle,
    while the aggregated posterior looks Gaussian.
    Subsequent rows show the result of minimizing $D\lbrack q_\phi(z) || p(z) \rbrack$ for different divergences.
    Only our Henze-Zirkler optimizes the means to be Gaussian; in contrast,
    the remaining divergences can be trivially minimized by keeping the means near their initialized positions
    and inflating the variational variances.}}
    \label{fig:div-pathology}
\end{figure*}

\subsection{Defining the NCAI Objective}
Finally, we incorporate the ELBO and the differentiable forms of the constraints (as exponentially smoothed penalties) 
into Equation \ref{eqn:obj}:
\begin{equation} 
\begin{split}
  \mathcal{L}_{\text{NCAI}}(\phi) = & -\text{ELBO}(\phi)\\
  &+ \lambda_1 N \exp \left( \frac{\text{HZ}(\{\mu_{z_1}, \ldots, \mu_{z_N}\}}{\epsilon_T} \right) 
  + \lambda_2 N \lVert \offdiag\; \widehat{\Sigma}(\{\mu_{z_1}, \ldots, \mu_{z_N}\}) \rVert_2 \\
  &+ \lambda_3 N \exp\left(\frac{\text{PC}(\{x_n\}, \{\mu_{z_n}\})}{\epsilon_{x}}\right) 
  \cdot \exp\left(\frac{\text{PC}(\{y_n\}, \{\mu_{z_n}\})}{\epsilon_{y}}\right)
  \end{split}
  \label{eq:full-ncai-objective}
\end{equation}
where $\epsilon_T, \epsilon_x, \epsilon_y$ control the growth rate of the exponential penalties. We minimize the negative ELBO following Bayes by Backprop (BBB) \citep{blundell2015weight}: back-propagating through $\E{q_\phi(Z, W | \mathcal{D})} \lbrack \cdot \rbrack$ in the ELBO using the reparameterization trick \citep{kingma_auto-encoding_2013}, computing the KL-divergence terms and constraints using closed-form expressions.

\FloatBarrier
\section{Uncertainty Decomposition}\label{sec:uncertainty}

In addition to the quantitative results in the paper, showing that the uncertainty decomposition learned by NCAI
is quantitatively closer to that produced by HMC than the decomposition learned by BNN+LV,
we also show qualitatively that our uncertainty decomposition is closer to that of HMC in Figure \ref{fig:decomp}.
The figure shows that like HMC, NCAI has appropriately high total and aleatoric uncertainties at $x$'s for which there is a 
high variance in $y$, as well as high epistemic uncertainty for $x$'s near the boundary of the data (near $x = -5$ and $x = 5$).
In contrast, BNN+LV trained via mean-field VI does not.
In comparison to HMC, however, both NCAI and BNN+LV tend to underestimate the epistemic 
uncertainty, which should be high where $p(x)$ is low, signifying uncertainty over the model parameters due to lack of data.
This is because mean-field variational inference is generally unable to capture ``in-between uncertainty'' in BNNs ~\citep{Foong2019}.

\begin{figure*}[p]
    \centering
    
     \begin{subfigure}[t]{0.99\textwidth}
        \centering
        \includegraphics[width=\textwidth]{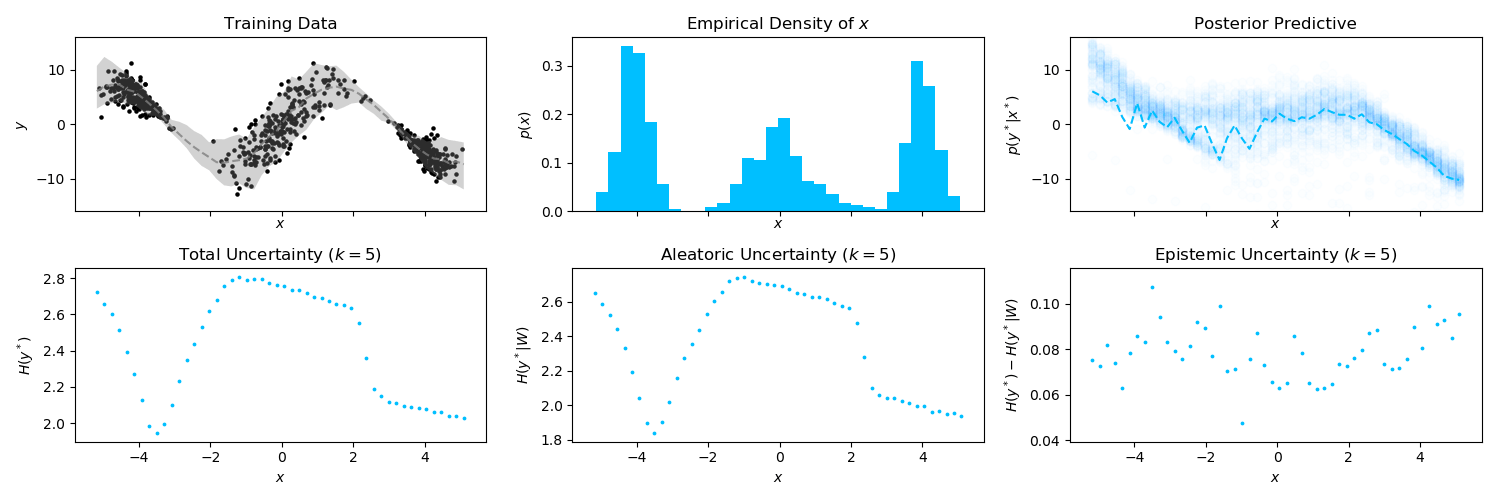} 
        \caption{{Uncertainty decomposition given by BNN+LV with mean-field VI.}}
    \end{subfigure}\vskip0.2cm
    ~ 
    \begin{subfigure}[t]{0.99\textwidth}
        \centering
         \includegraphics[width=\textwidth]{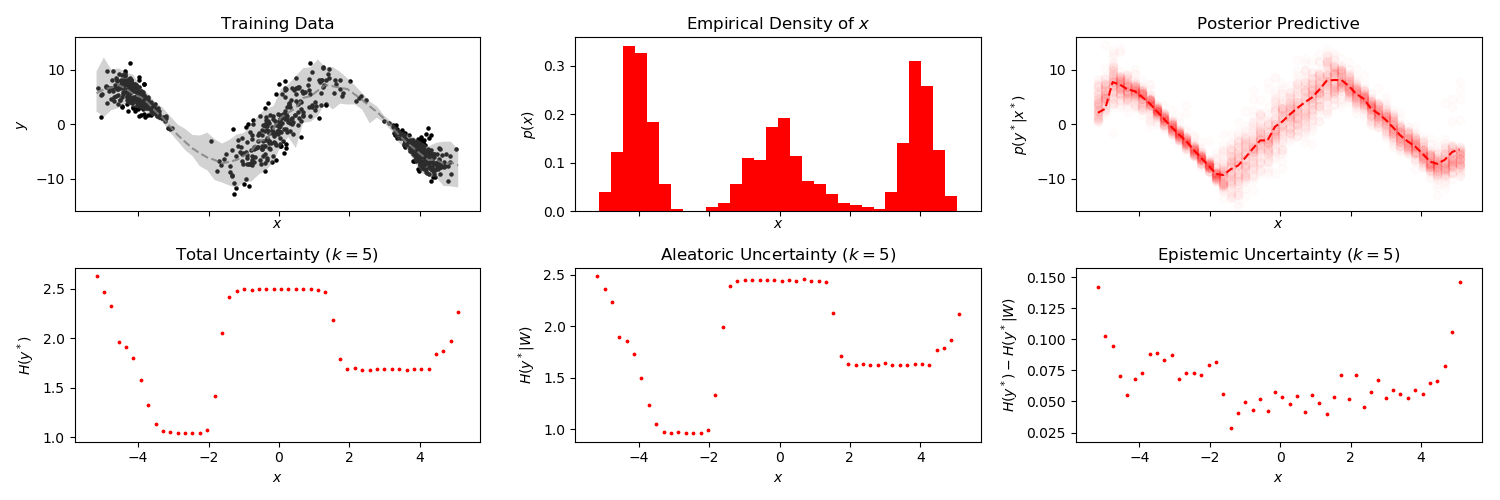} 
        \caption{{Uncertainty decomposition given by BNN+LV with NCAI.}}
    \end{subfigure}
    \vskip0.2cm
    ~ 
    \begin{subfigure}[t]{0.99\textwidth}
        \centering
         \includegraphics[width=\textwidth]{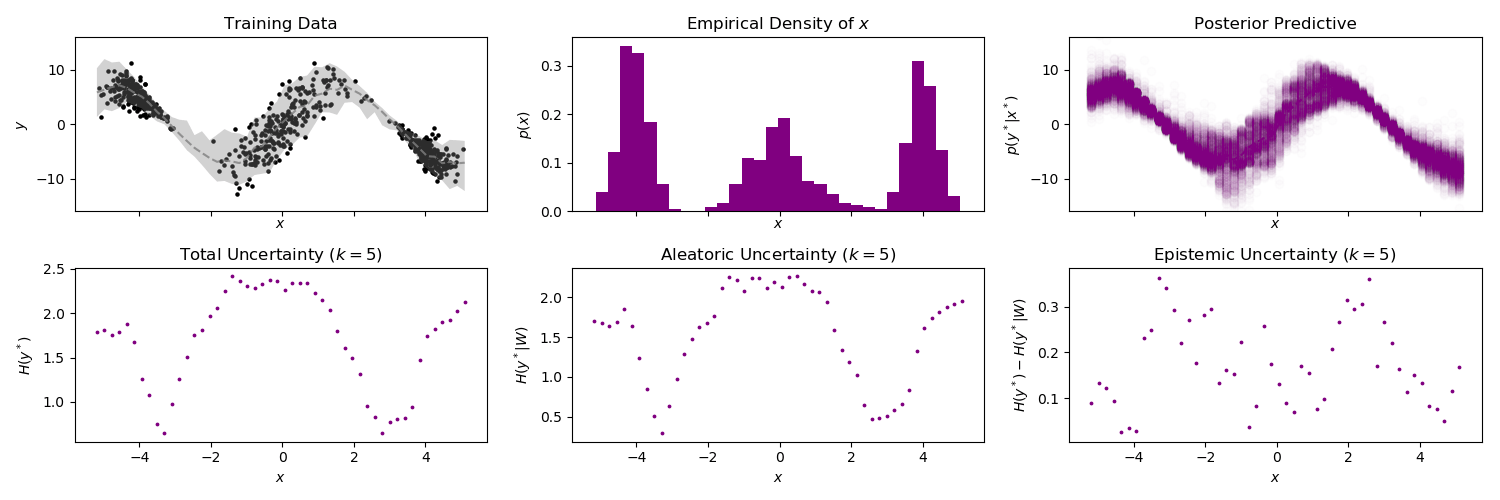} 
        \caption{{Uncertainty decomposition given by HMC.}}
    \end{subfigure}
    \caption{\textbf{Comparison of uncertainty decompositions.} We expect the epistemic noise to be roughly inversely proportional to the empirical density of $x$, e.g. epistemic uncertainty is higher where $x$ is sparsely sampled. We expect the aleatoric uncertainty to match the observed level of noise in the data. We see that BNN+LV trained with mean-field VI is unable to detect the highly noisy regions in data while BNN+LV with NCAI training learns aleatoric uncertainty that matches these regions very well.
The uncertainties are estimated using a nearest neighbor based entropy estimator with $k=5$ nearest neighbors.
HMC was trained with $\sigma^2_z = 1.0, \sigma^2_\epsilon$, as in the ground-truth, and with $\sigma^2_w = 10.0$.}
    \label{fig:decomp}
\end{figure*}
\FloatBarrier

\section{Experimental Setup} \label{sec:exp_setup} 

\subsection{Experimental Details} \label{sec:exp_details}

\paragraph{Architecture.} The network architectures we use for all experiments are summarized in Tables \ref{table:arch_real} and \ref{table:arch_toy}.
We note that we have purposefully selected lower capacity architectures to encourage the non-identifiability 
described in the paper to occur in practice.
We also note that the non-identifiability occurs even when the ground-truth network capacity is known,
as in the case of the HeavyTail and Depeweg data-sets, which have been generated using
a neural network.
As such, even when the data was generated by the same generative process as the one assumed by the model,
the problem of non-identifiability still occurs. 

\begin{table*}[t!]
\center
\footnotesize
\begin{tabular}{l ||c|c|c|c}
\textbf{Data-set} & \textbf{Size} & \textbf{Dimensionality} & \textbf{Hidden Nodes} \\ \hline \hline
\textbf{Lidar} & $221$ & 1 & $10$ \\
\textbf{Yacht} & $309$ & 6 & $5$ \\
\textbf{Energy Efficiency} & $768$ & 8 & $10$ \\
\textbf{Airfoil} & Subsampled to $1000$ & 5 & $30$ \\
\textbf{Abalone} & Subsampled to $1000$ & 10 (1-hot for categorical) & $10$ \\
\textbf{Wine Quality Red} & $1600$ & 1 & $20$ 
\end{tabular}
\caption{Experimental details for the real data-sets}
\label{table:arch_real}
\end{table*}

\paragraph{Train/Validation/Test Data-splits.}
We split each data-set into train/validation/test set 6 times.
We use the first data-split to select hyper-parameters by selecting the hyper-parameters that yield the best
average log-likelihood performance on the validation set across 10 random restarts.
After having selected the hyper-parameter for each method,
we select between $\text{NCAI}_{\lambda=0}$ and $\text{NCAI}_{\lambda}$ by picking the approach that 
yielded the best log-likelihood performance on the validation set across the 10 random restarts.
Now, using the selected hyper-parameters and form of NCAI, we train our models on the remaining 5 data-splits,
averaging the best-of-10 random restarts (selected by validation log-likelihood) across the data-splits. 

For Abalone, Airfoil, Boston Housing, Energy Efficiency, Lidar, Wine Quality Red, Yacht, Goldberg, Williams, Yuan, we split the data into a $70\%$ training set, $20\%$ validation set and $10\%$ test set.

\paragraph{Prior Hyper-parameters Selection.}
Since hyper-parameter search is expensive to compute,
we choose hyper-parameters of the priors $p(z), p(W)$ using empirical Bayes (MAP Type II). That is, we place Inverse Gamma priors ($\alpha = 3.0, \beta = 0.5$) on the variances of the network weights and the latent variables; then we approximate the negative ELBO with the MAP estimates of the variances, making the assumption that these terms dominate the respective integrals in which they appear:
\begin{align}
\begin{split}
-\text{ELBO}(\phi) \approx  &-\E{q_\phi(Z, W | \mathcal{D})} \lbrack \log p(Y | X, W, Z) \rbrack \\
  &+ D_{\text{KL}}[q_\phi(W | \mathcal{D}) \| p(W | s_w^*)p(s_w^*)] \\
  &+ D_{\text{KL}}[q_\phi(Z | \mathcal{D}) \| p(Z | s_z^*)p(s_z^*)] \\
  \end{split}
 \end{align}
 where we define:
 \begin{align}
 s_z^* &= \underset{s_z}{\text{argmin}}\; D_{\text{KL}}[q_\phi(Z | \mathcal{D}) \| p(Z | s_z)p(s_z)] 
 =  \frac{2 \beta + \frac{1}{N} \sum\limits_n \left\lbrack \text{tr}\left(\Sigma_{q_n} \right) + \mu_{q_n}^T \mu_{q_n} \right\rbrack}{K + 2 \alpha - 2}, \\
 s_w^* &= \underset{s_w}{\text{argmin}}\; D_{\text{KL}}[q_\phi(W | \mathcal{D}) \| p(W | s_w)p(s_w)] 
 = \frac{2 \beta + \text{tr}\left(\Sigma_q \right) + \mu_q^T \mu_q }{H + 2 \alpha - 2},
 \end{align}
with $K$ and $H$ as the dimensionality of $z_n$ and $W$, respectively. 
The optimal variances, $s_z^*, s_w^*$, have analytic solutions:
\begin{align}
s_z^* &= \frac{2 \beta_z + \frac{1}{N} \sum\limits_n \left\lbrack \tr\left(\Sigma_{q_n} \right) + \mu_{q_n}^\intercal \mu_{q_n} \right\rbrack}{K + 2 \alpha_z - 2}, \\
s_w^* &= \frac{2 \beta_W + \tr\left(\Sigma_q \right) + \mu_q^\intercal \mu_q }{K + 2 \alpha_W - 2},
\end{align}
where the $\beta$'s and $\alpha$'s are the parameters of the respective Inverse Gamma distributions. 
In training, we update the objective as well as the optimal hyper-parameters via coordinate descent. That is, we iteratively compute $s_z^*, s_w^*$ in closed-form given the current $\phi$, and then optimize $\phi$ while holding  $s_z^*, s_w^*$ fixed. 

\paragraph{Remaining Hyperparameter Selection.} 
For the data-sets Abalone, Airfoil, Boston Housing, Energy Efficiency, Lidar, Wine Quality Red, Yacht, Goldberg, Williams, Yuan,
we used grid-searched over the following parameters:
\begin{itemize}
\item BNN: $\sigma^2_\epsilon = \{1.0, 0.1, 0.01\}$
\item BNN+LV: $\sigma^2_\epsilon = \{0.1, 0.01\}$
\item NCAI: 
$\sigma^2_\epsilon = \{0.1, 0.01\}$, 
$\lambda_2=\{10.0\}$, 
$\epsilon_T=\{0.01, 0.0003\}$, 
$\epsilon_y=\{0.1, 0.5\}$, 
$\epsilon_x=\{0.5, 1.0\}$
\end{itemize}
For HeavyTails we grid-searched over the following parameters:
\begin{itemize}
\item BNN: $\sigma^2_\epsilon = \{1.0, 0.1, 0.5\}$
\item BNN+LV: $\sigma^2_\epsilon = \{0.1\}, \sigma^2_z = \{0.01\}$
\item NCAI: 
$\sigma^2_\epsilon = \{0.1\}$,
$\lambda_2=\{10.0\}$, 
$\epsilon_T=\{0.01, 0.0003\}$,
$\epsilon_y=\{0.1, 0.5\}$,
$\epsilon_x=\{0.5, 1.0\}$
\end{itemize}
For Depeweg we grid-searched over the following parameters:
\begin{itemize}
\item BNN: $\sigma^2_\epsilon = \{1.0, 0.1, 0.5\}$
\item BNN+LV: $\sigma^2_\epsilon = \{0.1\}, \sigma^2_z = \{1.0\}$
\item NCAI: \\
$\sigma^2_\epsilon = \{0.1\}$,
$\lambda_2=\{10.0\}$,
$\epsilon_T=\{0.01, 0.0003\}$,
$\epsilon_y=\{0.1, 0.5\}$,
$\epsilon_x=\{0.5, 1.0\}$
\end{itemize}

\begin{table*}[t!]
\center
\footnotesize
\begin{tabular}{l ||c|c|c|c}
\textbf{Data-set} & \textbf{Hidden Nodes} & \textbf{Layers} \\ \hline \hline
\textbf{Williams} & 20 & 2 \\
\textbf{Yuan} & 20 & 1 \\
\textbf{Goldberg} & 20 & 1 \\
\textbf{HeavyTail} & 50 & 1 \\
\textbf{Depeweg} & 50 & 1 \\
\textbf{Bimodal} & 50 & 2 \\
\end{tabular}
\caption{Experimental details for the synthetic data-sets}
\label{table:arch_toy}
\end{table*}

\paragraph{Random Initialization.} For all models without a specialized random initialization,
we initialize all variational parameters initialized to samples drawn from xavier-normal with a gain of $1.0$.
For runs in which the prior variances are estimated via empirical Bayes MAP (described above),
we initialize the prior variances to samples from the hyper-prior. 


\subsection{Evaluation Metrics}\label{sec:eval}
\paragraph{Quality of Fit.} We measure the \emph{training reconstruction MSE}, the ability of the model to reconstruct the training targets with the inferred weights and latent variables:
\begin{equation}
\frac{1}{N} \sum\limits_{n=1}^N \E{q_\phi(z_n | x_n, y_n) q_\phi(W | \mathcal{D})} \left\lbrack \lVert y_n - f(z_n, W) \rVert_2^2 \right\rbrack. 
\label{eqn:recon}
\end{equation}
At test time, we measure the quality of the posterior predictive distribution of the model by computing the \emph{average log-likelihood}:
\begin{equation}
\frac{1}{N} \sum\limits_{n=1}^N \log \E{p(z_n) q_\phi(W | \mathcal{D})} \left\lbrack p(y_n | x_n, W, z_n) \right\rbrack.
\end{equation}
We also compute the predictive quality of the model by computing the \emph{predictive MSE}:
\begin{equation}
\frac{1}{N} \sum\limits_{n=1}^N \left\lbrack \lVert y_n -  \E{p(z_n)q_\phi(W | \mathcal{D})} [f(z_n, W)] \rVert_2^2 \right\rbrack.
\label{eqn:mse}
\end{equation}
Note that the difference between the reconstruction MSE 
(Equation \ref{eqn:recon}) and the predictive MSE (Equation \ref{eqn:mse}) 
is that in the latter we sample the latent variables from the prior distributions rather than the learned posterior distributions.

\paragraph{Posterior Predictive Calibration.} We measure the quality of the model's predictive uncertainty by computing the percentage of observations for which the ground-truth $y$ lies within a $95\%$ predictive-interval (PI) 
of the learned posterior predictive---this quantity is called the Prediction Interval Coverage Probability (PICP). 
We measure the tightness of the model's predictive uncertainty by computing the $95\%$ Mean Prediction Interval Width (MPIW).

\paragraph{Satisfaction of Model Assumptions.} We estimate the mutual information between $x$ and $z$
by computing the Kraskov nearest-neighbor based estimator~\citep{kraskov2004} (with $5$ nearest neighbors)
on the $x$'s and the means of the $z$'s: $\hat{I}(x; \mu_z)$. 
We use the variational means $\mu_{z_n}$ instead of samples from the aggregated posterior $z \sim q_\phi(z)$, 
since otherwise dependence can be artificially reduced with large variational variances $\sigma_{z_n}^2$.

For the univariate case, when $D = K = 1$, we use the Kolmogorov-Smirnov (KS) 
two-sample test statistic~\citep{kstest} to evaluate divergence between $q_\phi(z)$ and $p(z)$. 
When computing the test statistic, we represent $q_\phi(z)$ using $\mu_{z_n}$'s and $p(z)$ using its samples.
This is because the $\sigma_{z_n}^2$'s are large, the distance between $q_\phi(z)$ and $p(z)$ 
more difficult to detect. A lower KS test-statistic indicates that $q_\phi(z)$ and $p(z)$ are more similar. 
We compute the Jensen-Shannon divergence between $q_\phi(z)$ and $p(z)$ in multivariate cases.

\section{Data-sets}\label{sec:datasets}

 \paragraph{Synthetic Data Generated with Ground-Truth.}
We generate three synthetic data-sets with corresponding ground-truth in order to guarantee 
that our generative process matches our data.
We generate these data-sets by training a neural network to map the $x_n$'s and ground-truth $z_n$'s 
to the noise-less $y_n$'s (i.e. before adding $\epsilon$), specified by some ground-truth function $f_\text{GT}$. 
We then re-generate the $y_n$'s from the learned neural network and treat that network as the ground-truth function.
We specify the parameters of the generative process in Equation \ref{eqn:gen_model} 
used to generate each of these data-sets below:
\begin{enumerate}
\item \textbf{Heavy-Tail}: 
\begin{align*}
f_\text{GT}(x, z) &= 6 \cdot \tanh \left( 0.1 x^3(z + 1)^6 -10 \cdot xz^2 + z \right) \\
\sigma^2_\epsilon &= 0.1 \\
\sigma^2_z &= 0.01 
\end{align*}
For this data-set, we sampled 300 training, validation and test $x$'s uniformly on $[-4, 4]$. 

\item \textbf{Depeweg} \citep{depeweg2018decomposition}:
\begin{align*}
f_\text{GT}(x, z) &= 7 \sin(x) + 3 |\cos (x / 2) | z \\
\sigma^2_\epsilon &= 0.1 \\
\sigma^2_z &= 1.0
\end{align*}
For this data-set, we sampled 750 training $x$'s, 250 validation and 250 test $x$'s from a uniform mixture of the following gaussians: $\mathcal{N}(0, -4.0, 0.16), \mathcal{N}(0, 0, 0.81), \mathcal{N}(0, 4.0, 0.16)$.
(Note: $\sigma^2_\epsilon = 1.0$ in the original data-set from \cite{depeweg2018decomposition}).

\item \textbf{Bimodal} \citep{depeweg2018decomposition}: 
\begin{align*}
f_\text{GT}(x, z) &= \mathbb{I}(z > 0) \cdot 10 \sin(x) + \mathbb{I}(z \leq 0) \cdot 10 \cos(x) \\
\sigma^2_\epsilon &= 1.0 \\
\sigma^2_z &= 0.1
\end{align*}
For this data-set, we sampled 750 training $x$'s, 250 validation and 250 test $x$'s from an exponential distribution with $\lambda = 2$, shifted to the left by $0.5$ and clipped to the range of $[-0.5, 2.0]$. 
 \end{enumerate}

\paragraph{Synthetic Data without Ground-Truth.}
We also consider 3 synthetic data-sets, most of which have been widely used to evaluate heteroscedastic regression models \citep{goldberg1998regression, wright1999bayesian, kersting2007most, wang2012Gaussian}:
\begin{enumerate}
\item \textbf{Goldberg} \citep{goldberg1998regression}: targets are given by $y = 2\sin(2\pi x) + \epsilon(x)$, where $\epsilon(x) \sim \mathcal{N}(0, x + 0.5)$. Evaluated on 200 training input, 200 validation and 200 test inputs uniformly sampled from $[0, 1]$. 
\item \textbf{Yuan} \citep{yuan2004doubly}: targets are given by $y = 2[\mathrm{exp}\{-30(x - 0.25)^2 + \sin(\pi x^2)\}] - 2 + \epsilon(x)$, where $\epsilon(x) \sim \mathcal{N}(0, \mathrm{exp}\{\sin(2\pi x)\})$. Evaluated on 200 training input, 200 validation and 200 test inputs uniformly sampled from $[0, 1]$. 
\item \textbf{Williams} \citep{williams1996using}: the targets are given by $y = \sin(2.5x) \cdot \sin(1.5x) + \epsilon(x)$, where $\epsilon(x) \sim \mathcal{N}(0, 0.01 + 0.25 (1 - \sin(2.5x))^2)$. Evaluated on 200 training input, 200 validation and 200 test inputs uniformly sampled from $[0, 1]$. 
 \end{enumerate}

\paragraph{Real Data.} We use 6 UCI data-sets~\citep{uci17} and a data-set commonly used in the heteroscedastic literature, Lidar~\citep{lidar94}---see Table \ref{table:arch_real} for details.

\FloatBarrier
\section{Additional Quantitative Results and Metrics}\label{sec:results_tab}

Table \ref{table:synth_rmse} shows the RMSE for all models on synthetic data---NCAI consistently 
achieves lower RMSE than mean-field VI.
Tables \ref{table:real_metrics1} and \ref{table:real_metrics2} show the mutual information 
between $x$ and $z$ under the posterior on real data-sets and the Henze-Zirkler test-statistic---on most 
data-sets NCAI better satisfies modeling assumptions than mean-field VI.
Lastly, Tables \ref{res:abalone}, \ref{res:airfoil}, \ref{res:ee}, \ref{res:heavytail}, \ref{res:goldberg}, \ref{res:williams}, 
\ref{res:yuan}, \ref{res:wine}, \ref{res:yacht}, \ref{res:lidar} and \ref{res:depeweg})
provide the complete results for all experiments, including evaluation metrics from Appendix \ref{sec:eval}
omitted from the main text for brevity. 

\FloatBarrier

\begin{table*}[h!]
\centering
\footnotesize
\begin{tabular}{ll||ccccc}
\multicolumn{6}{c}{\emph{RMSE on Test Data for Synthetic Data-sets }}\\
Model & Inference                                    & Heavy Tail         & Goldberg           & Williams           & Yuan               & Depeweg             \\
\hline \hline
\textbf{BNN}  & \textbf{MFVI}                      & $1.831 \pm 0.074$  & $0.335 \pm 0.025$  & $1.017 \pm 0.06$   & $0.607 \pm 0.035$  & $1.953 \pm 0.071$   \\
\textbf{BNN+LV}  & \textbf{MFVI}                   & $1.882 \pm 0.088$  & $0.376 \pm 0.032$  & $1.118 \pm 0.096$  & $0.622 \pm 0.039$  & $3.523 \pm 0.501$   \\ 
\hline
\textbf{BNN+LV}  &  \textbf{$\text{NCAI}_{\lambda=0}$}  & $1.787 \pm 0.094$  & $0.339 \pm 0.026$  & $0.978 \pm 0.083$  & $0.619 \pm 0.039$  & $1.932 \pm 0.059$   \\
\textbf{BNN+LV}  & \textbf{$\text{NCAI}_\lambda$}      & $1.79 \pm 0.09$    & $0.337 \pm 0.025$  & $0.978 \pm 0.083$  & $0.619 \pm 0.039$  & $1.932 \pm 0.059$  
\end{tabular}
\vskip0.1cm
\caption{Comparison of RMSE on synthetic data-sets ($\pm$ std). Across all data-sets BNN+LV trained with $\text{NCAI}_\lambda$ training yields comparable if not better generalization than BNN+LV and BNN trained with mean-field variational inference (MFVI).}
\label{table:synth_rmse}
\end{table*}

\begin{table*}[h!]
\setlength{\tabcolsep}{2.5pt}
\centering
\small
\begin{tabular}{l||cccccc}
\multicolumn{6}{c}{\emph{Mutual Information between $x$ and $z$ (Real Data)}}\\
Inference                                     & \multicolumn{1}{c}{Abalone} & Airfoil            & Energy             & Wine               & Lidar              & Yacht               \\
\hline \hline
 \textbf{MFVI}                     & $0.152 \pm 0.015$            & $0.485 \pm 0.054$  & $0.139 \pm 0.086$  & $0.045 \pm 0.012$  & $0.667 \pm 0.061$  & $0.077 \pm 0.012$   \\ 
\hline
 \textbf{$\text{NCAI}_{\lambda=0}$}  & $0.149 \pm 0.078$            & $0.29 \pm 0.021$   & $0.162 \pm 0.063$  & $0.047 \pm 0.011$  & $0.373 \pm 0.037$  & $0.087 \pm 0.012$   \\
\textbf{$\text{NCAI}_\lambda$}       & $0.149 \pm 0.078$            & $0.29 \pm 0.021$   & $0.226 \pm 0.041$  & $0.029 \pm 0.008$  & $0.842 \pm 0.06$   & $0.087 \pm 0.012$  
\end{tabular}
\vskip0.1cm
\caption{Comparison of \emph{mutual information} between z and x on real data-sets ($\pm$ std). Across nearly all data-sets, NCAI training infers $z$'s that have the least mutual information in comparison mean-field VI (MFVI).}
\label{table:real_metrics1}
\normalsize
\end{table*}

\begin{table*}[h!]
\setlength{\tabcolsep}{2.5pt}
\centering
\footnotesize
\begin{tabular}{l||cccccc}
\multicolumn{6}{c}{\emph{Henze-Zirkler Test-Statistic (Real Data)}}\\
Inference                                     & \multicolumn{1}{c}{Abalone} & Airfoil              & Energy              & Wine                & Lidar              & Yacht                \\

\hline \hline
 \textbf{MFVI}                     & $26.148 \pm 4.394$           & $28.108 \pm 3.205$   & $16.976 \pm 3.519$  & $52.566 \pm 2.633$  & $5.09 \pm 0.991$   & $27.059 \pm 4.144$   \\ 
\hline
 \textbf{$\text{NCAI}_{\lambda=0}$}  & $17.975 \pm 6.725$           & $49.122 \pm 11.426$  & $19.071 \pm 5.414$  & $53.201 \pm 1.247$  & $7.804 \pm 1.727$  & $51.283 \pm 9.548$   \\
\textbf{$\text{NCAI}_\lambda$}       & $17.975 \pm 6.725$           & $49.122 \pm 11.426$  & $1.186 \pm 0.558$   & $1.641 \pm 0.242$   & $0.005 \pm 0.001$  & $51.283 \pm 9.548$  
\end{tabular}
\vskip0.1cm
\caption{Comparison of model assumption satisfaction (in terms of the HZ metric) on real data-sets ($\pm$ std). Across nearly all data-sets, NCAI training infers $z$'s that are more Gaussian (lowest HZ) relative to mean-field VI (MFVI).}
\label{table:real_metrics2}
\normalsize
\end{table*}

\FloatBarrier

\begin{table*}[h!]
  \centering
  \tiny
  \setlength{\tabcolsep}{3pt}
    \begin{tabular}{l|p{25mm}|c|c|c}
    & $\text{NCAI}_\lambda$ & $\text{NCAI}_{\lambda=0}$ & BNN & BNN+LV \\ \hline
$D_{\text{JS}}(q(z)||p(z))$ & $0.007 \pm 0.004$ & {{$0.021 \pm 0.019$}} & N/A & $0.009 \pm 0.003$ \\ \hline 
$D_{\text{KL}}(p(z)||q(z))$ & $0.015 \pm 0.012$ & {{$0.05 \pm 0.046$}} & N/A & $0.015 \pm 0.006$ \\ \hline 
$\hat{I}(x; \mu_{z})$ & $0.146 \pm 0.131$ & {{$0.149 \pm 0.078$}} & N/A & $0.152 \pm 0.015$ \\ \hline 
$\hat{I}(x; z)$ & $0.015 \pm 0.006$ & {{$0.012 \pm 0.005$}} & N/A & $0.011 \pm 0.002$ \\ \hline 
$\text{HZ}(\{\mu_{z_1}, \ldots, \mu_{z_N}\})$ & $0.839 \pm 0.15$ & {{$17.975 \pm 6.725$}} & N/A & $26.148 \pm 4.394$ \\ \hline 
$s_w^*$ & $0.2 \pm 0.018$ & {{$1.272 \pm 1.043$}} & $0.948 \pm 0.68$ & $0.202 \pm 0.02$ \\ \hline 
$s_y^*$ & $0.01 \pm 0.0$ & {{$0.01 \pm 0.0$}} & $0.1 \pm 0.0$ & $0.1 \pm 0.0$ \\ \hline 
$s_z^*$ & $0.252 \pm 0.004$ & {{$0.248 \pm 0.001$}} & N/A & $0.249 \pm 0.0$ \\ \hline 
95\%-MPIW Test (Unnorm) & $7.564 \pm 0.459$ & {{$7.027 \pm 0.357$}} & $4.112 \pm 0.103$ & $7.127 \pm 0.311$ \\ \hline 
95\%-MPIW Train (Unnorm) & $7.716 \pm 0.448$ & {{$7.096 \pm 0.446$}} & $4.111 \pm 0.099$ & $7.248 \pm 0.215$ \\ \hline 
95\%-PICP Test & $94.3 \pm 2.515$ & {{$92.3 \pm 0.975$}} & $76.9 \pm 4.292$ & $94.4 \pm 2.329$ \\ \hline 
95\%-PICP Train & $94.771 \pm 0.559$ & {{$93.343 \pm 1.743$}} & $77.771 \pm 1.476$ & $94.4 \pm 0.509$ \\ \hline 
PC($x, \mu_z$) & $0.001 \pm 0.0$ & {{$0.006 \pm 0.002$}} & N/A & $0.005 \pm 0.002$ \\ \hline 
PC($y, \mu_z$) & $0.032 \pm 0.003$ & {{$0.063 \pm 0.023$}} & N/A & $0.115 \pm 0.016$ \\ \hline 
Post-Pred Avg-LL Test & $-0.837 \pm 0.105$ & {{$-0.831 \pm 0.086$}} & $-1.248 \pm 0.153$ & $-0.843 \pm 0.071$ \\ \hline 
Post-Pred Avg-LL Train & $-0.798 \pm 0.051$ & {{$-0.799 \pm 0.064$}} & $-1.086 \pm 0.099$ & $-0.832 \pm 0.022$ \\ \hline 
RMSE Test (Unnorm) & $0.208 \pm 0.012$ & {{$0.205 \pm 0.013$}} & $0.194 \pm 0.011$ & $0.204 \pm 0.012$ \\ \hline 
RMSE Train (Unnorm) & $0.208 \pm 0.012$ & {{$0.205 \pm 0.013$}} & $0.194 \pm 0.011$ & $0.204 \pm 0.011$ \\ \hline 
Recon MSE & $0.011 \pm 0.0$ & {{$0.012 \pm 0.0$}} & N/A & $0.165 \pm 0.002$ \\ \hline 
Hyperparams & $\sigma^2_\epsilon=0.01$,\newline $\lambda_2=10$, \newline$\epsilon_T=0.01$, \newline$\epsilon_x=0.54$, \newline$\epsilon_y=1.0$ & $\sigma^2_\epsilon=0.01$ & $\sigma^2_\epsilon=0.1$ & $\sigma^2_\epsilon=0.1$ \\  
\end{tabular} 
\caption{Experiment Evaluation Summary for Abalone ($\pm$ std).}
\label{res:abalone}
 \end{table*}

\begin{table*}[h!]
  \centering
  \tiny
    \begin{tabular}{l|p{25mm}|c|c|c}
    & $\text{NCAI}_\lambda$ & $\text{NCAI}_{\lambda=0}$ & BNN & BNN+LV \\ \hline
$D_{\text{JS}}(q(z)||p(z))$ & $-0.0 \pm 0.003$ & {{$-0.001 \pm 0.003$}} & N/A & $0.001 \pm 0.004$ \\ \hline 
$D_{\text{KL}}(p(z)||q(z))$ & $0.002 \pm 0.002$ & {{$-0.0 \pm 0.002$}} & N/A & $0.005 \pm 0.003$ \\ \hline 
$\hat{I}(x; \mu_{z})$ & $0.262 \pm 0.03$ & {{$0.29 \pm 0.021$}} & N/A & $0.485 \pm 0.054$ \\ \hline 
$\hat{I}(x; z)$ & $0.107 \pm 0.005$ & {{$0.105 \pm 0.008$}} & N/A & $0.174 \pm 0.025$ \\ \hline 
$\text{HZ}(\{\mu_{z_1}, \ldots, \mu_{z_N}\})$ & $0.25 \pm 0.11$ & {{$49.122 \pm 11.426$}} & N/A & $28.108 \pm 3.205$ \\ \hline 
$s_w^*$ & $0.548 \pm 0.084$ & {{$0.509 \pm 0.076$}} & $0.762 \pm 0.182$ & $0.105 \pm 0.032$ \\ \hline 
$s_y^*$ & $0.1 \pm 0.0$ & {{$0.1 \pm 0.0$}} & $0.1 \pm 0.0$ & $0.1 \pm 0.0$ \\ \hline 
$s_z^*$ & $0.25 \pm 0.001$ & {{$0.247 \pm 0.0$}} & N/A & $0.25 \pm 0.0$ \\ \hline 
95\%-MPIW Test (Unnorm) & $10.426 \pm 0.81$ & {{$10.283 \pm 0.423$}} & $8.999 \pm 0.164$ & $17.023 \pm 1.972$ \\ \hline 
95\%-MPIW Train (Unnorm) & $10.294 \pm 0.732$ & {{$10.15 \pm 0.41$}} & $8.985 \pm 0.148$ & $16.974 \pm 1.925$ \\ \hline 
95\%-PICP Test & $94.7 \pm 2.797$ & {{$95.5 \pm 1.969$}} & $91.9 \pm 1.475$ & $92.9 \pm 1.245$ \\ \hline 
95\%-PICP Train & $97.429 \pm 0.598$ & {{$97.286 \pm 0.769$}} & $93.2 \pm 2.077$ & $94.686 \pm 0.584$ \\ \hline 
PC($x, \mu_z$) & $0.001 \pm 0.0$ & {{$0.005 \pm 0.002$}} & N/A & $0.004 \pm 0.001$ \\ \hline 
PC($y, \mu_z$) & $0.002 \pm 0.001$ & {{$0.03 \pm 0.01$}} & N/A & $0.162 \pm 0.069$ \\ \hline 
Post-Pred Avg-LL Test & $-0.48 \pm 0.076$ & {{$-0.462 \pm 0.056$}} & $-0.512 \pm 0.083$ & $-0.995 \pm 0.143$ \\ \hline 
Post-Pred Avg-LL Train & $-0.407 \pm 0.043$ & {{$-0.401 \pm 0.026$}} & $-0.422 \pm 0.043$ & $-0.972 \pm 0.137$ \\ \hline 
RMSE Test (Unnorm) & $0.05 \pm 0.005$ & {{$0.05 \pm 0.003$}} & $0.05 \pm 0.003$ & $0.094 \pm 0.009$ \\ \hline 
RMSE Train (Unnorm) & $0.05 \pm 0.005$ & {{$0.05 \pm 0.003$}} & $0.05 \pm 0.003$ & $0.094 \pm 0.008$ \\ \hline 
Recon MSE & $0.177 \pm 0.01$ & {{$0.174 \pm 0.006$}} & N/A & $0.139 \pm 0.013$ \\ \hline 
Hyperparams &  $\sigma^2_\epsilon=0.1$,\newline $\lambda_2=10$,\newline $\epsilon_T=0.01$, \newline$\epsilon_x=0.5$,\newline $\epsilon_y=1.0$ & $\sigma^2_\epsilon=0.1$ & $\sigma^2_\epsilon=0.1$ & $\sigma^2_\epsilon=0.1$ \\  
\end{tabular} 
\caption{Experiment Evaluation Summary for Airfoil($\pm$ std).}
\label{res:airfoil}
 \end{table*}

 \begin{table*}[h!]
  \centering
  \tiny
   \begin{tabular}{l|p{25mm}|c|c|c}
    & $\text{NCAI}_\lambda$ & $\text{NCAI}_{\lambda=0}$ & BNN & BNN+LV \\ \hline
$D_{\text{JS}}(q(z)||p(z))$ & $0.01 \pm 0.005$ & $0.011 \pm 0.009$ & {{N/A}} & $0.032 \pm 0.012$ \\ \hline 
$D_{\text{KL}}(p(z)||q(z))$ & $0.021 \pm 0.01$ & $0.025 \pm 0.01$ & {{N/A}} & $0.066 \pm 0.029$ \\ \hline 
$\hat{I}(x; \mu_{z})$ & $0.226 \pm 0.041$ & $0.162 \pm 0.063$ & {{N/A}} & $0.139 \pm 0.086$ \\ \hline 
$\hat{I}(x; z)$ & $0.14 \pm 0.006$ & $0.122 \pm 0.011$ & {{N/A}} & $0.127 \pm 0.02$ \\ \hline 
$\text{HZ}(\{\mu_{z_1}, \ldots, \mu_{z_N}\})$ & $1.186 \pm 0.558$ & $19.071 \pm 5.414$ & {{N/A}} & $16.976 \pm 3.519$ \\ \hline 
$s_w^*$ & $0.496 \pm 0.403$ & $2.91 \pm 2.843$ & {{$1.87 \pm 1.005$}} & $0.921 \pm 1.313$ \\ \hline 
$s_y^*$ & $0.01 \pm 0.0$ & $0.01 \pm 0.0$ & {{$0.01 \pm 0.0$}} & $0.01 \pm 0.0$ \\ \hline 
$s_z^*$ & $0.251 \pm 0.003$ & $0.245 \pm 0.001$ & {{N/A}} & $0.247 \pm 0.0$ \\ \hline 
95\%-MPIW Test (Unnorm) & $11.828 \pm 2.308$ & $10.534 \pm 1.751$ & {{$5.704 \pm 0.16$}} & $15.159 \pm 5.563$ \\ \hline 
95\%-MPIW Train (Unnorm) & $11.925 \pm 2.154$ & $10.329 \pm 1.331$ & {{$5.71 \pm 0.146$}} & $13.813 \pm 2.45$ \\ \hline 
95\%-PICP Test & $94.51 \pm 2.384$ & $93.464 \pm 1.533$ & {{$81.438 \pm 2.758$}} & $94.51 \pm 2.981$ \\ \hline 
95\%-PICP Train & $95.139 \pm 2.296$ & $94.36 \pm 1.556$ & {{$84.527 \pm 4.621$}} & $94.731 \pm 2.809$ \\ \hline 
PC($x, \mu_z$) & $0.002 \pm 0.001$ & $0.005 \pm 0.004$ & {{N/A}} & $0.008 \pm 0.004$ \\ \hline 
PC($y, \mu_z$) & $0.003 \pm 0.001$ & $0.014 \pm 0.006$ & {{N/A}} & $0.022 \pm 0.006$ \\ \hline 
Post-Pred Avg-LL Test & $0.898 \pm 0.452$ & $0.862 \pm 0.138$ & {{$1.281 \pm 0.171$}} & $0.573 \pm 0.288$ \\ \hline 
Post-Pred Avg-LL Train & $0.953 \pm 0.393$ & $0.941 \pm 0.108$ & {{$1.443 \pm 0.16$}} & $0.657 \pm 0.205$ \\ \hline 
RMSE Test (Unnorm) & $0.035 \pm 0.007$ & $0.029 \pm 0.005$ & {{$0.016 \pm 0.002$}} & $0.041 \pm 0.011$ \\ \hline 
RMSE Train (Unnorm) & $0.035 \pm 0.007$ & $0.029 \pm 0.005$ & {{$0.016 \pm 0.002$}} & $0.041 \pm 0.011$ \\ \hline 
Recon MSE & $0.028 \pm 0.001$ & $0.031 \pm 0.003$ & {{N/A}} & $0.028 \pm 0.002$ \\ \hline 
Hyperparams & $\sigma^2_\epsilon=0.01$,\newline $\lambda_2=10$, \newline $\epsilon_T=0.0003$, \newline $\epsilon_x=0.1$, \newline $\epsilon_y=1.0$ & $\sigma^2_\epsilon=0.01$ & $\sigma^2_\epsilon=0.01$ & $\sigma^2_\epsilon=0.01$ \\  
\end{tabular} 
\caption{Experiment Evaluation Summary for Energy Efficiency($\pm$ std). }
\label{res:ee}
 \end{table*} 

\begin{table*}[h!]
 \setlength{\tabcolsep}{1pt}
  \centering
  \tiny
    \begin{tabular}{l|p{25mm}|c|c|c}
    & $\text{NCAI}_\lambda$ & $\text{NCAI}_{\lambda=0}$ & BNN & BNN+LV \\ \hline
$D_{\text{JS}}(q(z)||p(z))$ & {{$0.014 \pm 0.004$}} & $0.019 \pm 0.006$ & N/A & $0.037 \pm 0.022$ \\ \hline 
$D_{\text{KL}}(p(z)||q(z))$ & {{$0.032 \pm 0.009$}} & $0.041 \pm 0.015$ & N/A & $0.082 \pm 0.058$ \\ \hline 
$\hat{I}(x; \mu_{z})$ & {{$0.036 \pm 0.04$}} & $0.051 \pm 0.049$ & N/A & $0.243 \pm 0.079$ \\ \hline 
$\hat{I}(x; z)$ & {{$0.018 \pm 0.025$}} & $0.023 \pm 0.03$ & N/A & $0.214 \pm 0.052$ \\ \hline 
$\text{HZ}(\{\mu_{z_1}, \ldots, \mu_{z_N}\})$ & {{$0.027 \pm 0.011$}} & $7.137 \pm 5.436$ & N/A & $4.701 \pm 5.439$ \\ \hline 
$s_w^*$ & {{$2.643 \pm 0.226$}} & $2.355 \pm 0.28$ & $0.12 \pm 0.007$ & $1.246 \pm 0.149$ \\ \hline 
$s_y^*$ & {{$0.1 \pm 0.0$}} & $0.1 \pm 0.0$ & $0.5 \pm 0.0$ & $0.1 \pm 0.0$ \\ \hline 
$s_z^*$ & {{$0.01 \pm 0.0$}} & $0.01 \pm 0.0$ & N/A & $0.01 \pm 0.0$ \\ \hline 
95\%-MPIW Test (Unnorm) & {{$5.428 \pm 0.403$}} & $5.418 \pm 0.729$ & $2.979 \pm 0.016$ & $6.789 \pm 0.408$ \\ \hline 
95\%-MPIW Train (Unnorm) & {{$5.371 \pm 0.38$}} & $5.368 \pm 0.7$ & $2.98 \pm 0.005$ & $6.67 \pm 0.342$ \\ \hline 
95\%-PICP Test & {{$94.933 \pm 0.723$}} & $93.333 \pm 1.054$ & $74.2 \pm 2.834$ & $94.733 \pm 0.641$ \\ \hline 
95\%-PICP Train & {{$95.4 \pm 0.894$}} & $93.467 \pm 2.116$ & $73.867 \pm 3.288$ & $94.867 \pm 1.095$ \\ \hline 
KS Test-Stat & {{$0.023 \pm 0.004$}} & $0.051 \pm 0.028$ & N/A & $0.058 \pm 0.035$ \\ \hline 
PC($x, \mu_z$) & {{$0.001 \pm 0.0$}} & $0.0 \pm 0.0$ & N/A & $0.0 \pm 0.0$ \\ \hline 
PC($y, \mu_z$) & {{$0.111 \pm 0.017$}} & $0.068 \pm 0.086$ & N/A & $0.126 \pm 0.107$ \\ \hline 
Post-Pred Avg-LL Test & {{$-1.426 \pm 0.042$}} & $-1.481 \pm 0.018$ & $-2.47 \pm 0.083$ & $-1.867 \pm 0.078$ \\ \hline 
Post-Pred Avg-LL Train & {{$-1.399 \pm 0.058$}} & $-1.429 \pm 0.066$ & $-2.6 \pm 0.143$ & $-1.894 \pm 0.078$ \\ \hline 
RMSE Test (Unnorm) & {{$1.79 \pm 0.09$}} & $1.787 \pm 0.094$ & $1.831 \pm 0.074$ & $1.882 \pm 0.088$ \\ \hline 
RMSE Train (Unnorm) & {{$1.789 \pm 0.09$}} & $1.787 \pm 0.094$ & $1.831 \pm 0.074$ & $1.883 \pm 0.087$ \\ \hline 
Recon MSE & {{$0.142 \pm 0.003$}} & $0.16 \pm 0.017$ & N/A & $0.13 \pm 0.007$ \\ \hline 
Hyperparams & $\sigma_z^2=0.01$, \newline $\sigma^2_\epsilon=0.1$, \newline $\lambda_2=10$, \newline $\epsilon_T=0.01$, \newline $\epsilon_x=0.1$, \newline $\epsilon_y=1.0$ & $\sigma_z^2=0.01$, \newline $\sigma^2_\epsilon=0.1$ & $\sigma^2_\epsilon=0.5$ & $\sigma_z^2=0.01$, \newline $\sigma^2_\epsilon=0.1$ \\  
\end{tabular} 
\caption{Experiment Evaluation Summary for Heavy-Tail ($\pm$ std). }
\label{res:heavytail}
 \end{table*}

\begin{table*}[h!]
  \centering
  \tiny
    \begin{tabular}{l|p{25mm}|c|c|c}
    & $\text{NCAI}_\lambda$ & $\text{NCAI}_{\lambda=0}$ & BNN & BNN+LV \\ \hline
$D_{\text{JS}}(q(z)||p(z))$ & {{$0.002 \pm 0.002$}} & $0.001 \pm 0.002$ & N/A & $0.003 \pm 0.003$ \\ \hline 
$D_{\text{KL}}(p(z)||q(z))$ & {{$0.001 \pm 0.001$}} & $0.002 \pm 0.002$ & N/A & $0.007 \pm 0.002$ \\ \hline 
$\hat{I}(x; \mu_{z})$ & {{$0.046 \pm 0.067$}} & $0.02 \pm 0.024$ & N/A & $0.229 \pm 0.113$ \\ \hline 
$\hat{I}(x; z)$ & {{$-0.006 \pm 0.008$}} & $-0.011 \pm 0.007$ & N/A & $0.076 \pm 0.101$ \\ \hline 
$\text{HZ}(\{\mu_{z_1}, \ldots, \mu_{z_N}\})$ & {{$0.026 \pm 0.038$}} & $0.621 \pm 0.234$ & N/A & $0.918 \pm 0.41$ \\ \hline 
$s_w^*$ & {{$0.456 \pm 0.093$}} & $0.463 \pm 0.087$ & $0.627 \pm 0.039$ & $0.416 \pm 0.152$ \\ \hline 
$s_y^*$ & {{$0.1 \pm 0.0$}} & $0.1 \pm 0.0$ & $0.1 \pm 0.0$ & $0.1 \pm 0.0$ \\ \hline 
$s_z^*$ & {{$0.249 \pm 0.002$}} & $0.247 \pm 0.0$ & N/A & $0.248 \pm 0.001$ \\ \hline 
95\%-MPIW Test (Unnorm) & {{$3.969 \pm 0.184$}} & $4.091 \pm 0.21$ & $2.265 \pm 0.102$ & $4.226 \pm 0.676$ \\ \hline 
95\%-MPIW Train (Unnorm) & {{$3.948 \pm 0.171$}} & $4.099 \pm 0.176$ & $2.264 \pm 0.096$ & $4.298 \pm 0.674$ \\ \hline 
95\%-PICP Test & {{$91.8 \pm 2.308$}} & $92.7 \pm 1.924$ & $73.4 \pm 3.681$ & $91.8 \pm 2.49$ \\ \hline 
95\%-PICP Train & {{$93.9 \pm 0.894$}} & $94.1 \pm 0.822$ & $75.5 \pm 2.598$ & $93.5 \pm 1.173$ \\ \hline 
KS Test-Stat & {{$0.016 \pm 0.004$}} & $0.019 \pm 0.005$ & N/A & $0.025 \pm 0.003$ \\ \hline 
PC($x, \mu_z$) & {{$0.001 \pm 0.001$}} & $0.0 \pm 0.0$ & N/A & $0.0 \pm 0.0$ \\ \hline 
PC($y, \mu_z$) & {{$0.247 \pm 0.148$}} & $0.403 \pm 0.04$ & N/A & $0.193 \pm 0.227$ \\ \hline 
Post-Pred Avg-LL Test & {{$-0.963 \pm 0.041$}} & $-0.962 \pm 0.04$ & $-1.055 \pm 0.08$ & $-1.026 \pm 0.056$ \\ \hline 
Post-Pred Avg-LL Train & {{$-0.885 \pm 0.03$}} & $-0.884 \pm 0.033$ & $-0.95 \pm 0.07$ & $-0.981 \pm 0.107$ \\ \hline 
RMSE Test (Unnorm) & {{$0.337 \pm 0.025$}} & $0.339 \pm 0.026$ & $0.335 \pm 0.025$ & $0.376 \pm 0.032$ \\ \hline 
RMSE Train (Unnorm) & {{$0.337 \pm 0.026$}} & $0.339 \pm 0.026$ & $0.335 \pm 0.025$ & $0.376 \pm 0.031$ \\ \hline 
Recon MSE & {{$0.16 \pm 0.008$}} & $0.151 \pm 0.007$ & N/A & $0.156 \pm 0.013$ \\ \hline 
Hyperparams & $\sigma^2_\epsilon=0.1$, \newline $\lambda_2=10$, \newline $\epsilon_T=0.0003$, \newline $\epsilon_x=0.1$, \newline $\epsilon_y=1.0$ & $\sigma^2_\epsilon=0.1$ & $\sigma^2_\epsilon=0.1$ & $\sigma^2_\epsilon=0.1$ \\  
\end{tabular}
\caption{Experiment Evaluation Summary for Goldberg($\pm$ std). } 
\label{res:goldberg}
 \end{table*}

\begin{table*}[h!]
  \centering
  \tiny
    \begin{tabular}{l|p{25mm}|c|c|c}
    & $\text{NCAI}_\lambda$ & $\text{NCAI}_{\lambda=0}$ & BNN & BNN+LV \\ \hline
$D_{\text{JS}}(q(z)||p(z))$ & $0.012 \pm 0.005$ & {{$0.015 \pm 0.005$}} & N/A & $0.006 \pm 0.005$ \\ \hline 
$D_{\text{KL}}(p(z)||q(z))$ & $0.025 \pm 0.011$ & {{$0.031 \pm 0.009$}} & N/A & $0.016 \pm 0.008$ \\ \hline 
$\hat{I}(x; \mu_{z})$ & $0.614 \pm 0.075$ & {{$0.519 \pm 0.091$}} & N/A & $0.982 \pm 0.121$ \\ \hline 
$\hat{I}(x; z)$ & $0.155 \pm 0.048$ & {{$0.059 \pm 0.02$}} & N/A & $0.235 \pm 0.035$ \\ \hline 
$\text{HZ}(\{\mu_{z_1}, \ldots, \mu_{z_N}\})$ & $0.015 \pm 0.019$ & {{$7.248 \pm 2.598$}} & N/A & $6.445 \pm 2.818$ \\ \hline 
$s_w^*$ & $2.927 \pm 1.612$ & {{$2.368 \pm 1.55$}} & $0.75 \pm 0.065$ & $0.997 \pm 0.943$ \\ \hline 
$s_y^*$ & $0.01 \pm 0.0$ & {{$0.01 \pm 0.0$}} & $0.1 \pm 0.0$ & $0.1 \pm 0.0$ \\ \hline 
$s_z^*$ & $0.247 \pm 0.002$ & {{$0.246 \pm 0.001$}} & N/A & $0.247 \pm 0.001$ \\ \hline 
95\%-MPIW Test (Unnorm) & $1.468 \pm 0.207$ & {{$1.314 \pm 0.126$}} & $0.89 \pm 0.036$ & $1.881 \pm 0.171$ \\ \hline 
95\%-MPIW Train (Unnorm) & $1.479 \pm 0.249$ & {{$1.31 \pm 0.162$}} & $0.889 \pm 0.033$ & $1.908 \pm 0.165$ \\ \hline 
95\%-PICP Test & $95.4 \pm 1.517$ & {{$92.9 \pm 1.294$}} & $77.2 \pm 3.439$ & $92.9 \pm 3.008$ \\ \hline 
95\%-PICP Train & $96.9 \pm 0.894$ & {{$95.0 \pm 0.5$}} & $78.8 \pm 3.978$ & $95.1 \pm 0.418$ \\ \hline 
KS Test-Stat & $0.03 \pm 0.008$ & {{$0.034 \pm 0.008$}} & N/A & $0.033 \pm 0.011$ \\ \hline 
PC($x, \mu_z$) & $0.005 \pm 0.002$ & {{$0.001 \pm 0.001$}} & N/A & $0.0 \pm 0.0$ \\ \hline 
PC($y, \mu_z$) & $0.018 \pm 0.007$ & {{$0.41 \pm 0.113$}} & N/A & $0.572 \pm 0.068$ \\ \hline 
Post-Pred Avg-LL Test & $-0.489 \pm 0.154$ & {{$-0.414 \pm 0.184$}} & $-1.591 \pm 0.417$ & $-1.033 \pm 0.156$ \\ \hline 
Post-Pred Avg-LL Train & $-0.228 \pm 0.125$ & {{$-0.195 \pm 0.108$}} & $-1.357 \pm 0.119$ & $-0.965 \pm 0.078$ \\ \hline 
RMSE Test (Unnorm) & $0.987 \pm 0.103$ & {{$0.978 \pm 0.083$}} & $1.017 \pm 0.06$ & $1.118 \pm 0.096$ \\ \hline 
RMSE Train (Unnorm) & $0.988 \pm 0.101$ & {{$0.979 \pm 0.083$}} & $1.017 \pm 0.06$ & $1.117 \pm 0.096$ \\ \hline 
Recon MSE & $0.017 \pm 0.001$ & {{$0.017 \pm 0.001$}} & N/A & $0.145 \pm 0.004$ \\ \hline 
Hyperparams & $\sigma^2_\epsilon=0.01$, \newline $\lambda_2=10$, \newline $\epsilon_T=0.01$, \newline $\epsilon_x=0.5$, \newline $\epsilon_y=0.5$ & $\sigma^2_\epsilon=0.01$ & $\sigma^2_\epsilon=0.1$ & $\sigma^2_\epsilon=0.1$ \\  
\end{tabular} 
\caption{Experiment Evaluation Summary for Williams($\pm$ std). }
\label{res:williams}
 \end{table*}

\begin{table*}[h!]
  \centering
  \tiny
    \begin{tabular}{l|p{25mm}|c|c|c}
    & $\text{NCAI}_\lambda$ & $\text{NCAI}_{\lambda=0}$ & BNN & BNN+LV \\ \hline
$D_{\text{JS}}(q(z)||p(z))$ & $0.005 \pm 0.003$ & {{$0.006 \pm 0.003$}} & N/A & $0.008 \pm 0.003$ \\ \hline 
$D_{\text{KL}}(p(z)||q(z))$ & $0.01 \pm 0.005$ & {{$0.013 \pm 0.006$}} & N/A & $0.012 \pm 0.004$ \\ \hline 
$\hat{I}(x; \mu_{z})$ & $0.254 \pm 0.057$ & {{$0.283 \pm 0.112$}} & N/A & $0.24 \pm 0.129$ \\ \hline 
$\hat{I}(x; z)$ & $0.128 \pm 0.025$ & {{$0.006 \pm 0.03$}} & N/A & $0.028 \pm 0.017$ \\ \hline 
$\text{HZ}(\{\mu_{z_1}, \ldots, \mu_{z_N}\})$ & $0.004 \pm 0.004$ & {{$8.091 \pm 5.185$}} & N/A & $5.252 \pm 5.607$ \\ \hline 
$s_w^*$ & $0.418 \pm 0.083$ & {{$0.304 \pm 0.028$}} & $0.251 \pm 0.137$ & $0.311 \pm 0.031$ \\ \hline 
$s_y^*$ & $0.1 \pm 0.0$ & {{$0.1 \pm 0.0$}} & $0.1 \pm 0.0$ & $0.1 \pm 0.0$ \\ \hline 
$s_z^*$ & $0.248 \pm 0.001$ & {{$0.248 \pm 0.0$}} & N/A & $0.249 \pm 0.0$ \\ \hline 
95\%-MPIW Test (Unnorm) & $6.862 \pm 1.262$ & {{$5.243 \pm 0.573$}} & $2.007 \pm 0.155$ & $5.275 \pm 0.569$ \\ \hline 
95\%-MPIW Train (Unnorm) & $6.346 \pm 0.821$ & {{$4.906 \pm 0.354$}} & $2.006 \pm 0.153$ & $4.957 \pm 0.36$ \\ \hline 
95\%-PICP Test & $95.5 \pm 2.151$ & {{$94.5 \pm 2.0$}} & $63.4 \pm 1.981$ & $93.6 \pm 2.485$ \\ \hline 
95\%-PICP Train & $97.4 \pm 1.14$ & {{$95.5 \pm 0.707$}} & $69.6 \pm 4.519$ & $94.9 \pm 0.822$ \\ \hline 
KS Test-Stat & $0.031 \pm 0.012$ & {{$0.027 \pm 0.006$}} & N/A & $0.029 \pm 0.009$ \\ \hline 
PC($x, \mu_z$) & $0.0 \pm 0.0$ & {{$0.0 \pm 0.0$}} & N/A & $0.0 \pm 0.0$ \\ \hline 
PC($y, \mu_z$) & $0.109 \pm 0.036$ & {{$0.879 \pm 0.028$}} & N/A & $0.813 \pm 0.159$ \\ \hline 
Post-Pred Avg-LL Test & $-1.285 \pm 0.066$ & {{$-1.211 \pm 0.083$}} & $-2.846 \pm 0.346$ & $-1.278 \pm 0.164$ \\ \hline 
Post-Pred Avg-LL Train & $-1.111 \pm 0.065$ & {{$-1.04 \pm 0.057$}} & $-2.347 \pm 0.154$ & $-1.079 \pm 0.065$ \\ \hline 
RMSE Test (Unnorm) & $0.635 \pm 0.042$ & {{$0.619 \pm 0.039$}} & $0.607 \pm 0.035$ & $0.622 \pm 0.039$ \\ \hline 
RMSE Train (Unnorm) & $0.635 \pm 0.042$ & {{$0.62 \pm 0.039$}} & $0.607 \pm 0.035$ & $0.622 \pm 0.039$ \\ \hline 
Recon MSE & $0.145 \pm 0.008$ & {{$0.159 \pm 0.007$}} & N/A & $0.153 \pm 0.006$ \\ \hline 
Hyperparams & $\sigma^2_\epsilon=0.1$, \newline $\lambda_2=10$, \newline $\epsilon_T=0.01$, \newline $\epsilon_x=0.1$, \newline $\epsilon_y=1.0$ & $\sigma^2_\epsilon=0.1$ & $\sigma^2_\epsilon=0.1$ & $\sigma^2_\epsilon=0.1$ \\  
\end{tabular} 
\caption{Experiment Evaluation Summary for Yuan($\pm$ std). }
\label{res:yuan}
 \end{table*}

 \begin{table*}[h!]
  \centering
  \tiny
    \begin{tabular}{l|p{25mm}|c|c|c}
    & $\text{NCAI}_\lambda$ & $\text{NCAI}_{\lambda=0}$ & BNN & BNN+LV \\ \hline
$D_{\text{JS}}(q(z)||p(z))$ & {{$0.012 \pm 0.012$}} & $0.002 \pm 0.002$ & N/A & $0.001 \pm 0.002$ \\ \hline 
$D_{\text{KL}}(p(z)||q(z))$ & {{$0.042 \pm 0.012$}} & $0.003 \pm 0.003$ & N/A & $0.003 \pm 0.003$ \\ \hline 
$\hat{I}(x; \mu_{z})$ & {{$0.029 \pm 0.008$}} & $0.047 \pm 0.011$ & N/A & $0.045 \pm 0.012$ \\ \hline 
$\hat{I}(x; z)$ & {{$0.013 \pm 0.007$}} & $0.022 \pm 0.003$ & N/A & $0.02 \pm 0.004$ \\ \hline 
$\text{HZ}(\{\mu_{z_1}, \ldots, \mu_{z_N}\})$ & {{$1.641 \pm 0.242$}} & $53.201 \pm 1.247$ & N/A & $52.566 \pm 2.633$ \\ \hline 
$s_w^*$ & {{$0.209 \pm 0.017$}} & $0.15 \pm 0.002$ & $0.197 \pm 0.183$ & $0.147 \pm 0.005$ \\ \hline 
$s_y^*$ & {{$0.01 \pm 0.0$}} & $0.1 \pm 0.0$ & $0.1 \pm 0.0$ & $0.1 \pm 0.0$ \\ \hline 
$s_z^*$ & {{$0.257 \pm 0.001$}} & $0.251 \pm 0.0$ & N/A & $0.251 \pm 0.0$ \\ \hline 
95\%-MPIW Test (Unnorm) & {{$2.4 \pm 0.121$}} & $2.45 \pm 0.04$ & $1.028 \pm 0.014$ & $2.466 \pm 0.026$ \\ \hline 
95\%-MPIW Train (Unnorm) & {{$2.396 \pm 0.106$}} & $2.439 \pm 0.04$ & $1.027 \pm 0.013$ & $2.463 \pm 0.034$ \\ \hline 
95\%-PICP Test & {{$94.796 \pm 1.185$}} & $94.279 \pm 1.479$ & $61.191 \pm 2.176$ & $94.734 \pm 1.426$ \\ \hline 
95\%-PICP Train & {{$94.897 \pm 1.145$}} & $94.893 \pm 0.286$ & $63.265 \pm 1.179$ & $94.969 \pm 0.257$ \\ \hline 
PC($x, \mu_z$) & {{$0.001 \pm 0.0$}} & $0.004 \pm 0.001$ & N/A & $0.003 \pm 0.0$ \\ \hline 
PC($y, \mu_z$) & {{$0.049 \pm 0.01$}} & $0.142 \pm 0.041$ & N/A & $0.154 \pm 0.055$ \\ \hline 
Post-Pred Avg-LL Test & {{$-0.849 \pm 0.038$}} & $-1.147 \pm 0.025$ & $-1.709 \pm 0.22$ & $-1.143 \pm 0.027$ \\ \hline 
Post-Pred Avg-LL Train & {{$-0.805 \pm 0.033$}} & $-1.119 \pm 0.013$ & $-1.479 \pm 0.056$ & $-1.123 \pm 0.015$ \\ \hline 
RMSE Test (Unnorm) & {{$0.983 \pm 0.023$}} & $0.976 \pm 0.016$ & $0.92 \pm 0.022$ & $0.981 \pm 0.017$ \\ \hline 
RMSE Train (Unnorm) & {{$0.983 \pm 0.023$}} & $0.976 \pm 0.017$ & $0.92 \pm 0.022$ & $0.981 \pm 0.017$ \\ \hline 
Recon MSE & {{$0.011 \pm 0.001$}} & $0.114 \pm 0.001$ & N/A & $0.113 \pm 0.001$ \\ \hline 
Hyperparams & $\sigma^2_\epsilon=0.01$, \newline $\lambda_2=10$, \newline $\epsilon_T=0.01$,  \newline $\epsilon_x=0.1$, \newline $\epsilon_y=0.5$ & $\sigma^2_\epsilon=0.01$ & $\sigma^2_\epsilon=0.1$ & $\sigma^2_\epsilon=0.1$ \\  
\end{tabular} 
\caption{Experiment Evaluation Summary for Wine Quality Red($\pm$ std). }
\label{res:wine}
 \end{table*}

  \begin{table*}[h!]
  \centering
  \tiny
    \begin{tabular}{l|p{25mm}|c|c|c}
    & $\text{NCAI}_\lambda$ & $\text{NCAI}_{\lambda=0}$ & BNN & BNN+LV \\ \hline
$D_{\text{JS}}(q(z)||p(z))$ & $0.006 \pm 0.008$ & {{$-0.0 \pm 0.002$}} & N/A & $0.005 \pm 0.004$ \\ \hline 
$D_{\text{KL}}(p(z)||q(z))$ & $0.003 \pm 0.006$ & {{$0.001 \pm 0.005$}} & N/A & $0.002 \pm 0.004$ \\ \hline 
$\hat{I}(x; \mu_{z})$ & $0.086 \pm 0.013$ & {{$0.087 \pm 0.012$}} & N/A & $0.077 \pm 0.012$ \\ \hline 
$\hat{I}(x; z)$ & $0.108 \pm 0.002$ & {{$0.108 \pm 0.004$}} & N/A & $0.107 \pm 0.001$ \\ \hline 
$\text{HZ}(\{\mu_{z_1}, \ldots, \mu_{z_N}\})$ & $5.42 \pm 0.747$ & {{$51.283 \pm 9.548$}} & N/A & $27.059 \pm 4.144$ \\ \hline 
$s_w^*$ & $1.094 \pm 0.903$ & {{$1.137 \pm 0.84$}} & $1.395 \pm 1.155$ & $0.384 \pm 0.026$ \\ \hline 
$s_y^*$ & $0.01 \pm 0.0$ & {{$0.01 \pm 0.0$}} & $0.01 \pm 0.0$ & $0.01 \pm 0.0$ \\ \hline 
$s_z^*$ & $0.251 \pm 0.001$ & {{$0.246 \pm 0.001$}} & N/A & $0.247 \pm 0.0$ \\ \hline 
95\%-MPIW Test (Unnorm) & $6.734 \pm 0.212$ & {{$6.712 \pm 0.235$}} & $6.163 \pm 0.178$ & $8.095 \pm 0.762$ \\ \hline 
95\%-MPIW Train (Unnorm) & $6.724 \pm 0.223$ & {{$6.703 \pm 0.217$}} & $6.163 \pm 0.193$ & $8.118 \pm 0.601$ \\ \hline 
95\%-PICP Test & $99.016 \pm 2.199$ & {{$98.689 \pm 2.933$}} & $97.377 \pm 5.865$ & $98.033 \pm 2.137$ \\ \hline 
95\%-PICP Train & $99.444 \pm 0.387$ & {{$99.444 \pm 0.387$}} & $97.593 \pm 2.544$ & $98.611 \pm 0.655$ \\ \hline 
PC($x, \mu_z$) & $0.007 \pm 0.003$ & {{$0.007 \pm 0.002$}} & N/A & $0.007 \pm 0.003$ \\ \hline 
PC($y, \mu_z$) & $0.005 \pm 0.002$ & {{$0.022 \pm 0.01$}} & N/A & $0.017 \pm 0.01$ \\ \hline 
Post-Pred Avg-LL Test & $0.836 \pm 0.074$ & {{$0.832 \pm 0.077$}} & $0.818 \pm 0.187$ & $0.638 \pm 0.121$ \\ \hline 
Post-Pred Avg-LL Train & $0.865 \pm 0.025$ & {{$0.872 \pm 0.024$}} & $0.868 \pm 0.074$ & $0.678 \pm 0.047$ \\ \hline 
RMSE Test (Unnorm) & $0.005 \pm 0.001$ & {{$0.005 \pm 0.001$}} & $0.005 \pm 0.001$ & $0.008 \pm 0.001$ \\ \hline 
RMSE Train (Unnorm) & $0.005 \pm 0.001$ & {{$0.005 \pm 0.001$}} & $0.005 \pm 0.001$ & $0.008 \pm 0.001$ \\ \hline 
Recon MSE & $0.014 \pm 0.0$ & {{$0.014 \pm 0.0$}} & N/A & $0.014 \pm 0.001$ \\ \hline 
Hyperparams & $\sigma^2_\epsilon=0.01$, \newline $\lambda_2=10$, \newline $\epsilon_T=0.01$, \newline $\epsilon_x=0.5$, \newline $\epsilon_y=0.5$ & $\sigma^2_\epsilon=0.01$ & $\sigma^2_\epsilon=0.01$ & $\sigma^2_\epsilon=0.01$ \\  
\end{tabular} 
\caption{Experiment Evaluation Summary for Yacht($\pm$ std). }
\label{res:yacht}
 \end{table*}

\begin{table*}[h!]
  \centering
  \tiny
   \begin{tabular}{l|p{25mm}|c|c|c}
    & $\text{NCAI}_\lambda$ & $\text{NCAI}_{\lambda=0}$ & BNN & BNN+LV \\ \hline
$D_{\text{JS}}(q(z)||p(z))$ & {{$0.004 \pm 0.002$}} & $0.004 \pm 0.001$ & N/A & $0.006 \pm 0.002$ \\ \hline 
$D_{\text{KL}}(p(z)||q(z))$ & {{$0.008 \pm 0.003$}} & $0.008 \pm 0.004$ & N/A & $0.01 \pm 0.003$ \\ \hline 
$\hat{I}(x; \mu_{z})$ & {{$0.842 \pm 0.06$}} & $0.373 \pm 0.037$ & N/A & $0.667 \pm 0.061$ \\ \hline 
$\hat{I}(x; z)$ & {{$0.035 \pm 0.012$}} & $-0.015 \pm 0.007$ & N/A & $0.146 \pm 0.026$ \\ \hline 
$\text{HZ}(\{\mu_{z_1}, \ldots, \mu_{z_N}\})$ & {{$0.005 \pm 0.001$}} & $7.804 \pm 1.727$ & N/A & $5.09 \pm 0.991$ \\ \hline 
$s_w^*$ & {{$0.488 \pm 0.026$}} & $0.444 \pm 0.019$ & $0.28 \pm 0.132$ & $0.231 \pm 0.002$ \\ \hline 
$s_y^*$ & {{$0.01 \pm 0.0$}} & $0.01 \pm 0.0$ & $0.1 \pm 0.0$ & $0.01 \pm 0.0$ \\ \hline 
$s_z^*$ & {{$0.247 \pm 0.0$}} & $0.247 \pm 0.0$ & N/A & $0.248 \pm 0.0$ \\ \hline 
95\%-MPIW Test (Unnorm) & {{$0.258 \pm 0.026$}} & $0.247 \pm 0.022$ & $0.366 \pm 0.005$ & $0.313 \pm 0.03$ \\ \hline 
95\%-MPIW Train (Unnorm) & {{$0.293 \pm 0.011$}} & $0.277 \pm 0.012$ & $0.366 \pm 0.005$ & $0.336 \pm 0.014$ \\ \hline 
95\%-PICP Test & {{$96.818 \pm 2.591$}} & $96.364 \pm 2.033$ & $96.364 \pm 3.447$ & $95.455 \pm 2.784$ \\ \hline 
95\%-PICP Train & {{$95.871 \pm 0.736$}} & $94.968 \pm 0.957$ & $93.161 \pm 1.08$ & $93.29 \pm 0.866$ \\ \hline 
KS Test-Stat & {{$0.028 \pm 0.005$}} & $0.025 \pm 0.005$ & N/A & $0.024 \pm 0.009$ \\ \hline 
PC($x, \mu_z$) & {{$0.0 \pm 0.0$}} & $0.0 \pm 0.0$ & N/A & $0.0 \pm 0.0$ \\ \hline 
PC($y, \mu_z$) & {{$0.007 \pm 0.003$}} & $0.084 \pm 0.009$ & N/A & $0.121 \pm 0.008$ \\ \hline 
Post-Pred Avg-LL Test & {{$0.263 \pm 0.11$}} & $0.269 \pm 0.107$ & $-0.31 \pm 0.069$ & $0.129 \pm 0.131$ \\ \hline 
Post-Pred Avg-LL Train & {{$0.155 \pm 0.043$}} & $0.159 \pm 0.046$ & $-0.386 \pm 0.035$ & $-0.021 \pm 0.053$ \\ \hline 
RMSE Test (Unnorm) & {{$0.995 \pm 0.059$}} & $0.988 \pm 0.061$ & $1.143 \pm 0.087$ & $1.231 \pm 0.057$ \\ \hline 
RMSE Train (Unnorm) & {{$0.994 \pm 0.059$}} & $0.988 \pm 0.062$ & $1.143 \pm 0.087$ & $1.231 \pm 0.056$ \\ \hline 
Recon MSE & {{$0.017 \pm 0.0$}} & $0.017 \pm 0.0$ & N/A & $0.015 \pm 0.001$ \\ \hline 
Hyperparams & $\sigma^2_\epsilon=0.01$, \newline $\lambda_2=10$, \newline $\epsilon_T=0.01$, \newline $\epsilon_x=0.5$, \newline $\epsilon_y=0.5$ & $\sigma^2_\epsilon=0.01$ & $\sigma^2_\epsilon=0.1$ & $\sigma^2_\epsilon=0.01$ \\  
\end{tabular} 
\caption{Experiment Evaluation Summary for Lidar($\pm$ std). }
\label{res:lidar}
 \end{table*}

 \begin{table*}[h!]
  \centering
  \tiny
    \begin{tabular}{l|p{25mm}|c|c|c}
    & $\text{NCAI}_\lambda$ & $\text{NCAI}_{\lambda=0}$ & BNN & BNN+LV \\ \hline
    $D_{\text{JS}}(q(z)||p(z))$ & $0.003 \pm 0.001$ & {{$0.005 \pm 0.001$}} & N/A & $0.031 \pm 0.005$ \\ \hline 
$D_{\text{KL}}(p(z)||q(z))$ & $0.008 \pm 0.002$ & {{$0.009 \pm 0.002$}} & N/A & $0.357 \pm 0.088$ \\ \hline 
$\hat{I}(x; \mu_{z})$ & $0.057 \pm 0.017$ & {{$0.032 \pm 0.017$}} & N/A & $0.428 \pm 0.04$ \\ \hline 
$\hat{I}(x; z)$ & $0.047 \pm 0.015$ & {{$0.024 \pm 0.014$}} & N/A & $0.387 \pm 0.045$ \\ \hline 
$\text{HZ}(\{\mu_{z_1}, \ldots, \mu_{z_N}\})$ & $0.015 \pm 0.004$ & {{$0.792 \pm 0.357$}} & N/A & $6.408 \pm 2.439$ \\ \hline 
$s_w^*$ & $34.575 \pm 17.89$ & {{$22.305 \pm 7.342$}} & $1.805 \pm 0.094$ & $12.39 \pm 4.903$ \\ \hline 
$s_y^*$ & $0.1 \pm 0.0$ & {{$0.1 \pm 0.0$}} & $1.0 \pm 0.0$ & $0.1 \pm 0.0$ \\ \hline 
$s_z^*$ & $1.0 \pm 0.0$ & {{$1.0 \pm 0.0$}} & N/A & $1.0 \pm 0.0$ \\ \hline 
95\%-MPIW Test (Unnorm) & $7.375 \pm 0.263$ & {{$7.145 \pm 0.16$}} & $4.011 \pm 0.006$ & $22.165 \pm 10.073$ \\ \hline 
95\%-MPIW Train (Unnorm) & $7.433 \pm 0.299$ & {{$7.114 \pm 0.217$}} & $4.011 \pm 0.001$ & $22.267 \pm 10.346$ \\ \hline 
95\%-PICP Test & $93.84 \pm 1.78$ & {{$93.44 \pm 1.757$}} & $73.68 \pm 1.842$ & $96.0 \pm 1.095$ \\ \hline 
95\%-PICP Train & $95.493 \pm 0.256$ & {{$95.227 \pm 0.289$}} & $75.493 \pm 1.489$ & $96.773 \pm 0.446$ \\ \hline 
KS Test-Stat & $0.014 \pm 0.001$ & {{$0.02 \pm 0.002$}} & N/A & $0.044 \pm 0.007$ \\ \hline 
PC($x, \mu_z$) & $0.003 \pm 0.001$ & {{$0.0 \pm 0.0$}} & N/A & $0.0 \pm 0.0$ \\ \hline 
PC($y, \mu_z$) & $0.035 \pm 0.009$ & {{$0.138 \pm 0.014$}} & N/A & $0.161 \pm 0.039$ \\ \hline 
Post-Pred Avg-LL Test & $-1.979 \pm 0.04$ & {{$-1.973 \pm 0.049$}} & $-2.306 \pm 0.059$ & $-2.342 \pm 0.048$ \\ \hline 
Post-Pred Avg-LL Train & $-1.92 \pm 0.021$ & {{$-1.895 \pm 0.018$}} & $-2.217 \pm 0.069$ & $-2.229 \pm 0.04$ \\ \hline 
RMSE Test (Unnorm) & $1.985 \pm 0.051$ & {{$1.932 \pm 0.059$}} & $1.953 \pm 0.071$ & $3.523 \pm 0.501$ \\ \hline 
RMSE Train (Unnorm) & $1.985 \pm 0.051$ & {{$1.933 \pm 0.059$}} & $1.953 \pm 0.071$ & $3.521 \pm 0.501$ \\ \hline 
Recon MSE & $0.124 \pm 0.002$ & {{$0.122 \pm 0.001$}} & N/A & $0.123 \pm 0.005$ \\ \hline 
Hyperparams & $\sigma_z^2=1.0$, \newline $\sigma^2_\epsilon=0.1$, \newline $\lambda_2=10$, \newline $\epsilon_T=0.01$, \newline $\epsilon_x=0.5$, \newline $\epsilon_y=1.0$ & $\sigma_z^2=1.0, \sigma^2_\epsilon=0.1$ & $\sigma^2_\epsilon=1.0$ & $\sigma_z^2=1.0$, \newline $\sigma^2_\epsilon=0.1$ \\  
\end{tabular} 
\caption{Experiment Evaluation Summary for Depeweg($\pm$ std). }
\label{res:depeweg}
 \end{table*}


\FloatBarrier
\newpage

\bibliography{bib}

\newpage

\end{document}